\documentclass{article}

\PassOptionsToPackage{dvipsnames}{xcolor}
\usepackage{iclr2025_conference,times}
\iclrfinalcopy
\usepackage[verbose=true,letterpaper]{geometry}
  \newgeometry{
    textheight=9in,
    textwidth=5.5in,
    top=1in,
    headheight=12pt,
    headsep=25pt,
    footskip=30pt
  }

\usepackage{parskip}

\PassOptionsToPackage{numbers}{natbib}
\usepackage{natbib}
\PassOptionsToPackage{dvipsnames}{xcolor}

\usepackage{hyperref}

\usepackage{float}
\usepackage[utf8]{inputenc}
\usepackage[T1]{fontenc}   
\usepackage{url}         
\usepackage{booktabs}    
\usepackage{amsfonts}  
\usepackage{nicefrac}   
\usepackage{microtype}   
\usepackage{amsmath}
\usepackage{tabularx}
\usepackage{mathabx}
\usepackage{amssymb}
\usepackage{enumitem}
\usepackage{cleveref}
\usepackage{wrapfig}
\usepackage{multirow}
\usepackage{tcolorbox}
\usepackage{adjustbox}
\usepackage{makecell}
\usepackage{bbm}
\usepackage{algorithm}
\usepackage{subcaption}

\usepackage{algpseudocode}

\usepackage{amsthm}
\theoremstyle{plain}
\newtheorem{theorem}{Theorem}[section]
\newtheorem{proposition}[theorem]{Proposition}
\newtheorem{lemma}[theorem]{Lemma}

\theoremstyle{definition}

\theoremstyle{remark}

\usepackage{graphicx}
\graphicspath{{images/}}

\definecolor{our-blue}{HTML}{1f77b4}

\definecolor{lightblue}{HTML}{84C7F9}
\definecolor{lighterblue}{HTML}{D4ECFF}
\newtcolorbox{mybox}{colback=lighterblue,colframe=lightblue}

\usepackage[toc,page,header]{appendix}
\usepackage{minitoc}

\date{}

\newcommand{\diag}{\mathrm{diag}}

\title{A Unification of Discrete, Gaussian, and \\Simplicial Diffusion}

\author{\textbf{Nuria Alina Chandra}$^{*1}$
\quad \textbf{Yucen Lily Li}$^{*1}$
\quad \textbf{Alan N. Amin}$^{*1}$
\quad \textbf{Alex Ali}$^{1}$\\
\textbf{Joshua Rollins}$^{2}$ \quad
\textbf{Sebastian W. Ober}$^{3}$ \quad
\textbf{Aniruddh Raghu}$^{3}$\quad
\textbf{Andrew Gordon Wilson}$^{1}$ \\
$^{1}$New York University \quad $^{2}$CUNY \quad $^{3}$BigHat Biosciences \quad $^{*}$Equal contribution
}
\author{Nuria Alina Chandra\thanks{Equal contribution.} \\
New York University \\
\And
Yucen Lily Li\footnotemark[1] \\
New York University \\
\And
Alan N. Amin\footnotemark[1] \\
New York University \\
\And
Alex Ali \\
New York University \\
\And
Joshua Rollins \\
CUNY \\
\And
Sebastian W. Ober \\
BigHat Biosciences \\
\And
Aniruddh Raghu \\
BigHat Biosciences \\
\And
Andrew Gordon Wilson \\
New York University \\
}

\begin{document}

\doparttoc
\faketableofcontents 

\maketitle
\begin{abstract}
To model discrete sequences such as DNA, proteins, and language using diffusion, practitioners must choose between three major methods: diffusion in discrete space, Gaussian diffusion in Euclidean space, or diffusion on the simplex. Despite their shared goal, these models have disparate algorithms, theoretical structures, and tradeoffs: discrete diffusion has the most natural domain, Gaussian diffusion has more mature algorithms, and diffusion on the simplex in principle combines the strengths of the other two but in practice suffers from a numerically unstable stochastic processes. Ideally we could see each of these models as instances of the same underlying framework, and enable practitioners to switch between models for downstream applications. 
However previous theories have only considered connections in special cases.
Here we build a theory unifying all three methods of discrete diffusion as different parameterizations of the same underlying process: the Wright-Fisher population genetics model. 
In particular, we find simplicial and Gaussian diffusion as two large-population limits.
Our theory formally connects the likelihoods and hyperparameters of these models and leverages decades of mathematical genetics literature to unlock stable simplicial diffusion.
Finally, we relieve the practitioner of balancing model trade-offs by demonstrating it is possible to train a single model that can perform diffusion in any of these three domains at test time.
Our experiments show that Wright-Fisher simplicial diffusion is more stable and outperforms previous simplicial diffusion models on conditional DNA generation.
We also show that we can train models on multiple domains at once that are competitive with models trained on any individual domain.
\end{abstract}

\section{Introduction}
To generate high quality sequences conditioned on desired properties, 
practitioners build diffusion models of language, DNA, and proteins~\citep{Sahoo2024-uj, Sarkar2024-mk, Alamdari2023-nj, li2024absorb}.
These models corrupt each letter in a sequence -- the ``forward" process -- and train a model to reverse that corruption -- the ``backward'' process.
A model which has been trained to de-noise can be used for high-quality conditional generation~\citep{Wang2024-cf}, for optimization~\citep{Gruver2023-sf}, and myriad other downstream tasks~\citep{Luo2022-ha, Baron2025-wg}.

A practitioner has three main choices of forward process (Fig.~\ref{fig:unificationmain}), each with their own strengths:
\begin{enumerate}
    \item \textbf{Discrete:} occurs in the most natural domain~\citep{Campbell2022-zm}.
    \item \textbf{Gaussian:} has more mature sampling and training procedures~\citep{Dieleman2022-ym}.
    \item \textbf{Simplicial:} in theory inherits the continuous algorithms of Gaussian diffusion while in a natural space, but in practice suffers from numerical instability~\citep{Avdeyev2023-mv}.
\end{enumerate}
Unfortunately, there is little theoretical infrastructure to compare these models, and thus practitioners have little tacit knowledge to rely on when selecting or designing a model. 
This gap in understanding is particularly evident in two basic comparison problems which have yet to be solved.
First, despite models from the three frameworks achieving similar likelihood values, there is a belief that the ``continuous-space likelihood is not directly comparable with discrete-space likelihood" \citep{Avdeyev2023-mv}.
Second, forward processes in each of these models are specified by hyperparameters with vastly different interpretations.
It is unclear how to qualitatively compare the assumptions embedded into each set of hyperparameters across models.

Here we address these theoretical and practical challenges by unifying these streams with a process from human population genetics -- the Wright-Fisher (WF) model.
Our contributions are as follows:
\begin{itemize}
    \item We formally prove all three methods are instances of WF (Fig.~\ref{fig:unification}).
    In particular discrete diffusion corresponds to the WF model with a population size of $1$, and simplicial and Gaussian diffusion correspond to large population limits with and without reproduction.
    \item We use this connection to answer the two comparison questions above. Surprisingly, we show that likelihoods can only be compared in some cases, depending on a seemingly inconsequential parameterization choice introduced for only discrete diffusion models in \citet{Austin2021-dg} which we call the \textbf{hollow parameterization}.
    \item We apply our theory to explain and solve the instability of simplicial diffusion by leveraging decades of mathematical genetics literature.
    We show that this stable simplicial diffusion is superior in conditional generation of DNA.
    \item We leverage our theory to show that a particular parameterization choice -- the \textbf{sufficient-statistic parameterization} -- allows one to train a single model that can perform diffusion on all three domains at test time\footnote{
    Of independent interest, it also explains the root of the noted ``time-invariance'' of masking diffusion and extends this property to every diffusion model.
    We discuss this in App.~\ref{app: time invariance}.
    }.
    We show in experiment that models trained this way are competitive with models trained on single domains.
    This removes the necessity for the practitioner to choose a particular model before training.
\end{itemize}
\begin{figure}
    \centering
    \begin{subfigure}[b]{0.25\textwidth}
        \centering
        \raisebox{0.6cm}{\includegraphics[width=\linewidth]{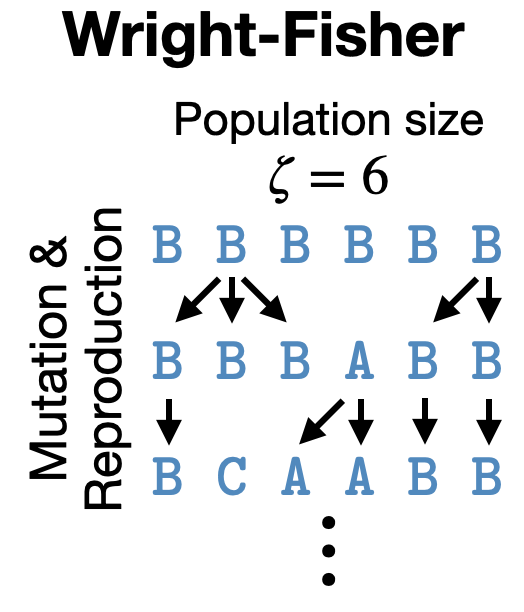}}
        \caption{}
        \label{fig:unificationwf}
    \end{subfigure}
    \hspace{1cm}
    \begin{subfigure}[b]{0.43\textwidth}
        \centering
        \includegraphics[width=\linewidth]{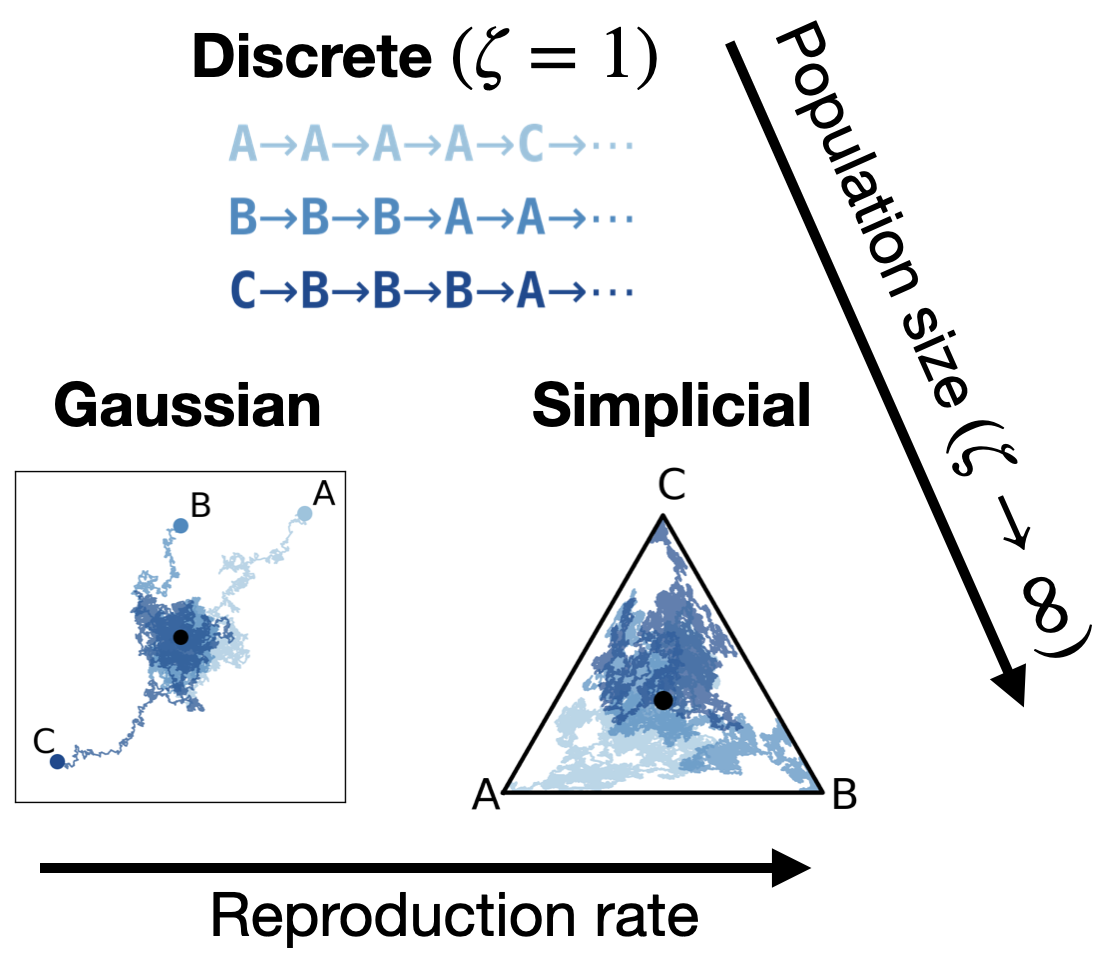}
        \caption{}
        \label{fig:unificationmain}
    \end{subfigure}
    \caption{\textbf{Discrete, Gaussian, and Simplicial diffusion for discrete data are unified by Wright-Fisher diffusion.} 
    \textbf{(a)} Wright-Fisher diffusion with population size $\zeta=6$, showing mutation and reproduction processes across generations.
    \textbf{(b)} The three diffusion methods emerge as different limits of Wright-Fisher: discrete diffusion corresponds to $\zeta=1$, while Gaussian and simplicial diffusion arise as $\zeta \to \infty$ with zero and non-zero reproduction rates.
}
    \label{fig:unification}
\end{figure}

Our code is available at \url{https://github.com/yucenli/unify-diffusion}.

\section{Related work}\label{sec: related-work}
We discuss past unification theories and attempts at stable simplicial diffusion. 
In App.~\ref{app: related work} we discuss related works in classical diffusion theory, and parameterizations of diffusion models.

\paragraph{Theories unifying discrete and continuous diffusion}
\citet{Winkler2024-am} indirectly used a result from \citep{Stone1963-ui} to connect the special case of one-dimensional, unbiased discrete diffusion to one-dimensional Gaussian diffusion.
They use this observation to heuristically argue, or conjecture, the convergence of the backwards processes as well.
\citet{Sahoo2025-dt} suggested that by taking Gaussian diffusion and applying argmax, one recovers discrete diffusion.
\footnote{
Interestingly, \citet{Stone1963-ui} also wrote discrete diffusion as the function of an underlying Gaussian diffusion. However the function from \citet{Stone1963-ui} was a path-dependent time-dilation rather than $\mathrm{argmax}$.}
They used this insight to answer the loss comparison problem by proving that the ELBO of discrete diffusion is always superior to that of continuous diffusion.
Unfortunately, this is based on a mathematical error (details in App.~\ref{sec: sahoo error}): by applying argmax to Gaussian diffusion one does not get a Markov process, a property which was crucial to their proof of the loss comparison question. 
In our approach, we build a mathematically rigorous foundation to compare these models. 

\paragraph{Stable simplicial diffusion models}
\citet{Richemond2022-de} and \citet{Avdeyev2023-mv} suggest diffusion on a simplex using two processes used in finance -- the ``Cox-Ingersoll-Ross process'', and its normalization onto the simplex, the ``Jacobi process'' --  and \citet{Benton2024-br} suggest ``Wright–Fisher diffusion'' in a toy experiment.
However these models struggle from numerical instability.
One solution to this instability is to essentially perform Gaussian diffusion (see App.~\ref{app: related work}).
Another is to build flow-matching models on the simplex \citep{Stark2024-rf, Tang2025-an, Davis2024-iu, eijkelboom2024variational}.
However these sacrifice the ability to straightforwardly calculate a likelihood and access to many diffusion algorithms, such as classifier guidance.

\section{Background and Motivation}
First we describe diffusion models for discrete data and the challenges unifying the frameworks.
\subsection{Diffusion models for discrete data}
We consider modelling a distribution $p(x_0)$ over a discrete space of size $B$, and will extend to sequences of discrete objects below.
Our model will begin with a distribution that is easy to sample from, $q(x_1)$, and then applies a stochastic process parametrized by $\theta$ from time $1$ to $0$.
This produces a trajectory $q_\theta((x_t)_{t=0}^1)$ and we hope to pick $\theta$ so that $q_\theta(x_0)\sim p(x_0).$
\paragraph{Markov processes}
To generate training data to fit $q_\theta((x_t)_{t=0}^1)$, we take samples $x_0\sim p(x_0)$ and evolve it according to a Markov process to get a trajectory $p((x_t)_{t=0}^1)$.
We can train $q_\theta$ on these trajectories by optimizing a negative ELBO
\begin{equation}\label{eqn: main ELBO}
    \begin{aligned}
        -\log q_\theta(x_0)\leq& -\mathbb E_{p((x_t)_{t=0}^1|x_0)}\log\frac{q_\theta((x_t)_{t=0}^1)}{p((x_t)_{t=0}^1|x_0)}\\
    =&-\mathbb E_{p((x_t)_{t=0}^1|x_0)}\log\frac{q_\theta((x_t)_{t=0}^1|x_1)}{p((x_t)_{t=0}^1|x_0, x_1)}+\mathrm{KL}(p(x_1|x_0)|q(x_1)).
    \end{aligned}
\end{equation}
\paragraph{The time dilation function}
To make the second term of Eqn.~\ref{eqn: main ELBO} small we need $p(x_1|x_0)\approx q(x_1)$.
Conveniently, applying a Markov process to $x_0$ usually leads to $p(x_t|x_0)$ converging to a stationary distribution $p(x_\infty)$ as $t\to\infty$, a good choice for $q(x_1)$.
However our $t$ is on the interval $[0, 1]$,
so we compress $[0, \infty)$ into $[0, 1]$:
we pick an increasing ``time dialation'' function $\tau: [0, 1]\to [0, \infty)$ and simulate $x_t$ so that it has had the Markov process applied to it for time time $\tau_t$.
In particular, if $\tau_1$ is very large, $p(x_1|x_0)\approx p(x_\infty)=q(x_1)$.
$\tau_t$ is a more convenient parametrization for our presentation than equivalent functions $\beta_t=\dot\tau_t, \alpha_t=\exp(-\tau_t)$ in other works~\citep{Shi2024-fs}. 
Picking $\tau_1$ very large, the second term of the ELBO can be made arbitrarily small, so we leave it out of the presentation below.
\paragraph{Matching forward and backward flow}
$q_\theta$ is usually parameterized to take $x_t, t$ and predict the $x_0$ that generated $x_t$, that is, approximate $p(x_0\mid x_t, t)$; we represent this prediction $\tilde x_{0}=q_\theta(x_0|x_t, t)$ as a vector of probabilities over the $B$ tokens $\sum_b \tilde x_{0, b}=1$.
Some rearrangement then allows one to rewrite the first term of Eqn.~\ref{eqn: main ELBO} as an expectation of a term $L$ that can be interpreted as the divergence between the ``infinitesimal flow'' forward $p$ and backward $q_\theta$ at $x_t$:
$$E_{t\sim\mathrm{Unif}(0, 1)}E_{p(x_t|x_0)}L(x_t, t, x_0, \tilde x_{0}).$$
Thus getting a stochastic estimate of the ELBO has 3 steps:
(1) Sample noisy $x_t$ by simulating the Markov process for time $\tau_t$, (2) Predict de-noised $x_0$ with $q_\theta(x_0\mid x_t, t)$, and (3) Estimate the ELBO by computing the particular form of $L$.

\paragraph{Moving to multiple dimensions}
To model sequences of discrete objects $x_0=x_0^{1}\cdots x_0^{D}$, we simply apply the forward process to each position $x_0^{d}$ independently.
``{Sample noisy} $x_t$" remains the same, repeated for every $d$.
The ``infinitesimal flow'' for each position is also independent: the ``{Estimate ELBO}'' step also remains the same, repeated for every $d$ and then summed across all $d$.
Therefore, in the ``{Predict de-noised} $x_0$'' step we will predict $\tilde x_{0}^d=q_\theta(x_0^d|x_t, t)$ for each $d$.

\subsection{Challenges comparing domains for discrete diffusion}

\paragraph{Comparing diffusion models}
A practitioner much choose a forward process which will determine how they train their diffusion model.
For discrete diffusion, the forward process is mutation defined with a rate matrix $\mathcal L$; the form for $L$ was derived in \citet{Campbell2022-zm}.
This gives Alg.~\ref{alg: discrete diffusion}, where $\vec x_0$ is the indicator vector for the token $x_0$, $\mathbb D(\lambda_1||\lambda_2)=\lambda_1\log\frac{\lambda_1}{\lambda_2}-\lambda_1+\lambda_2$ is the KL divergence between two Poisson distributions, $\hat w(\tilde x_0):=\sum_b\tilde x_{0b}\hat w(b)$, and $\dot \tau_t$ is the derivative of $\tau_t$.
For Gaussian diffusion, the forward process is Brownian motion on embeddings $\mathrm{emb}(x_0)\in\mathbb R^r$; the form for $L$ was derived in \citet{Ho2020-yq}.
This gives Alg.~\ref{alg: Gaussian diffusion}.
Now, how can a practitioner compare how well each model fits its data, and how can they leverage their expert knowledge when designing their forward process?
Unfortunately there is little infrastructure for answering these questions.

\begin{minipage}{0.48\textwidth}
\begin{algorithm}[H]
\caption{ELBO for discrete diffusion\label{alg: discrete diffusion}}
\begin{algorithmic}[1]
\State Sample $t \sim \mathrm{Unif}(0, 1)$
\State \textbf{Sample noisy $x_t$:}
\State Sample $ x_t\sim \mathrm{Categorical}( \vec x_0^Te^{\tau_t\mathcal L})$\vphantom{$x_t=e^{-\tau_t}\mathrm{emb}( x_0)+\sqrt{1-e^{-2\tau_t}}N(0, I)$}
\State \textbf{Predict de-noised $x_0$:}
\State Predict $\tilde{x}_0 = q_\theta(x_0|x_t, t)$
\State \textbf{Estimate ELBO:}
\State $\hat w(b) = (\vec b^Te^{\tau_t\mathcal L})(1/\vec b^Te^{\tau_t\mathcal L})^{T}$\vphantom{$L = \frac{\dot\tau_t e^{-2\tau_t}}{(1-e^{-2\tau_t})^2}\Vert \mathrm{emb}( x_0)-\mathrm{emb}( \tilde x_0)\Vert^2\vphantom{x^Te^{\tau_t\mathcal L}}$}
\State $L = \sum\limits_{b\neq x_t}\mathcal L_{b\to x_t} \dot\tau_t\mathbb{D}\left(\hat w(x_0)_{bx_t}||\hat w(\tilde x_0)_{bx_t}\right)$
\end{algorithmic}
\end{algorithm}
\end{minipage}
\hfill
\begin{minipage}{0.50\textwidth}
\begin{algorithm}[H]
\caption{ELBO for Gaussian diffusion\label{alg: Gaussian diffusion}}
\begin{algorithmic}[1]
\State Sample $t \sim \mathrm{Unif}(0, 1)$
\State \textbf{Sample noisy $x_t$:}
\State Set $x_t=e^{-\tau_t}\mathrm{emb}( x_0)+\sqrt{1-e^{-2\tau_t}}N(0, I)$
\State \textbf{Predict de-noised $x_0$:}
\State Predict $\tilde{x}_0 = q_\theta(x_0|x_t, t)$
\State \textbf{Estimate ELBO:}
\State $L = \frac{\dot\tau_t e^{-2\tau_t}}{(1-e^{-2\tau_t})^2}\Vert \mathrm{emb}( x_0)-\mathrm{emb}( \tilde x_0)\Vert^2\vphantom{x^Te^{\tau_t\mathcal L}}$
\State \vphantom{$L = \sum\limits_{b\neq x_t}\mathcal L_{b\to x_t} \dot\tau_t\mathbb{D}\left(\hat w(x_0)_{bb'}||\hat w(\tilde x_0)_{bb'}\right)$}
\end{algorithmic}
\end{algorithm}
\end{minipage}

\textit{Likelihood comparison}
We would like to compare the ELBOs $\mathbb E[L]$ of discrete and Gaussian diffusion, but the later are infinity due to a singularity as $t$ becomes small\footnote{
To see this, note at initialization $\Vert \mathrm{emb}( x_0)-\mathrm{emb}( \tilde x_0)\Vert^2$ is roughly a constant, and for the classical choice $\tau_t=-\frac 1 2\log(1-t)$, the square error in Alg.~\ref{alg: Gaussian diffusion} is weighted by $\frac{1}{2t^2}$, so the loss is $\gtrsim\int_0^1t^{-2}dt=\infty;$
a different choice of $\tau_t$ only acts as a change-of-variables, and therefore cannot make the loss finite.
}.
Practitioners must therefore choose a minimum $t_{\min}$\footnote{
Some discrete diffusion models also have a singularity at $t\to 0^+$, requiring one to specify a $t_{\min}$~\citep{Campbell2022-zm, Lou2023-vm}.
This is not the case for ``schedule-conditioned'' models, including masking, partially explaining its popularity~\citep{Amin2025-ag, Shi2024-fs}.
}.
Formally this is equivalent to estimating an ELBO for $\log p(x_{t_{\min}})$ instead of $\log p(x_0).$
However, $p(x_{t_{\min}})$ is a continuous density, fundamentally a different object than the probability of a discrete object $p(x_0)$.
Paradoxically, the values of the ELBOs of the two models are often close suggesting they may nevertheless be formally comparable.

\textit{Hyperparameter comparison}
Discrete and Gaussian diffusion models are specified by hyperparameters $\mathcal L$ and $\mathrm{emb}$ with vastly different interpretations:
a matrix whose entry $\mathcal L_{b_1\to b_2}$ describes the rate at which $b_1$ mutates to $b_2$, versus an embedding function $\mathrm{emb}$ that takes the alphabet into Euclidean space $\mathbb R^r$ for some $r$ (we write $\mathrm{emb}(\tilde x_0)$ as shorthand for $\sum_b\tilde x_{0, b}\mathrm{emb}(b)$).

\paragraph{Stability of simplex diffusion}
Below we'll also discuss simplicial diffusion which in principle combines the combines the strengths of discrete and Gaussian diffusion.
In practice, it is numerically unstable and slow as \textbf{Sample noisy $x_t$} involves ``sampling from Jacobi diffusion processes [which] is more expensive than commonly used SDEs'', and \textbf{Estimate ELBO} involves a calculation which ``at very small $t$ tends to become very large and cause numerical issues'' \citep{Avdeyev2023-mv}.

\paragraph{Practical unification}
Currently, practitioners must commit to a $q_\theta(x_0\mid x_t, t)$ trained on one these three modalities before training, restricting their access to downstream algorithms.
Ideally they could avoid making this choice.

\section{Unifying discrete and Gaussian diffusion}

To build the infrastructure for comparing domains for discrete diffusion, we unify discrete and Gaussian diffusion in a broader framework.
Our results lead to better understanding of loss and hyperparameter comparisons.
In the following section we extend our framework to simplicial diffusion.

\subsection{Unification result}

Our idea is to represent each dimension of a sequence with $\zeta$ copies to get a \textit{sequence of sequences}.
$$\text{ex. for}\ \zeta=4, x_0=\texttt{A|C|C|T}\text{ is represented as } \texttt{AAAA|CCCC|CCCC|TTTT}.$$
Then each letter in each sequence is evolved by the mutation matrix $\mathcal L$.
When $\zeta=1$ we get discrete diffusion and we show that as $\zeta\to\infty$ we get Gaussian diffusion.
Below we discuss the case where $x_0$ is a single letter / token, which can naturally be extended to a multi-dimensional diffusion model.

\begin{wrapfigure}{r}{0.51\textwidth}
\vspace{-0.3cm}
\begin{minipage}{\linewidth}
\begin{algorithm}[H]
\caption{ELBO for $\zeta$ discrete diffusion\label{alg: zeta arb}}
\begin{algorithmic}[1]
\State Sample $t \sim \mathrm{Unif}(0, 1)$
\State \textbf{Sample noisy $x_t$:}
\State Sample $ \vec x_t\sim \textcolor{blue}{\mathrm{Multinomial}}(\textcolor{blue}{\zeta}, \vec x_0^Te^{\tau_t\mathcal L})\textcolor{blue}{/\zeta}$
\State \textbf{Predict de-noised $x_0$:}
\State Predict $\tilde{x}_0 = q_\theta(x_0\mid \vec x_t, t)$
\State \textbf{Estimate ELBO:}
\State $\hat w(b) = (\vec b^Te^{\tau_t\mathcal L})(1/\vec b^Te^{\tau_t\mathcal L})^{T}$\vphantom{$L = \frac{\dot\tau_t}{2\tau_t^{2}}\Vert \textcolor{blue}{P_1}x_0-\textcolor{blue}{P_1}\tilde x_0\Vert^2\vphantom{X^Te^{\tau_t\mathcal L}}$}
\State $L =\! \sum\limits_{b\neq b'}\mathcal L_{b\to b'} \dot\tau_t\textcolor{blue}{ \zeta \vec x_{tb'}} \mathbb{D}\left(\hat w(x_0)_{bb'}||\hat w(\tilde x_0)_{bb'}\right)$
\end{algorithmic}
\end{algorithm}
\end{minipage}
\end{wrapfigure}
\paragraph{$\vec x_t$ on the simplex}
We will ultimately arrive at a Gaussian limit in Euclidean space, but we first represent $x_t$ on the simplex.
Above, $x_t$ was one of $B$ tokens; now it's one of $B^{\zeta}$ sequences of $B$ tokens $x_t=x_t^{(1)}\cdots x_t^{(\zeta)}$.
It can be generated by sampling each $x_t^{(z)}\sim \mathrm{Categorical}( \vec x_0^Te^{\tau_t\mathcal L}).$
In App.~\ref{app: mid zeta proof} we note that the loss and target $p(x_0\mid x_t, t)$ \textit{do not depend on the order} of $x_t$.
Therefore we can represent $x_t$ as a normalized vector of counts $\vec x_{t, b}=\#\{b\text{ in }x_t\}/\zeta$.
In App.~\ref{app: mid zeta proof} we derive the ELBO, giving Alg.~\ref{alg: zeta arb} -- differences to discrete diffusion in Alg.~\ref{alg: discrete diffusion} are in \textcolor{blue}{blue}.

\paragraph{Gaussian limit as $\zeta\to\infty$}
The main idea of our proof below is that as $\zeta\to \infty$, trajectories converge quickly to $\vec \pi$, the stationary distribution of $\mathcal L$, and behave like Gaussians near $\vec \pi$ because of the central limit theorem (Fig.~\ref{fig: gaussian convergence}).
As $\zeta\to\infty$ we will zoom further and further into the neighbourhood of $\pi$ where the diffusion occurs -- we move from \textit{diffusion on the simplex} to \textit{diffusion in Euclidean space}.
Our proof extends previous results in one-dimension~\citep{Stone1963-ui}, but uses more modern machinery;
    interestingly, we see that in the multi-dimensional case, the relevant Gaussian diffusion occurs in a subspace determined by the first eigenspace of $\mathcal L$.
\footnote{This is analogous to asymptotic methods that zoom into a point in a bounded space to get a limit in its unbounded tangent plane (ex. chapter 20 of \citet{Van_der_Vaart1998-qy})}.
\begin{theorem}\label{thm: gaussian}
    (Formal statement and proof in App.~\ref{app: gaussian proof})
    Call $0>-\lambda_1>-\lambda_2>\dots$ the eigenvalues of $\mathcal L$ and $P_1$ the projection onto the left eigenspace corresponding to $\lambda_1$.
    Without loss of generality, assume $\lambda_1=1$\footnote{
    This assumption is for convenience. Rescale $\mathcal L^{\mathrm{new}}=\frac 1 {\lambda_1}\mathcal L$ and $\tau_t^{\mathrm{new}}=\lambda_1\tau_t$ to get the same diffusion.}.
    For each $\zeta$ pick time dilation $\tau^\zeta_t=\frac 1 {2}\log\left(\zeta e^{-2\tau_t}-\zeta+1\right)$ and rescale $\vec x_t^\zeta = \sqrt{\zeta-(\zeta-1)e^{2\tau_t}}(\vec{x_t}-\vec\pi)/\sqrt{\vec\pi}$.
    Define the embedding into $\mathbb R^{\mathrm{rank}(P_1)}$, $Q_1=\mathfrak j_1(\tilde Q_1\tilde Q_1^T)^{-1/2}\tilde Q_1$ where $\tilde Q_1=\diag(\vec\pi)^{-1/2}P_1\diag(\vec\pi)^{1/2}$ and $\mathfrak j_1$ is any isometry from $\mathrm{Im}(\tilde Q_1)\to\mathbb R^{\mathrm{rank}(P_1)}$.
    
    \textbf{When $\zeta=1$ we get discrete diffusion}: $\tau^\zeta_t=\tau_t$ and $\vec x_t^\zeta$ is only linearly transformed $(\vec x_t-\vec\pi)/\sqrt{\vec\pi}.$

    \textbf{When $\zeta\to\infty$, we get Gaussian diffusion in the first eigenspace}:
    \begin{itemize}
    \item Only the first eigenspace has signal: the component of $x_t^\zeta$ in $\mathrm{Ker}Q_1$ becomes independent of $x_0$.
    \item The paths $(Q_1\vec x_t^\zeta)_{t\in(0, 1)}$ converge in distribution to paths from Gaussian diffusion with time dilation $\tau_t$ and embedding $\mathrm{emb}(x_0)=Q_1(\vec x_0/\sqrt{\vec\pi})$.
    \item The ELBO in Alg.~\ref{alg: zeta arb} converges to the ELBO in Alg.~\ref{alg: Gaussian diffusion}.
    \end{itemize}
\end{theorem}

\begin{wrapfigure}{l}{0.5\linewidth}
    \centering
    \vspace{-0.5cm}
    \includegraphics[width=1\linewidth]{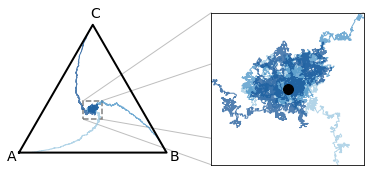}
    \caption{\textbf{Discrete diffusion with a large population converges to Gaussian diffusion.}
    With $\zeta=1000$, we show example trajectories $(\vec x_t)_t$ that converge to approximate Gaussians near $\vec \pi$.
    \vspace{-1cm}
  }
    \label{fig: gaussian convergence}
\end{wrapfigure}\textit{Proof idea:}
As $\zeta\to\infty$, by the law of large numbers, $\vec x_t$ approaches $\vec x_0^Te^{\tau_t\mathcal L}$ which itself goes to the stationary distribution of $\mathcal L$.
We can therefore decompose
\begin{equation*}
\vec x_t-\vec \pi=  \underbrace{\vec x_0^Te^{\tau_t^\zeta\mathcal L}-\vec \pi}_{\text{signal}}+\underbrace{\vec x_t - \vec x_0^Te^{\tau_t^\zeta\mathcal L}}_{\text{noise}}.
\end{equation*}
The ``noise'' term is $\vec x_t-\mathbb E\vec x_t$.
Since $x_t$ is an average of $\zeta$ samples, by the central limit theorem, it is approximately Gaussian with scale $\zeta^{-1/2}$ and independent of $x_0$.
The ``signal'' term therefore is what allows us to predict $x_0$.

The only relevant behaviour is that of the slowest-decaying eigenspaces of $\mathcal L$: the top eigen-space represents the convergence to $\vec \pi$ and cancels with $-\vec \pi$, the next one is $P_1$ with eigenvalue $-1$, and all others vanish quickly.
Therefore the signal is approximately $e^{-\tau_t^\zeta}P_1\vec x_0$.
This means
$$\vec x_t-\vec \pi\approx e^{-\tau_t^\zeta}P_1\vec x_0 + \frac{1}{\sqrt{\zeta}} \mathcal{N}(0, \Sigma)\text{ for some }\Sigma.$$
Finally, choosing the right scaling and $\tau_t^\zeta$ gives us Gaussian diffusion.
Most of the formal proof involves checking regularity conditions.
\hfill$\square$
\subsection{Application: Understanding comparisons of losses and hyperparameters}\label{sec: solve theory}
\begin{wrapfigure}{r}{0.25\linewidth}
    \centering
    \vspace{-0.4cm}
    \includegraphics[width=1\linewidth]{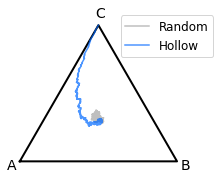}
    \caption{\textbf{The hollow parameterization leads to realistic reverse path samples.} $\zeta=300$.
    \vspace{-0.4cm}
}
    \label{fig: hollow p}
\end{wrapfigure}\paragraph{Loss comparison}
Thm.~\ref{thm: gaussian} suggests that there is virtually no difference to training a discrete diffusion model with $\zeta=10^{100}$ and training Gaussian diffusion with Alg.~\ref{alg: Gaussian diffusion} on a computer, suggesting their ELBOs are comparable.
Yet the limiting Gaussian ELBO is infinite!
Fig.~\ref{fig: gaussian convergence} suggests why:
paths from $\vec x_t$ have two phases, a nearly deterministic phase at low $t$ (Fig \ref{fig: gaussian convergence} left), and then a random phase (Fig \ref{fig: gaussian convergence} right).
Diffusion models reversing these paths should therefore go through a random phase, until $p(x_0\mid \vec x_t, t)$ becomes obvious, and then trace a deterministic path back to $x_0$.
However, at initialization, $x_0$ is ``never obvious'' to the neural network $q_\theta(x_0\mid \vec x_t, t)$, leading to mismatches to the deterministic paths in samples (Fig.~\ref{fig: hollow p} ``Random").
As $\zeta$ gets larger, the paths get more deterministic, \textbf{causing the singularity in the limit}.

The practical solution is simple -- weight the output of the neural network by the evidence for each $x_0$, $q_\theta(x_0\mid x_t, t)\propto p( x_t\mid x_0, t)q_\theta(x_0)$ where $p(x_t\mid x_0, t)$ ``automatically handles'' deciding when $x_0$ is obvious (Fig.~\ref{fig: hollow p} ``Hollow'').
This was suggested in the appendix of~\citet{Austin2021-dg} as way to improve discrete diffusion models, but becomes important here as a way to build new Gaussian diffusion models with formally comparable likelihoods\footnote{Note the hollow parametrization is specific to \textit{discrete data} where there are only finitely many possible $x_0$.}.
\citet{Amin2025-ag} showed that in higher dimensions this becomes equivalent to using the \textbf{``hollow'' predictor}\footnote{This does not require a change of architecture: the network can take in $x_t$ but must learn to disregard $x^d_t$.} $q_\theta(x_0^d\mid x_t, t)\propto p(x_t^d\mid x_0^d, t)q_\theta(x_0^d|x_t^{-d}, t)$ where $x_t^{-d}$ is all positions except $d$.
In App.~\ref{app: hollow proof} we formally prove that the hollow parametrization removes the singularity of the ELBO.

\begin{wrapfigure}{l}{0.3\linewidth}
    \vspace{-0.3cm}
    \centering
    \includegraphics[width=1\linewidth]{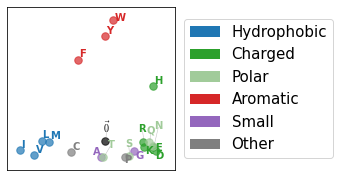}
    \caption{\textbf{$\mathrm{emb}$ of amino acids from BLOSUM $\mathcal L$.}
    $\mathrm{emb}(x_0)$ from Thm.~\ref{thm: gaussian} for $\mathcal L$ from~\citet{Amin2025-ag}.
    \vspace{-0.2cm}}
    \label{fig: blosum}
\end{wrapfigure}\paragraph{Hyperparameter comparison}
Thm.~\ref{thm: gaussian} gives us a formula for $\mathrm{emb}$ determined by the slowest-decaying directions in $\mathcal L$.
App.~\ref{app: rep to inf gen} also shows that every $\mathrm{emb}$ can be induced from some $\mathcal L$.
Remarkably, this connection accommodates embeddings in different dimensions $\mathbb R^r$: $r$ is determined by the dimension of the dominant eigenspace of $\mathcal L$.
In Fig.~\ref{fig: blosum} we show $\mathrm{emb}$ for the BLOSUM stochastic processes for amino acids, and see it clusters similar amino acids together.
The practical implications are 
(1) one can sanity-check their designed $\mathcal L$ by plotting its induced embeddings, and 
(2) discrete diffusion offers a richer design space, as one can specify all the interacting eigenspaces of $\mathcal L$ rather than just the dominant one, $\mathrm{emb}$.

\clearpage
\section{Unifying simplicial diffusion}\label{sec: wf}
Now we add simplicial diffusion to our unification of discrete and Gaussian diffusion.
Proving the equivalence of the forward process: we add reproduction to our population of $\zeta$ letters and simply refer to the well known result of~\citet{Kimura1955-aj} from mathematical genetics.
We also derive new results on the limit of the ELBO and explore our connection with theory of mathematical genetics; this will allow us to address the instabilities that plague simplicial diffusion models.
\subsection{Unification results}
\paragraph{The Wright-Fisher model}
We now allow our population of $\zeta$ to reproduce.
The population is swapped with a new generation at rate $\zeta$ (that is, a new generation occurs at $\Delta\tau\sim\mathrm{Exp}(1)/\zeta$) and at each generation we create $\zeta$ ``children'' which pick a parent uniformly at random.
Between generations, individuals also mutate according $\mathcal L$ (Fig.~\ref{fig:unificationwf}).
We now ask what happens when $\zeta\to\infty$.   

\paragraph{The limit of $p((x_t)_t)$} \citet{Kimura1955-aj} was the first to derive the $\zeta\to\infty$ limit of the stochastic process.
Unlike the mutation-only case which zooms in to $\vec\pi$, this limiting distribution has paths that travel throughout the simplex (Fig.~\ref{fig:unificationmain}).
This limit, often itself called ``Wright-Fisher diffusion'' is exactly the forward process in simplicial diffusion~\citep{Avdeyev2023-mv}.
Details are in App.~\ref{app: wf proof fwd}.
One biologically reasonable assumption past works make is a parent-independent mutation rate matrix, that is, $\mathcal L=\psi\times(\mathbbm{1}\vec\pi^T-I)$ for stationary distribution $\vec\pi$ and mutation rate $\psi>0$.
This does not restrict the design space of simplicial diffusion, which is specified by an intensity parameter $\psi$ and stationary distribution $\vec\pi$, so we make the same assumption.  

\paragraph{The limit of the ELBO} We derive the limit of the discrete diffusion ELBO.
Remarkably, we get an objective that matches ``score functions'' $\vec s$ like that heuristically derived in~\citet{Avdeyev2023-mv}.
It is also mathematically equivalent to the expression in Eqn 27 of~\citet{Benton2024-br} with two differences: (1) we avoid taking the derivative of the neural network, and (2) their expression differs by an unknown constant from the ELBO, while ours is directly comparable with ELBOs from other models.
\begin{theorem}\label{thm: simplex loss}
    (Proof in App.~\ref{app: wf proof elbo})
    As $\zeta\to\infty$, the discrete diffusion objective in Alg.~\ref{alg: discrete diffusion} converges to the quantity in line 9 of Alg.~\ref{alg: simplicial diffusion}.
\end{theorem}
The main idea of the proof is an application of a Taylor expansion and Stirling's approximation; the main challenge is handling of behaviour at the boundaries of the simplex and regularity conditions.

\subsection{Application: fast and stable simplicial diffusion}\label{sec: application wf improved}
We have unified simplicial diffusion with discrete and Gaussian diffusion, in particular allowing likelihood comparison, which will be crucial in the following section.
Our unification also immediately suggests a connection to the mathematical genetics literature.
We now apply the solutions from that literature to improve simplicial diffusion models.
Many of the formulas are standard but long
-- we save their statement and experimental validation to App.~\ref{app: wf details}.

\begin{minipage}{\linewidth}
\begin{algorithm}[H]
\caption{ELBO for simplicial diffusion\label{alg: simplicial diffusion}. Our changes to \citet{Avdeyev2023-mv} are coloured.}
\begin{algorithmic}[1]
\State Sample $t \sim \mathrm{Unif}(0, 1)$
\State \textbf{Sample noisy $x_t$:}
\State \textcolor{blue}{Sample $m\sim A(\psi, \tau_t)$ with Alg.~\ref{alg: jenkins general}}; \textcolor[RGB]{255,30,77}{if $\tau_t<0.05$, use Alg.~\ref{alg: jenkins low t}}
\State \textcolor{blue}{Sample $\vec x_t\sim\mathrm{Dirichlet}(\psi\vec\pi +m\vec x_0).$}
\State \textbf{Predict de-noised $x_0$:}
\State Predict $\tilde{x}_0 = q_\theta(x_0\mid x_t, t)$
\State \textbf{Estimate ELBO:}
\State Compute $\vec s(\vec x_t\mid x_0, t)=\nabla_{x_t}\log p(x_t| x_0, t)$ with Eqn.~\ref{eqn:WF score}
\State $L= \textcolor[RGB]{34,139,34}{\frac{\dot\tau_t}{2}}\|\vec s(\vec x_t\mid x_0, t)-\vec s(\tilde x_t\mid x_0, t)\|^2_{\textcolor[RGB]{34,139,34}{\diag(\vec x_t)-\vec x_t\vec x_t^T}}$ \textcolor[RGB]{34,139,34}{(this is an ELBO)}; \textcolor[RGB]{255,30,77}{if $\tau_t<0.05$, use Eqn.~\ref{eqn: L bound}}
\end{algorithmic}
\end{algorithm}
\end{minipage}

\paragraph{\textcolor{blue}{Sampling noisy $x_t$}}
\citet{Avdeyev2023-mv} samples $x_t$ by costly and approximate simulation from a stochastic differential equation (SDE).
Instead, the suggestively titled paper ``Exact simulation of the Wright-Fisher diffusion'' \citep{Jenkins2017-ry} gives a fast exact formula for the marginals $x_t$.
The algorithm samples $\vec x_t$ from a Dirichlet that is centred at the stationary mutation distribution $\vec\pi$ when $m=0$ and becomes more concentrated around the signal $x_0$ when $m$ is larger.
$m$ itself is an integer sampled from a distribution $A(\psi, \tau_t)$ that represents, going back in time $\tau_t$, how many ancestors the population descend from -- it is small when $\tau_t$ is large, when everyone descended from a handful of individuals from far back in time.
\citet{Benton2024-br} proposed the same procedure, but applied it in a toy setting.

\paragraph{\textcolor[RGB]{34,139,34}{Computing the loss}} For the loss, \citet{Avdeyev2023-mv} derived a likelihood that involved calculating the derivative the predictor $q_\theta(x_0 \mid x_t, t)$ making it too expensive to train on.
They instead suggest training a heuristically motivated loss matching the vector $\vec s$ to a ground truth.
In Thm.~\ref{thm: simplex loss} we recognize this loss as an ELBO and derive the appropriate scaling ${\frac{\dot\tau_t}{2}}$ and metric ${\diag(\vec x_t)-\vec x_t\vec x_t^T}$.
\paragraph{\textcolor[RGB]{255,30,77}{Low $\hspace{0.05em}t\hspace{0.05em}$ behaviour}}
Both the simulation of $A(\psi, \tau_t)$ and the calculation of the gradients $\nabla_{x_t}\log p(x_t\mid x_0)$ involves an infinite series~\citep{Tavare1984-uz}.
Luckily the terms converge extremely fast.
This is not true however at low $t$, which is the primary cause of the instability of simplicial diffusion.
This instability is also well known in the genetics literature, with \citet{Griffiths1984-bo} emphatically stating that using the infinite series at low $t$ ``produces nonsense from a computer.''

The solution at low $t$ is to replace the series approximation, which gets worse with lower $t$, with a central limit approximation for $A(\psi, \tau_t)$~\citep{Griffiths1984-bo, Jenkins2017-ry} that improves with lower $t$;
this is analogous to how reflected diffusion models were made stable despite their own infinite series expansion with the same problem~\cite{Luo2022-ha}.
We picked the $\tau_t<0.05$ threshold as recommended by ~\citet{Jenkins2017-ry}.
In App.~\ref{app: wf details low t} we describe how to use this approximation to also stabilize the loss computation.

\begin{figure}[t]
    \centering
    \begin{subfigure}[b]{0.5\textwidth}
    \centering{
    \includegraphics[width=1.3\textwidth]{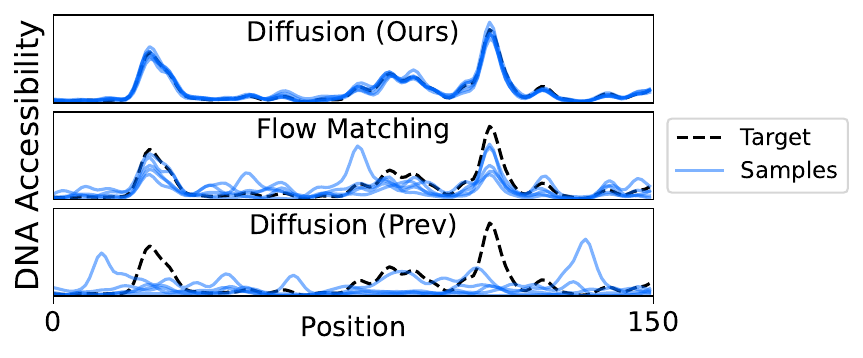}}
        \caption{Example samples}
        \captionsetup{
            justification=raggedright,
            singlelinecheck=false,
            margin={-5cm,0cm}  
        }

        \label{fig: dna examples}
    \end{subfigure}
    \hspace{2cm} 
    \begin{subfigure}[b]{0.3\textwidth}
        \centering
        \includegraphics[width=0.9\linewidth]{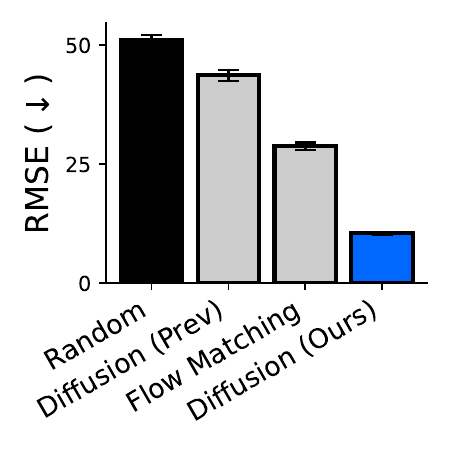}
        \caption{Average error}
        \label{fig: dna msa}
    \end{subfigure}
    \caption{\textbf{Improved simplicial diffusion performs accurate conditional DNA generation.} 
    We generate DNA samples of length 500 conditioned on accessibility with a classifier.
    \textbf{(a)} For an example target, we plot predicted accessibility profiles at the centre 150 positions of 5 example samples from each model.
    We smooth profiles with a bandwidth of 2.
    \textbf{(b)} For 1000 targets and 10 samples from each model, we plot the error between the predicted and target profiles and its standard error.
    }
    \label{fig:cond gen}
    \vspace{-0.35cm}
\end{figure}

\paragraph{State of the art DNA generation conditioned on a classifier}
Simplicial diffusion models are state of the art tools for generating DNA conditioned on high-dimensional epigenetic properties~\citep{Avdeyev2023-mv};
however they have recently been surpassed by flow-matching models \citep{Stark2024-rf}, which are more stable but sacrifice a closed-form ELBO and access to diffusion sampling algorithms.
Given our stability improvements above, we expect to be able to generate higher quality sequences than previous methods.
We fit the state of the art diffusion model ~\citep{Avdeyev2023-mv} and flow-matching model~\citep{Stark2024-rf} to DNA data ($B=4$) of length $D=500$
and generate samples conditioned on achieving target ``DNA accessibility profiles.''

First we see our model leads to a much better fit of the data.
The diffusion model from~\citet{Avdeyev2023-mv}, was only able to achieve an average ELBO of 8 nats / position (12.7 before training), while a trivial model which predict uniform letters in each position achieves 1.39.
In contrast, our model achieves an ELBO of 1.30.
In Fig.~\ref{fig:cond gen} we also see our new model generates conditional samples with profiles that much better match the target.
Experimental details are in App.~\ref{app: experiment-dets}.

\section{Practical unified diffusion models}

Our results show that discrete, Gaussian and simplicial diffusion are three views of the same process.
But which view should a practitioner choose for their particular downstream task?
Unfortunately, there is limited theoretical infrastructure we can use to answer such a question.

Instead our theory provides a practical solution: leveraging our finding that these methods have comparable likelihoods, we show through a particular parameter choice (Fig.~\ref{fig: phi ex}), one can train a single neural network that can perform diffusion on any domain at test time.
In App.~\ref{app: time invariance} we also show this parameterization will also allow us to make any diffusion model time-invariant, explaining and generalizing a celebrated property of masking diffusion.

\subsection{The sufficient statistic parameterization (SSP)}
\begin{minipage}{0.6\linewidth}

The goal of a diffusion model is to predict $x_0^d$.
To do so, one must integrate over the unseen $x_0^{-d}$ weighted by their likelihood of producing the data $x_t^{-d}$:
$$p(x_0^d\mid x_t^{-d})=\int p(x_0^d\mid x_0^{-d})dp(x_0^{-d}\mid x_t^{-d}).$$
We can summarize this ``evidence'' in the normalized vector
$\vec \phi(x_t^{d'}, t)_b\propto p(x_t^{d'}\mid t, x_0^{d'}=b)$ (Fig.~\ref{fig: phi ex}).

Some algebra shows that  $\vec \phi$'s are sufficient statistics, that is, they contain all relevant information about the diffusion process and $t$, leaving a regression task that invariant to both. 
\end{minipage}
\hfill
\begin{minipage}{0.35\linewidth}
\begin{figure}[H]    \centering
    \includegraphics[width=1\linewidth]{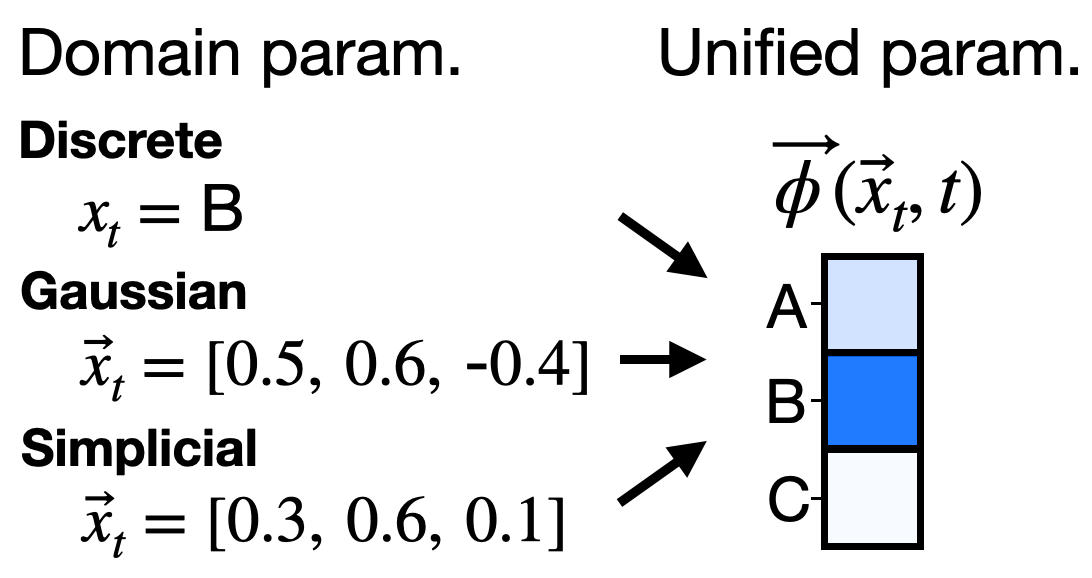}
    \caption{
    \textbf{The sufficient statistic parameterization represents $\vec x_t$ from all diffusion models in the same space.}
    }
    \label{fig: phi ex}
\end{figure}
\end{minipage}

\begin{proposition}\label{prop: ssp}
    (Proof in App.~\ref{app: ssp proof})
    There is a function $F^d$, \textbf{depending on $p(x_0)$ and not on the diffusion process or $t$,}
    such that
    $$p(x_0^d\mid x_t^{-d}, t)=F^d(\vec\phi(\vec x_t^1, t), \dots, \vec\phi(\vec x_t^D, t)).$$
\end{proposition}

Therefore we can parametrize our neural network $q_\theta(x_0^d\mid x_t^{-d}, t)=F^d_\theta(\vec\phi(\vec x_t^1, t), \dots, \vec\phi(\vec x_t^D, t))$ for a neural network $F^d_\theta$ that tries to learn the ``universal'' $F^d$.

\begin{figure}[t]
    \centering
    \begin{subfigure}[b]{0.47\linewidth}
        \centering
        \includegraphics[width=\textwidth]{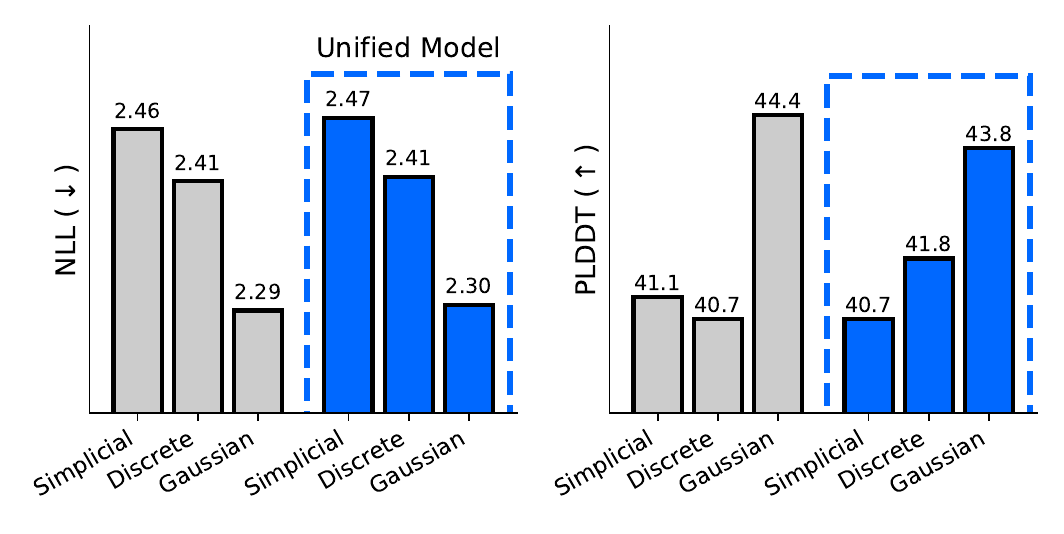}
        \caption{Protein likelihood and sample quality}
    \end{subfigure}
    \hfill
    \begin{subfigure}[b]{0.47\linewidth}
        \centering
        \includegraphics[width=\textwidth]{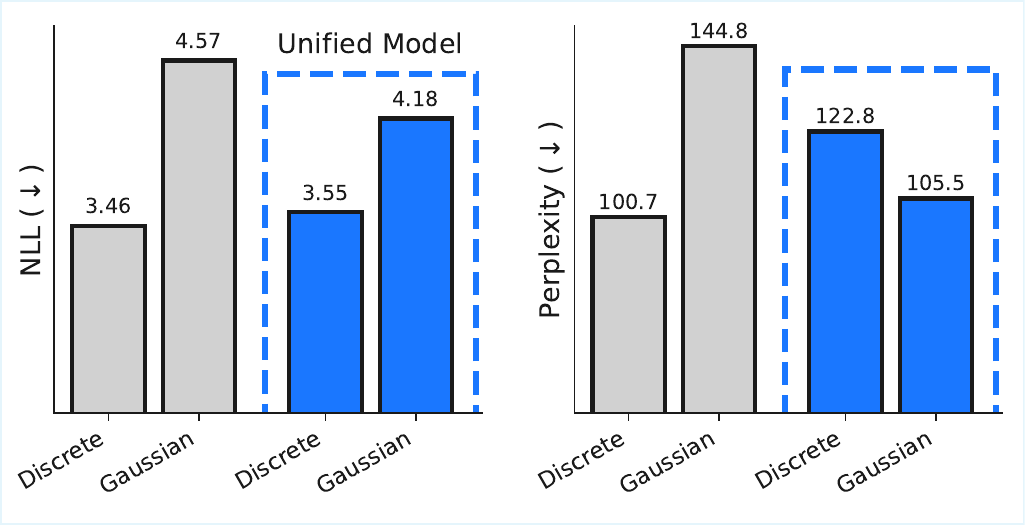}
        \caption{Language likelihood and sample quality}
    \end{subfigure}
    \caption{\textbf{The sufficient statistic parametrization enables a single model to perform competitive discrete, Gaussian, and simplicial diffusion.}
    We compare individual models for each modality with a single unified model using the SSP.
    \textbf{(a)} We train on proteins and measure sample quality by predicted protein fold-ability (pLDDT). Each model was trained for the same amount of time.
    \textbf{(b)} We train on language and measure sample quality using the perplexity of a much larger language model. Each model was trained for 33 epochs.}
    \label{fig: transfer}
\end{figure}

\subsection{Application: unified diffusion models}\label{sec:application unification}
Practitioners must commit upfront to the domain their diffusion occurs.
The SSP instead enables training a single neural network that can perform diffusion on any domain at test time: as long the target distribution $p(x_0)$ remains constant the optimum $F^d$ remains the same.
Furthermore, we've shown above that the ELBOs of each modality are comparable, so we can train $F_\theta$ by alternating minimizing the ELBO of a different modality in each batch.

We train discrete, Gaussian, and simplicial diffusion models on proteins and compare to a single model trained using the SSP which alternates between discrete, Gaussian, and simplicial training steps.
We trained our models to approach the performance of state-of-the-art protein diffusion model DPLM~\citep{Wang2024-bn} in likelihoods (2.36) and a ``foldability'' metric for samples (45.2)~\citep{Amin2025-ag}.
In Fig.~\ref{fig: transfer} we see that a single SSP model trained on proteins for 48 hours is competitive in perplexity and sample quality with three single-domain models each trained for the same amount of time.
We perform a similar experiment for discrete and Gaussian language models (simplicial diffusion models are challenging to scale to a large vocabulary size of $B\approx 3\times 10^4$) and see similar results.
We trained our language models with the same amount of training data as a state-of-the-art diffusion language model~\citep{Lou2023-vm}, matching its likelihood (SEDD uniform has an NLL of 3.70).
Experimental details are in App.~\ref{app: experiment-dets}, another downstream task is tested in App.~\ref{app: thermo experiment}, and we also repeat these results on MNIST images in App.~\ref{app: mnist}.

\section{Conclusion}
Our theoretical and practical unification developed foundations that we used to improve simplicial diffusion and avoid the need to choose a specific model at train time.
However, the theory suggests a number of other directions we have not yet explored.

Notable omissions from our presentation are reflected diffusion~\citep{Lou2023-jc}, flow matching~\citep{Campbell2024-hb}, masking diffusion~\citep{Shi2024-fs}, and diffusion with insertions and deletions~\citep{Johnson2021-gq}.
The later two can likely be easily accommodated with previous theories unifying masking and uniform diffusion on one hand~\citep{Amin2025-ag}, and substitution and insertion - deletion diffusion on the other~\citep{Johnson2021-gq}.

As well, our framework suggests new types of diffusion models ``between'' the three existing streams of diffusion which we only use as a lens for understanding existing models.
Implementing these intermediate models may be of independent practical interest. 

Finally, the SSP can be used to unify models beyond the three modalities.
For instance it can be used to train models across hyperparameter settings, or optimize hyperparameters without retraining.
In principle, the SSP can even be used to transfer a model to a modality it was not trained on.

\section*{Acknowledgements}
This work was supported in part by NSF CAREER IIS-2145492, NSF CDS\&E-
MSS 2134216, NSF HDR-2118310, BigHat Biosciences, and Capital One.

\bibliographystyle{abbrvnat}
\bibliography{bib}

\newpage
\appendix

\section{Extended related work}\label{app: related work}
We add more related work beyond those in Sec.~\ref{sec: related-work}.
\paragraph{Classical theories unifying discrete and continuous stochastic processes}
There is a long history of deriving continuous limits of discrete processes, the ``forward'' processes of diffusion models.
Groundbreaking work by \citet{Stone1963-ui} derived Gaussian diffusion as a limit to biased one-dimensional random walks.
In one of the most celebrated results in mathematical genetics, \citet{Kimura1955-aj} also derived a continuous limit of the Wright-Fisher process with non-zero reproduction. These models were originally developed to describe the stochastic fluctuations of allele frequencies in the absence of selection, also called genetic drift.
We (1) apply these results to understand and improve diffusion models, (2) also show convergence of the ELBO of diffusion models, and, to our knowledge, (3) derive a new result -- the multi-dimensional Gaussian-diffusion limit of Wright-Fisher with zero reproductions -- demonstrating previously un-characterized behaviour dependent on the eigenspace of the mutation operator.
Results (2) and (3) are what allow us to compare likelihoods and hyperparameters.

\paragraph{Parameterizations of discrete diffusion models}
In diffusion, one uses a neural network to ``denoise'' sequences; we call the choice of inputs and outputs of these neural networks the ``\textbf{parameterization}''.
A number of works suggest superficially distinct, but ultimately equivalent parameterizations \citep{Campbell2022-zm, Lou2023-vm}.
\citet{Austin2021-dg} however suggested a distinct parametrization for discrete diffusion models by scaling the output of the neural network to ``automatically'' incorporate the information about the noised token about that particular location; we call this choice the ``hollow'' parametrization for reasons discussed below.
\citet{Zheng2024-ky}, \citet{Ou2024-ms}, and \citet{Sahoo2024-uj} suggested that masking diffusion enables a special choice of ``time-invariant'' parametrization; in App.~\ref{app: time invariance} our theory shows on the contrary that every diffusion model can be made time-invariant.

\paragraph{Gaussian diffusion which appears as simplicial diffusion}
\citet{Han2022-pu, Mahabadi2023-ln, Shabalin2025-ad}, and \citet{Floto2023-mp} suggest a stable diffusion model on the simplex by applying softmax to Gaussian diffusion and using It\^{o}'s theorem.
This parameterization is stable because forward and backward diffusion can occur as Gaussian diffusion in the logit-space.
\citet{Lou2023-jc} has a similar idea, swapping the softmax for an asymmetric transformation and Gaussian diffusion with reflected Gaussian diffusion.
With these simplifications however, the process is exactly (reflected) Gaussian diffusion except the input to the neural network is transformed onto a simplex;
    in particular, it doesn't interact with the topology of the simplex.
In other words, this implements simplicial diffusion in the parametrization of the neural network, but not in the sampling or loss computation.

\paragraph{Another unification theory and a connection to evolution}
\citet{Li2025-go} looked at Gaussian diffusion with a generalized noising strategy; they noted a special case resembled masking diffusion.
However the training procedure and ELBO of this special case are distinct from standard masking diffusion~\citep{Shi2024-fs}.
\citet{zhang2024diffusion} connect diffusion with the evolution process to suggest an optimization algorithm, but do not formally establish connections with biological evolution. In contrast, our work makes this connection explicit. 

\section{Mathematical error in Sahoo et al. (2025)}\label{sec: sahoo error}

In Theorem 3.1, \citet{Sahoo2025-dt} shows that the ELBO of a discrete diffusion model is always tighter than that of a Gaussian diffusion model.
In its proof, with $w_t$ from Gaussian diffusion, $z_t=\mathrm{argmax}(w_t)$, and $x=z_0= w_0$, they state ``Since the transition $z_t\to z_s$ is Markov, we get: $q(z_s\mid w_t,z_t,x)=q(z_s\mid z_t,x)$''.
Putting aside the correctness of this statement, it is clear that the proof as stated requires the Markov property of $(z_t)_t$.

The way the Markov property is shown is as follows.
They first define a discrete diffusion model, let's call this $(\tilde z_t)_t$, such that $\tilde z_0$ comes from the data distribution and $\tilde z$ evolves with respect to a uniform forward process with rate parameter $\beta(t)$ chosen such that the marginals match $p(z_t|z_0)=p(\tilde z_t|\tilde z_0).$
In Eqn. 29 they compute $\frac{d}{dt}p(z_t|z_0)$ and in Eqn. 32 they compute $\frac{d}{dt}p(\tilde z_t|\tilde z_0)$ for all starting points and show they are identical.
After noting the equivalence of equations 29 and 32, they state "This pmf and the ODE are the unique signatures of a Uniform-state discrete diffusion process (Lou et al., 2023; Schiff et al., 2025)." and from this conclude that the path distributions of $(\tilde z_t)_t$ and $(z_t)_t$ are equivalent, and in particular, that $(z_t)_t$ is Markov\footnote{This interpretation of the text was confirmed in personal communication with the first author of \citet{Sahoo2025-dt}}.

However, despite a similar result for Markov chains (two Markov processes with identical semi-groups are equivalent), $p(z_t|z_0)=p(\tilde z_t|\tilde z_0)$ and $\frac{d}{dt}p(z_t|z_0)=\frac{d}{dt}p(\tilde z_t|\tilde z_0)$ for all starting points is not enough to conclude the identity of the path distributions $p((z_t)_t|z_0)=p((\tilde z_t)_t|\tilde z_0)$.
First note that $\frac{d}{dt}p(z_t|z_0)=\frac{d}{dt}p(\tilde z_t|\tilde z_0)$ is not an independent condition: it follows from $p(z_t|z_0)=p(\tilde z_t|\tilde z_0)$.
Next consider this counter example:
\begin{itemize}
\item $\tilde z_0=1$ and $(\tilde z_t)_t$ evolves by switching sign with rate $1$. 
Therefore $p(\tilde z_t=0)=1-\frac 1 2e^{-2t}$.
\item $z_0=1$ and $(z_t)_t$ has a 50\% chance to stay at $0$ forever and a 50\% chance to swap sign at time $-\frac 1 2\log U$ for a $U\sim\mathrm{Uniform}$ and never again.
Therefore $p(\tilde z_t=1)=\frac 1 2(1+p(-\frac 1 2\log\mathrm{Uniform}>t))=1-\frac 1 2e^{-2t}$.
\item When $z_0=-1$ or $\tilde z_0=-1$, then swap signs.
\end{itemize}
We have $p(z_t|z_0)=p(\tilde z_t|z_0)$ for all $z_0$ and therefore $\frac{d}{dt}p(z_t|z_0)=\frac{d}{dt}p(\tilde z_t|z_0)$ but clearly $p((z_t)_t)\neq p((\tilde z_t)_t)$.

Simple computer simulations indeed show that $p((\mathrm{argmax}(w_t))_t)$ and $p((\tilde z_t)_t)$ are different.
We show this in Fig.~\ref{fig:sahoo error}.
Indeed a statistical test applied to these simulations shows $p((\mathrm{argmax}(w_t))_t)\neq p((\tilde z_t)_t)$: a Mann-Whitney test shows that the paths of the argmax of Gaussian diffusion have more transitions that those of discrete diffusion with $p<10^{-300}$.
\begin{figure}[H]
    \centering
    \includegraphics[width=0.75\linewidth]{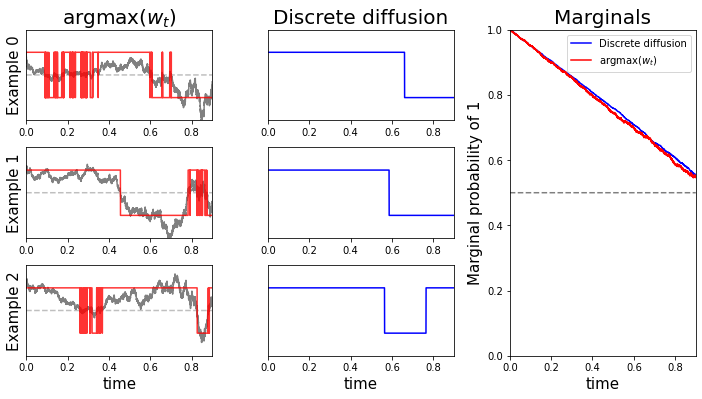}
    \caption{\textbf{The argmax of Gaussian diffusion appears different from discrete diffusion in simulation, despite having the same marginals.}
    We compare example paths of $p((\mathrm{argmax}(w_t))_t)$ (left, red; we show Gaussian diffusion $w_t$ in grey), $p((\tilde z_t)_t)$ for uniform discrete diffusion (centre, blue), and their empirical marginals over 10'000 simulations (right); we simulate using a grid size of 0.0001.
    Note the two processes have the same marginals but their paths appear different; in particular, whenever $w_t$ is near $0$, $(\mathrm{argmax}(w_t))_t$ undergoes a very large number of transitions in a small time\protect\footnotemark.}
    \label{fig:sahoo error}
\end{figure}\footnotetext{Indeed, noting the self-similarity of Brownian motion, one can show that, conditioned on $w_t=0$, with probability $1$ $(z_t)_t$ makes infinitely many transitions in the interval $[t, t+\epsilon)$ for any $\epsilon>0$.
The probability of infinitely many transitions in a bounded interval for discrete diffusion however is $0$.}

\section{Wright-Fisher sampling and score calculations}\label{app: wf details}
Here we discuss the details of the methods in Sec.~\ref{sec: wf}.
Note, just like App.~\ref{app: mid zeta proof}, we can deal with $\vec x_t$ rather than the actual sequences $x_t$.
In App.~\ref{app: simplicial alg details} we discuss details about our algorithms, in particular how to sample form $A(\psi, \tau_t)$ and calculate the functions $\vec s(\vec v\mid x_0).$
In App.~\ref{app: wf computational complexity} we discuss the computational complexity and stability of these algorithms in theory and in experiment.
In App.~\ref{app: wf details low t} we discuss how we sample and calculate the ELBO at low $t$.
Finally, in App.~\ref{app: sampling} we discuss our condiitonal sampling procedure.

We also note two more differences between our method and that of \citet{Avdeyev2023-mv}: 
\begin{itemize}
    \item Our neural network directly predicts $\vec x_0$ rather than the ``score'' $\vec s$.
    \item As described in App.~\ref{app: wf proof fwd}, we use the natural permutation-symmetric ``multi-allelic'' extension to the 1D SDE when $B>2$, while they use a stick-breaking procedure.
    \item We use high-precision operations to calculate large alternating series accurately, as described in App.~\ref{app: wf computational complexity}.
\end{itemize}

\subsection{Algorithm details}\label{app: simplicial alg details}
\paragraph{Sample noisy $x_t$}
We've discussed the algorithm from \citet{Jenkins2017-ry} in the main text. 
We now present their algorithm for sampling from $A(\psi, \tau_t).$

\begin{algorithm}[H]
\caption{Exact sampling from ancestral process $A(\psi, \tau_t)$}\label{alg: jenkins general}
\begin{algorithmic}[1]
\State Define coefficients: $c_{km}^{\psi} = \frac{(2k+\psi-1)(\psi + m)_{(k-1)}}{m!(k-m)!} e^{-k(k+\psi-1)\tau_t/2}$ for $k \geq m$
\State Sample $U \sim \mathrm{Uniform}[0, 1]$
\State Initialize $M \leftarrow 0$
\State Initialize an empty vector $\vec k = ()$
\While{True}
    \State Find $k_M\geq M$ such that $c_{(k_M+1)M}^{\psi} < c_{k_MM}^{\psi}$
    \State Make $k_M$ even: $k_M\gets2\lceil k_M/2\rceil$
    \State Update lower bound: $S^- \leftarrow S^- + \sum_{k=M}^{k_M+1} (-1)^{k-M} c_{kM}^{\psi}$
    \State Update upper bound: $S^+ \leftarrow S^+ + \sum_{k=M}^{k_M} (-1)^{k-M} c_{kM}^{\psi} $
    \State Update $\vec k = (k_0, \dots, k_{M-1}, k_{M})$
    \While{$S^- < U<S^+$}
        \State Update lower bound: $S^- \leftarrow S^- + \sum_{m=0}^M(c_{(k_m+2)m}^{\psi}-c_{(k_m+3)m}^{\psi}) $
        \State Update upper bound: $S^+ \leftarrow S^- +\sum_{m=0}^M(-c_{(k_m+1)m}^{\psi}+c_{(k_m+2)m}^{\psi}) $
        \State Update $\vec k = \vec k +(2, \dots, 2)$
    \EndWhile
    \If{$S^- > U$}
        \State \Return $m = M$
    \ElsIf{$S^+ < U$}
        \State $M \leftarrow M + 1$
    \EndIf
\EndWhile
\end{algorithmic}
\end{algorithm}

\paragraph{Compute loss}
We present a formula for $\vec s(\vec v\mid x_0, t)=\nabla\log p(\vec x_t\mid x_0, t)|_{\vec v}$ to enable computation of the loss.
\citet{Avdeyev2023-mv} computed these scores using a previously determined result with $B=2$ then generalizing to higher dimensions with their stick-breaking procedure and a change of variables.
We are instead able to derive it directly from first principles.

There are two infinite series which will be important,
\begin{gather*}
    G_\psi(\tau, x_0, \vec x_t)=1+\sum_{k=1}^\infty(-1)^ka_k^\psi(\tau, \pi_{x_0}, \vec x_{t, x_0})\\
    F_\psi(\tau, x_0, \vec x_t)=1+\sum_{k=1}^\infty(-1)^kb_k^\psi(\tau, \pi_{x_0}, \vec x_{t, x_0})\\
\end{gather*}
where 
\begin{gather*}
    a_k^\psi(\tau, \pi_{x_0}, \vec x_{t, x_0}) = e^{-\frac{k(k+\psi-1)\tau}{2}}\frac{(2k+\psi-1)(\psi)_{(k-1)}}{k!}{_2F_1(-k,\psi+k-1;\psi\pi_{x_0};\vec x_{t, x_0})}\\
    b_k^\psi(\tau, \pi_{x_0}, \vec x_{t, x_0}) = e^{-\frac{k(k+\psi+1)\tau}{2}}\frac{(\psi)_{(k)}}{k!}\frac{(2k+\psi+1)(\psi+k)}{(\psi+1)\psi}{_2F_1(-k,\psi+k+1;\psi\pi_{x_0}+1;\vec x_{t, x_0})}
\end{gather*}
where $_2F_1$ is the hypergeometric function.
Although these look complicated, in practice, most terms in the numerators and denominator of $a$ and $b$ nearly cancel to $1$, and, when $t$ is not too small, $e^{-k(k+\psi+1)\tau/2}$ decays extremely quickly.

Using the results in~\citet{Tavare1984-uz} we compute $\vec s(\vec v\mid x_0)$ in terms of these series.
Since we're only interested in differences for calculating the ELBO,
$\vec s(\vec v\mid x_0, t) - \vec s(\vec v\mid \tilde x_0, t)$ we ignore constants not depending on $x_0$.
\begin{proposition}\label{prop: wf scores}
    (Proof in App.~\ref{app: wf scores proof})
    $$p(\vec x_t\mid x_0, t)=\mathrm{Dirichlet}(\pi\psi)(\vec x_t)G_\psi(\tau_t, x_0, \vec v).$$
    For $\vec c(\vec v)=\nabla\log \mathrm{Dirichlet}(\pi\psi)(\vec x_t)=(\psi\vec\pi-\mathbbm 1)/\vec x_t$ which does not depend on $x_0$,
    \begin{equation}\label{eqn:WF score}
        \vec s(\vec v\mid x_0, t)=\vec c(\vec v)+\vec x_0w(x_0, \vec v)
    \end{equation}
    where
    $$w(x_0, \vec v) = \frac{e^{-\psi\tau_t/2}(\psi+1)}{\pi(x_0)}\frac{F_\psi(\tau_t, x_0, \vec v)}{G_\psi(\tau_t, x_0, \vec v)}.$$
\end{proposition}

With the hollow parameterization,
calling $\vec w_b=w(b)$,
we get
$$\vec s(\vec v\mid \tilde x_0, t)_b=\vec c(\vec v)_b+\frac{e^{-\psi\tau_t/2}(\psi+1)}{\pi(x_0)}\frac{\tilde x_{0, b}F_\psi(\tau_t, b, \vec v)}{\sum_{b'}\tilde x_{0, b}G_\psi(\tau_t, b', \vec v)}.$$

\subsection{Computational complexity and stability}\label{app: wf computational complexity}

\paragraph{Numerical stability}
Sampling from $A(\psi, \tau_t)$ and calculating $G_\psi$ and $F_\psi$ involve alternating series of many terms which vary by many orders of magnitude, and cancel out leaving very small residuals -- known as ``catastrophic cancellation''.
To calculate these accurately, we may need higher precision than provided by \texttt{float64};
we perform any high precision calculations using the \texttt{mpmath} library~\citet{Johansson2010-fq}.
\citet{Avdeyev2023-mv} did not use high precision in their calculations, potentially introducing errors and instability to their loss computation.

We perform all calculations at \texttt{float64} to take advantage of parallel GPU computations, estimate the error of each computation using a \textit{condition number} and recompute just those terms with condition number above a threshold in \texttt{mpmath} on a CPU.
In practice, we only need to perform calculations at high precision for small $t$, before we switch to the ``low time regimen'' $\tau_t<0.05$ where we switch tot he Griffiths approximation.

The condition number of a series $a_1+a_2+\dots+a_M$ is defined as $\eta = \sum_m |a_m|/|\sum_ma_m|$;
    one can estimate the error of their summation at finite precision by 
    $$\mathrm{error}\approx \eta\times\mathrm{precision}.$$
When there is catastrophic cancellation, the denominator in the definition of $\eta$ will be very large, representing that error might be high.
To estimate $\eta$, we keep track of $\tilde \sum_m |a_m|$ ($\tilde\sum$ representing our finite-precision summation) and estimate $\eta \approx \tilde \sum_m |a_m|/|\tilde \sum_ma_m|$.
If $\eta > \text{desired error}/ \text{float 64 precision}=10^{-6}\times 2^{52}$, then we recompute at higher precision.

\paragraph{Sampling}
The complexity for sampling $\vec x_t$ involves (1) $O(m)$ for sampling $m$, which is $O(1/\tau_t)$ in expectation (see App.~\ref{app: wf details low t}), and (2) $O(B)$ time for sampling $\vec x_t$ from a Dirichlet.
Crucially, the complex calculations involving an infinite series occur in (1) and are independent of the alphabet size $B$.
Comparatively, sampling from the SDE requires $O(BT/\Delta t)$ compute.
A higher $\Delta t$ will decrease compute but lead to lower-fidelity samples, especially at low $t$ where even small fluctuations in $x_t$ can lead to instability.
We also parallelize the computations in Alg.~\ref{alg: jenkins general} to benefit from GPU acceleration. 
The result is that, except when we must switch to high precision, our sampling procedure is much faster than that using an SDE (Fig.~\ref{fig:fast samples}), and much more stable at low $\tau$ (Fig.~\ref{fig:stable samples}).

\paragraph{Loss computation}
The complexity for computing the loss involves (1) $O(B k_{\max})$ computations for the series $F_\psi$ and $G_\psi$,
and (2) $O(B)$ computations for computing the loss given the vectors $\vec s$.
Crucially, the complex calculations involving an infinite series occur in (1) and can be parallelized across $B$ allowing massive acceleration on GPU.
$k_{\max}$ should become very large as $t$ becomes small, leading \citep{Avdeyev2023-mv} to choose $k_{\max}=1000$.
Instead we only use the series computation for $\tau\geq 0.05$, allowing us to use $k_{\max}=80$ (Fig.~\ref{fig:stable loss}), and compute a $O(B)$ ELBO for $\tau<0.05$ in App.~\ref{app: wf details low t}.

\begin{figure}
    \centering
    \begin{subfigure}[b]{0.25\textwidth}
        \centering
        {\includegraphics[width=\linewidth]{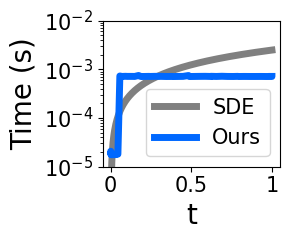}}
        \caption{Fast sampling}
        \label{fig:fast samples}
    \end{subfigure}
    \begin{subfigure}[b]{0.3\textwidth}
        \centering
        \includegraphics[width=\linewidth]{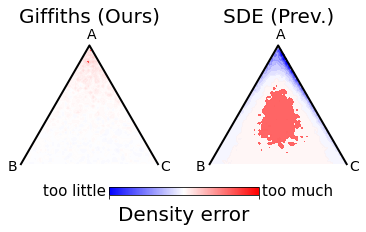}
        \caption{Stable sampling}
        \label{fig:stable samples}
    \end{subfigure}
    \begin{subfigure}[b]{0.41\textwidth}
        \centering
        \includegraphics[width=\linewidth]{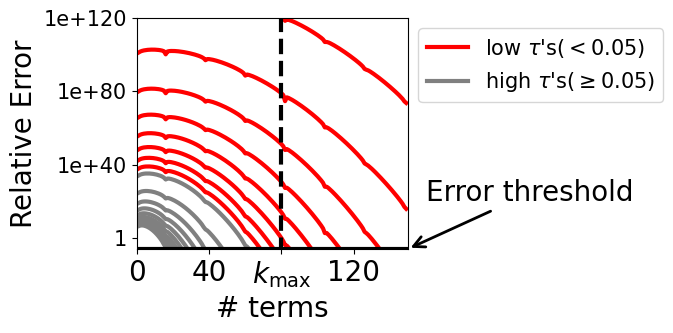}
        \caption{Fast and stable loss calculation}
        \label{fig:stable loss}
    \end{subfigure}
    \caption{\textbf{Leveraging mathematical genetics literature, we build fast and stable simplicial diffusion.} 
    \textbf{(a)} We plot the time it takes to sample a sequence of $D=500$ using an SDE, versus our exact sampling for various values of $t$ on an A100 80GB GPU.
    We threshold switching to the Griffiths approximation at $\tau_t=0.1$.
    \textbf{(b)} For $\tau=0.1$ and $B=3$ we sample $3\times 10^7$ points from the exact sampling method, Griffith's approximation, and using an SDE with 25 steps as used in ~\citet{Avdeyev2023-mv}.
    We then perform density estimates of these data and plot the error to the exact samples.
    We plot a $\times 6$ zoom into the vertex $A$.
    We see the SDE struggles to sample near the corner.
    We use $\psi=3, \vec\pi=[0.25, 0.4, 0.35]$
    \textbf{(c)} We plot the accuracy of our approximation of the infinite series $G_\psi(\tau, x_0, \vec x_t)$ including different numbers of terms for various values of $\tau\in[0, 0.2]$.
    We choose $\psi=B=4$ and $\vec \pi$ uniform, and plot the relative error for two values of $\vec x_{t, x_0}.$
    We only use the series approximation for $\tau\geq0.05$ (grey), which allows us to only use 80 terms.
    Meanwhile~\citep{Avdeyev2023-mv} used 1000 terms to accommodate smaller $\tau$ (red).
    Our error threshold is $10^{-6}$.
    }
    \label{fig:wf improve}
\end{figure}

\subsection{Low time regimen}\label{app: wf details low t}

When $t$ is small, sampling from $A(\psi, \tau_t)$ or calculating $G_\psi, F_\psi$ become unstable and can require unbounded compute.
\citet{Griffiths1984-bo} suggested a Gaussian approximation for $A(\psi, \tau_t)$ which we will also use for deriving stable approximations of $\vec s(\vec v\mid x_0, t)$ which require bounded compute.

\paragraph{Sample noisy $x_t$}
We copy the following from~\citet{Jenkins2017-ry}.
Note compute does not scale with $\tau_t$.
\begin{algorithm}[H]
\caption{Sampling from ancestral process $A(\psi, \tau_t)$ - Low $t$ approximation}\label{alg: jenkins low t}
\begin{algorithmic}[1]
\State Set $\beta \leftarrow \frac{1}{2}(\psi - 1)\tau_t$
\If{$\beta \neq 0$}
    \State Set $\eta \leftarrow \beta/ (e^{\beta}-1)$
    \State Set $\mu \leftarrow \frac{2\eta}{\tau_t}$
    \State Set $\sigma^2 \leftarrow \frac{2\eta}{\tau_t}\left(\eta + \beta\right)^2\left(1 + \frac{\eta}{\eta+\beta} - 2\eta\right)\beta^{-2}$
\Else
    \State Set $\mu \leftarrow \frac{2}{\tau_t}$
    \State Set $\sigma^2 \leftarrow \frac{2}{3\tau_t}$
\EndIf
\State Sample $Z \sim \mathcal{N}(\mu, \sigma^2)$
\State \Return $m = \max(0, \lfloor Z + 0.5 \rfloor)$ \Comment{Round to nearest non-negative integer}
\end{algorithmic}
\end{algorithm}

\paragraph{Compute loss}
The loss in this regimen, even with the Griffiths approximation, becomes intractable; instead we use the Griffiths approximation to simply bound the loss.

When $t$ is small, $x_0$ is almost always $b^*=\mathrm{argmax}_{b}\vec x_{t, b}$.
We therefore set $\tilde x_0=\delta_{b^*}$.
In practice, in our protein setting, we only see $b^*\neq\mathrm{argmax}_{b}\vec x_{t, b}$ with $\tau_t<0.05$ at a rate of less than 1 in $7\times 10^7$.
Since $x_0\neq b^*$ is so rare we only aim to find a loose bound.
Calling $\vec v=\vec x_{t}$ we bound the loss by 
\begin{equation*}
    \begin{aligned}
    L\leq &\frac{\dot\tau_t}{2}(\|\vec s(\vec v\mid x_0, t)-\vec c(\vec v)\|_{\text{Diag}(\vec{v})-\vec{v}\vec{v}^T}+\|\vec s(\vec v\mid b^*, t)-\vec c(\vec v)\|_{\text{Diag}(\vec{v})-\vec{v}\vec{v}^T})^2\\
    =&\frac{\dot\tau_t}{2}(w(x_0, \vec v)\sqrt{\vec v_{x_0}}+w(b^*, \vec v)\sqrt{\vec v_{b^*}})^2.
    \end{aligned}
\end{equation*}
In the next proposition we give an alternate formula for $w(x_0, \vec v)$ which will allow us to Griffith's approximation and a saddle point approximation to estimate $w(b^*, \vec v)$.
It will also allow us to bound $w(x_0, \vec v).$
To our knowledge, this strategy is original.

\begin{proposition}
    $$w(x_0, \vec v)=\vec v_{x_0}^{-1}\tilde {\mathbb {E}}_{\vec v_{x_0}, \vec \pi_{x_0}}m_t$$
    where $\tilde {\mathbb {E}}_{\vec v_{x_0}, \vec \pi_{x_0}}$ is over the weighted, normalized distribution
    $p(A(\psi, \tau_t)=m_t)\frac{(\psi)_{(m_t)}}{(\psi\pi_{x_0})_{(m_t)}}\vec v_{x_0}^{m_t}$.
\end{proposition}
\begin{proof}
    Inspection of first expression of the proof of Prop.~\ref{prop: wf scores}.
\end{proof}

We can now bound $w(x_0, \vec v)\vee w(b^*, \vec v)\leq \vec v_{x_0}^{-1}\tilde {\mathbb {E}}_{1, p}m_t$ where $p=\vec\pi_{x_0}\wedge\vec\pi_{b^*}.$
Therefore, we get
\begin{equation}\label{eqn: L bound}
    L\leq 2{\dot\tau_t}\vec v_{x_0}^{-1}(\tilde {\mathbb {E}}_{1, p}m_t)^2.
\end{equation}
Now we only need to calculate $\tilde {\mathbb {E}}_{1, p}m_t$;
to do so we apply a saddle point approximation to Griffith's approximation to get Eqn.~\ref{eqn: small t approx} below.
Note compute does not scale with $\tau_t.$

\paragraph{Saddle point approximation}
Let's take the Griffiths approx as $t$ becomes small, so $w_t\sim N(\mu, \sigma)$ where $\mu, \sigma$ are form Alg.~\ref{alg: jenkins low t}.
We can use a Stirling approximation to get

\begin{equation*}
    \begin{aligned}
        \frac{(\psi)_{(m_t)}}{(\psi p)_{(m_t)}}\propto&(1+O(1/m))\left(1+\frac{(1-p)\psi}{\psi p+m_t-1}\right)^{(\psi p+m_t-1)+1/2}(\psi+m_t-1)^{(1-p)\psi}\\
        =&(1+O(1/m))(\psi+m_t-1)^{(1-p)\psi}.
    \end{aligned}
\end{equation*}

We take a saddle point approximation of $\tilde{ \mathbb{E}}_{1, p}m_t$, i.e. take its value as the maximizer of the approximate log likelihood 
\begin{equation*}
    \begin{aligned}
        &C-\frac 1 {2\sigma^2}(m_t-\mu)^2+(1-p)\psi\log(\psi+m_t-1)+O(1/m_t).
    \end{aligned}
\end{equation*}
We therefore get the approximation
\begin{equation}\label{eqn: small t approx}
\tilde{ \mathbb{E}}_{1, p}m_t \approx \left((\mu-(\psi-1))+\sqrt{(\mu+(\psi-1))^2+4(1-p)\psi\sigma^2}\right)/2.
\end{equation}
Noting $m_t\sim \tau_t^{-1}$, this approximation has relative error roughly $O(\tau_t^2)$.
And as $\tau_t\to 0$, $\mu\sim\tau_t^{-1}$ and $\sigma\sim\tau^{-1/2}$ so
\begin{equation*}
\tilde{ \mathbb{E}}_{1, p}m_t\approx \mu\approx 2/\tau_t.
\end{equation*}

\subsection{Time reversal SDE}\label{app: sampling}

Reversing the SDE Eqn.~\ref{eqn: wf sde} using the result of \citet{Anderson1982-fo}, we get
\begin{equation*}
    \begin{aligned}
     d\vec z_\tau=&\left(\frac{\psi}{2}(\vec\pi-\vec z_\tau)-B(\mathbbm 1/B-\vec z_\tau)-\left({\mathrm{diag}(\vec z_\tau})-{\vec z_\tau}{\vec z_\tau}^T\right)\nabla\log p(\vec x_t\mid t)\right)d\tau\\
     &+\mathrm{diag}\left(\sqrt{\vec z_\tau}\right)\left(I-\sqrt{\vec z_\tau}\sqrt{\vec z_\tau}^T\right)d\vec W_\tau
    \end{aligned}
\end{equation*}
where $\vec z_{\tau_t}=\vec x_t$.
$\vec s(\vec x_t\mid \tilde x_0, t)$ approximates
$\mathbb E_{x_0\mid \vec x_t}\log p(\vec x_t\mid x_0, t)=\nabla\log p(\vec x_t\mid t)$, meaning we can substitute it into the place of $\nabla\log p(\vec x_t\mid t)$.
We sample by discretizing this SDE and sampling backwards.

To perform classifier guidance conditioning on a variable $y$, we can add $\nabla \log p(y\mid x_t)$ to $\vec s(\vec x_t\mid \tilde x_0, t).$
In practice, we perform the classic one-step approximation
$\nabla\log p(y\mid x_t)\approx \nabla\log \mathbb E_{x_0\sim q_\theta(x_0\mid x_t, t)}p(y\mid x_0).$
If we have a classifier $f(x_0)=p(y\mid x_0)$ then we approximate $E_{x_0\sim q_\theta(x_0\mid x_t, t)}p(y\mid x_0)$ using the ``one-shot'' prediction $f(\tilde x_0)$.

\section{Time-invariant discrete diffusion models}\label{app: time invariance}

\citet{Zheng2024-ky}, \citet{Ou2024-ms}, and \citet{Sahoo2024-uj} noted that for masking diffusion, it is not necessary to pass $t$ to the neural network -- it has ``time-invariant" parametrization.
\citet{Zheng2024-ky} suggests this makes masking models a fundamentally different object than other diffusion models: ``we reveal that both training and sampling of [masked models] are theoretically free from the time variable, arguably the key signature of diffusion models, and are
instead equivalent to masked models.''
Our sufficient-statistic parameterization shows on the contrary that every diffusion model can be made time-invariant by a choice of parameterization, with masking as a special case.

Does this suggest that every diffusion may perform as well as masking diffusion after this choice of parameterization?
\citet{Amin2025-ag} suggests that masking performs well not because of its choice of parameterization, but because of ``schedule conditioning''.

\textbf{Time-invariance is a function of parameterization:}
Masking is time-invariant due to a choice of parametrization.
To see this, imagine applying a time-dependent rotation to each $x_t^d$;
we are essentially performing the same diffusion but now must also pass $t$ to $q_\theta$ so it can ``undo'' the transformation.
The $\vec\phi$ can be thought of as automatically transforming $x_t$ so $F^d$ is independent of time in any diffusion model.

\textbf{Masking uses SSP: }
Indeed the SSP of masking diffusion,
$\vec \phi( x_t^d, t)=\delta_{x_t}$ if $x_t\neq\mathrm{mask}$ and $\vec \phi( x_t^d, t)=[\frac 1 B, \dots, \frac 1 B]$ otherwise,
is exactly the canonical parametrization.
Thus the time-invariance of masking isn't special -- rather masking's most convenient parametrization happens to be the SSP.

\section{Theoretical results}

\subsection{Mutation population discrete diffusion loss}\label{app: mid zeta proof}
In this appendix we derive Alg.~\ref{alg: zeta arb} by showing it is equivalent to Alg.~\ref{alg: discrete diffusion}.
Namely, we assume $D=1$ and $x_t$ is a sequence of length $\zeta$ 
and show
\begin{itemize}
    \item \textbf{Predict de-noised $x_0$:} the target of $q_\theta(x_0\mid x_t, t)$, $p(x_0\mid x_t, t)$, only depends on the vectorized $\vec x_t$.
    \item \textbf{Compute loss: }$L = \sum_{x'\neq x_t}\mathcal L_{x'\to x_t} \dot\tau_t\mathbb{D}\left(\frac{p(x'\mid x_0, t)}{p(x_t\mid x_0, t)}\bigg|\bigg|\frac{p(x'\mid \tilde x_0, t)}{p(x_t\mid \tilde x_0, t)}\right)$ is equivalent to the form in Alg~\ref{alg: zeta arb}.
\end{itemize}
Given prediction and loss computation only depend on $\vec x_t$, we can also replace sampling $x_t$ with just sampling $\vec x_t\sim\mathrm{Mult}(\zeta, \vec x_0^Te^{\tau_t\mathcal L})/\zeta$, giving Alg.~\ref{alg: zeta arb}.

\paragraph{Predict de-noised $x_0$}
Simply note
\begin{equation*}
    \begin{aligned}
        p(x_0\mid x_t, t)\propto& p(x_0)p(x_t\mid x_0, t)\\
        =&p(x_0)\prod_{z=0}^\zeta(\vec x_0^Te^{\tau_t\mathcal L})_{x_t^{(z)}}\\
        =&p(x_0)\prod_{b=1}^B(\vec x_0^Te^{\tau_t\mathcal L})_{b}^{\zeta\vec x_{t, b}}.
    \end{aligned}
\end{equation*}

\textbf{Compute loss}
For sequences $x\neq x$ of length $\zeta$ which differ in exactly one position, say $x^{(z)}=b\neq b'=x^{\prime(z)}$, then
$\mathcal L_{x\to x'}=\mathcal L_{b\to b'}$
and for every $x_0$
$$\frac{p(x'\mid x_0, t)}{p(x\mid x_0, t)}=\frac{\vec x_0e^{\tau_t\mathcal L}\vec b'}{\vec x_0e^{\tau_t\mathcal L}\vec b}.$$
If $x, x'$ differ in more than one position, then $\mathcal L_{x\to x'}=0$.
Call $x_t^{[z,b]}$ a sequence which has all the same letters as $x_t$ except has $b$ in position $z$.
Then calling $\vec p=\vec x_0^Te^{\tau_t\mathcal L}$ and $\vec q=\tilde x_0^Te^{\tau_t\mathcal L}$,
\begin{equation*}
    \begin{aligned}
        L =& \sum_{x'\neq x_t}\mathcal L_{x'\to x_t} \dot\tau_t\mathbb{D}\left(\frac{p(x'\mid x_0, t)}{p(x_t\mid x_0, t)}\bigg|\bigg|\frac{p(x'\mid \tilde x_0, t)}{p(x_t\mid \tilde x_0, t)}\right)\\
        =&\sum_{z=0}^\zeta\sum_{b'\neq x_t^{(z)}}\mathcal L_{b'\to x_t^{(z)}} \dot\tau_t\mathbb{D}\left(\frac{\vec p_{b'}}{\vec p_{x_t^{(z)}}}\bigg|\bigg|\frac{\vec q_{b'}}{\vec q_{x_t^{(z)}}}\right)\\
        =&\sum_{b}\#\{z\mid x_t^{(z)}=b\}\sum_{b'\neq b}\mathcal L_{b'\to b} \dot\tau_t\mathbb{D}\left(\frac{\vec p_{b'}}{\vec p_{b}}\bigg|\bigg|\frac{\vec q_{b'}}{\vec q_{b}}\right)\\
        =&\sum_{b'\neq b}\mathcal L_{b'\to b} \dot\tau_t\zeta\vec x_{t, b}\mathbb{D}\left(\frac{\vec p_{b'}}{\vec p_{b}}\bigg|\bigg|\frac{\vec q_{b'}}{\vec q_{b}}\right).
    \end{aligned}
\end{equation*}

\subsection{Proof of Gaussian convergence}\label{app: gaussian proof}
Our formal statement of the theorem adds some mild positivity assumptions for $\tau$, $\pi$ and $P_1$ which are satisfied by any reasonable choice of $\tau$ and almost every choice of $\mathcal L$.
It is also more specific about the limiting behaviour of $\vec x^\zeta_t$ in non-dominant eigenspaces: we also limit to Gaussian diffusion, but with meaningless embeddings sampled from random Gaussian vectors independent of $x_0$.

Let us interpret the embedding $Q_1$.
In the case that $\mathcal L$ is doubly stochastic, or reversible, $\pi = [\frac 1 B, \dots, \frac 1 B]$ and $\mathcal L$ is symmetric; in this case $Q_1=\mathfrak{j}_1P_1$ is just the orthogonal projection onto the dominant eigenspace.
In the more general case that $\mathcal L$ satisfies detailed balance, $(\diag(\pi)^{1/2}\mathcal L\diag(\pi)^{-1/2})_{ij}=\sqrt{\frac{\pi_i}{\pi_j}}\mathcal L_{ij}$ is symmetric so $\tilde Q_1$ is the orthogonal projection onto the dominant eigenspace of the ``symmetrized'' generator.

In more general cases, we don't get a symmetrized operator or an orthogonal projection $\tilde Q_1$, so we must ``correct'' for this with the adjustment $(\tilde Q_1\tilde Q_1^T)^{-1/2}\tilde Q_1$ which makes $Q_1^TQ_1$ an orthogonal projection.

\begin{theorem}\label{thm: gaussian proof}
    (Formal statement and proof of Thm.~\ref{thm: gaussian})
    Call $-\lambda_1>-\lambda_2>\dots$ the negative eigenvalues of $\mathcal L$ and $P_1, P_2, \dots$ the projections onto the corresponding left eigen-space.
    Without loss of generality, assume $\lambda_1=1$.
    Assume $\dot\tau_t$ is bounded on every compact interval of $(0, 1)$, $\pi_b>0$ and $P_1\vec b\neq 0$ for all $b$ and $P_1\vec b\neq P_1\vec b'$ for any $b\neq b'$.
    For each $\zeta$ pick time dilation $\tau^\zeta_t=\frac 1 {2}\log\left(\zeta e^{2\tau_t}-\zeta+1\right)$ and rescale $\vec x_t^\zeta = \sqrt{\zeta-(\zeta-1)e^{-2\tau_t}}(\vec{x_t}-\pi)/\sqrt{\pi}$.
    Define the embedding into $\mathbb R^{\mathrm{rank}(P_i)}$, $Q_i=\mathfrak j_i(\tilde Q_i\tilde Q_i^T)^{-1/2}\tilde Q_i$ where $\tilde Q_i=\diag(\pi)^{-1/2}P_i\diag(\pi)^{1/2}$ and $\mathfrak j_i$ is any isometry from $\mathrm{Im}(\tilde Q_i)\to\mathbb R^{\mathrm{rank}(P_i)}$.

    Fix an $x_0.$
    \begin{itemize}
        \item (Path convergence) Call $(\vec z_t)_{t=0}^1$ the paths with
        $\vec z_0=Q_1(\vec x_0/\sqrt\pi)$ evolving under the Ornstein-Uhlenbeck process
        $$d\tilde z_\tau=-\tilde z_\tau d\tau+\sqrt{2}dW_\tau$$
        for a Brownian motion $(W_\tau)_{\tau=0}^\infty$ and call $\vec z_t=\tilde z_{\tau_t}.$
        Then $(Q_1\vec x_t^\zeta)_{t\in(0, 1)}$ converges in distribution to  $(\vec z_t)_{t\in(0, 1)}$ in the sense of Lem.~\ref{lem: convergence}.
        \item (Convergence of non-dominant directions) 
        The component of $\vec x^\zeta_t$ in $\mathrm{Ker}Q_1$ is $\sum_{i>1}\tilde Q_i\vec x^\zeta_t$.
        Each component $(Q_i\vec x^\zeta_t)_t$ also converges to a Gaussian diffusion independent of $\vec x_0$ with modified time-dilation and scaling:
        call $(\vec z_t)_{t=0}^1$ the paths with
        $\vec z_0\sim\mathcal N(0, I)$ independent of $x_0$ evolving, forward and backward on $(-\infty, \infty)$, under the stationary Ornstein-Uhlenbeck process
        $$d\tilde z_\tau=-\tilde z_\tau d\tau+\sqrt{2}dW_\tau$$
        for a Brownian motion $(W_\tau)_{\tau=0}^\infty$ and call $\vec z_t=\tilde z_{\tau_t^{(i)}}$ where $\tau_t^{(i)}=\frac {\lambda_i} 2\log(e^{2\tau_t}-1)$.
        Then 
        $$((1-e^{-2\tau_t})^{-1/2}Q_i\vec x_t^\zeta)_{t\in(0, 1)}\leadsto(\vec z_t)_{t\in(0, 1)}.$$
        \item Call the ELBO in Alg.~\ref{alg: zeta arb} 
        $$L(\vec x^\zeta_t, t, \vec x_0, \tilde x_0)
        = \sum_{b_1\neq b_2}\mathcal L_{b_2\to b_1} \dot\tau_t^\zeta \zeta \vec x_{t, b_1}(\vec x^\zeta_t) \mathbb{D}\left(\hat w(x_0)_{b_2, b_1}||\hat w(\tilde x_0)_{b_2, b_1}\right)$$
        where $\vec x_{t, b_1}(\vec v)$ is the inverse of the transform from $\vec x_{t, b_1}$ to $\vec x_{t, b_1}^\zeta$.
        Then, for all $\vec v, t, \vec x_0, \tilde x_0$
        $$L(\vec v, t, \vec x_0, \tilde x_0)\to \frac{\dot\tau_t e^{-2\tau_t}}{(1-e^{-2\tau_t})^2}\left\|\mathrm{emb}(x_0)-\mathrm{emb}(\tilde x_0)\right\|^2,$$
        the ELBO in Alg.~\ref{alg: Gaussian diffusion}, which, in particular, is independent of the value of $\vec v$.
    \end{itemize}
\end{theorem}
\begin{proof}
    We prove the convergence of paths using Lem.~\ref{lem: convergence} which makes use of standard techniques.
    We break the proof up into four sections: the first three verify the conditions of Lem.~\ref{lem: convergence} and the last shows the convergence of the ELBO.

    \textbf{Part 1. Convergence of Marginals:}
    Note
    $$\vec z_t\sim e^{-\tau_t}\vec z_0+\sqrt{1-e^{-2\tau_t}}\mathcal N(0, I).$$
    We want to prove convergence to this quantity.
    Note, writing $\mathrm{Mult}$ for a multinomial distribution,
    \begin{equation*}
    \begin{aligned}
    \vec x_t^\zeta\sim&\frac{\sqrt{\zeta-(\zeta-1)e^{-2\tau_t}}}{\zeta}\left(\mathrm{Mult}(\zeta, \vec x_0^Te^{\tau_t^\zeta\mathcal L})-\zeta\vec\pi\right)/\sqrt{\vec \pi}\\
    =&(1+o(1))\sqrt{1-e^{-2\tau_t}}\left(\zeta^{-1/2}(\mathrm{Mult}(\zeta, \vec x_0^Te^{\tau_t^\zeta\mathcal L})-\vec x_0^Te^{\tau_t^\zeta\mathcal L})+\zeta^{1/2}(\vec x_0^Te^{\tau_t^\zeta\mathcal L}-\vec\pi)\right)/\sqrt{\vec \pi}.
    \end{aligned}
    \end{equation*}
    The second term is
    \begin{equation*}
    \begin{aligned}
    \sqrt{1-e^{-2\tau_t}}\zeta^{1/2}(\vec x_0^Te^{\tau_t^\zeta\mathcal L}-\vec\pi)=&\sqrt{1-e^{-2\tau_t}}\sum_i\zeta^{1/2}e^{-\lambda_i\tau_t^\zeta}P_i\vec x_0\\
    =&\sum_i\left(\frac{\zeta(1-e^{-2\tau_t})}{(\zeta(e^{2\tau_t}-1)+1)^{\lambda_i}}\right)^{1/2}P_i\vec x_0\\
    \to&e^{-\tau_t}P_1\vec x_0.
    \end{aligned}
    \end{equation*}

    For the first term, we need a ``uniform'' central limit theorem as the underlying distribution changes with $\zeta$ because of $\vec x_0^Te^{\tau_t^\zeta\mathcal L}$.
    Lem.~\ref{lem:berry-esseen} shows that $\zeta^{-1/2}(\mathrm{Mult}(\zeta, \vec x_0^Te^{\tau_t^\zeta\mathcal L})-\vec x_0^Te^{\tau_t^\zeta\mathcal L})$ approaches $\mathcal N(0, \mathrm{diag}(\vec p_t)-\vec p_t\vec p_t^T)$ for $\vec p_t=\vec x_0^Te^{\tau_t^\zeta\mathcal L}$, which itself approaches $\vec\pi$ as $\tau_t^\zeta\to\infty$.
    Therefore the first term, divided by $\sqrt{\vec\pi}$ approaches
    $$\sqrt{1-e^{-2\tau_t}}\mathcal N\left(0, I-\sqrt{\vec \pi}\sqrt{\vec \pi}^T\right).$$
    Note $\tilde Q_i\sqrt{\vec\pi}=\sqrt{\vec\pi}^{-1}P_i\vec\pi=0$ for each $i$ and, for $i>1$, $\tilde Q_i(P_1\vec x_0/\sqrt{\vec{\pi}})=\sqrt{\vec\pi}^{-1}P_iP_1\vec x_0=0.$
    Therefore, as desired,
    $$Q_1 x_t^\zeta\leadsto \sqrt{1-e^{-2\tau_t}}\mathcal N\left(0, I\right)+e^{-\tau_t}\mathrm{emb}(x_0),$$
    and for $i>1$,
    $$(1-e^{-2\tau_t})^{-1/2}Q_i x_t^\zeta\leadsto \mathcal N\left(0, I\right).$$    

    \textbf{Part 2. Local uniform convergence of conditionals:}
    Note
    $$\vec z_t|\vec z_s\sim e^{-(\tau_t-\tau_s)}\vec z_s+\sqrt{1-e^{-2(\tau_t-\tau_s)}}\mathcal N(0, I).$$
    We want to prove convergence to this quantity.
    Note 
    $$\vec x_t|\vec x_s\sim\sum_b \mathrm{Mult}(\zeta \vec x_{s, b}, \vec b^Te^{(\tau_t^\zeta-\tau_s^\zeta)\mathcal L})/\zeta$$
    where $\vec x_t=\sqrt{\pi}\circ\vec  x_t^\zeta/\sqrt{\zeta-(\zeta-1)e^{-2\tau_t}}+\pi$ are the ``unscaled'' versions of the vector and $\vec x_s$ is similar.
    It will be convenient below to extend this definition to $\vec x_{s}^\zeta$ for which $\zeta \vec x_{s, b}$ are not integers, but which still satisfy $\sum_b\sqrt{\pi_b}\vec x_{t, b}^\zeta=0$.
    To do so, we just round $\zeta \vec x_{s, b}$ down to $\lfloor{\zeta \vec x_{s, b}}\rfloor$.

    Fix $\vec v$. We now show $\vec x_t^\zeta|\vec x_s^\zeta=\vec v\leadsto \vec z_t|\vec z_s=\vec v$; a very similar argument also shows $\vec x_t^\zeta\leadsto \vec z_t$.
    Call $\vec x^{\zeta}$ a variable distributed as $\vec x_t^\zeta|\vec x_s^\zeta=\vec v$, so, calling
    \begin{gather*}
    w_t^\zeta = \frac{\sqrt{\zeta-(\zeta-1)e^{-2\tau_t}}}{\zeta},\\
    N_{s, b}^\zeta = {\sqrt{\pi_b}\vec  v_b}/w_s^\zeta+\zeta\pi_b,\\
    C_{t, b}^\zeta \sim \mathrm{Mult}\left(\left\lfloor N_{s, b}^\zeta\right\rfloor, \vec b^Te^{(\tau_t^\zeta-\tau_s^\zeta)\mathcal L}\right)\text{ independent across $b$},
    \end{gather*}
    then
    \begin{equation*}
    \begin{aligned}\vec x^\zeta_t\sim& w_{t}^\zeta\left(\sum_b C_{t, b}^\zeta-\zeta\pi\right)/\sqrt{\pi}\\
    =&w_{t}^\zeta\left(\sum_b \left[(C_{t, b}^\zeta-N_{s, b}^\zeta\vec b^Te^{(\tau_t^\zeta-\tau_s^\zeta)\mathcal L}
)+N_{s, b}^\zeta(\vec b^Te^{(\tau_t^\zeta-\tau_s^\zeta)\mathcal L}-\pi)\right]\right)/\sqrt{\pi}
\end{aligned}
    \end{equation*}
    noting $\sum_bp_{t, b}^\zeta=\zeta$.
    This is exactly the ``noise, signal'' breakdown we had in the proof sketch.
    
    For the signal (second term), first note
    $$\sum_b\pi_b(\vec b^Te^{(\tau_t^\zeta-\tau_s^\zeta)\mathcal L}-\vec \pi)=\vec\pi^Te^{(\tau_t^\zeta-\tau_s^\zeta)\mathcal L}-\vec \pi=0,$$
    so, ignoring the $\pi$ term in $N_{s, b}^\zeta$ the second term is 
    \begin{equation*}
    \begin{aligned}\frac {w_t} {w_s}\left(\sum_b\sqrt{\pi_b}\vec v_b(\vec b^Te^{(\tau_t^\zeta-\tau_s^\zeta)\mathcal L}-\vec \pi)\right)/\sqrt\pi=&\frac {w_t} {w_s}\left((\sqrt\pi\circ\vec v)^Te^{(\tau_t^\zeta-\tau_s^\zeta)\mathcal L}\right)/\sqrt\pi\\
    =&(1+o(1))\sum_i \left(\frac{1-e^{-2\tau_s}}{1-e^{-2\tau_t}}\right)^{(\lambda_i-1)/2}e^{-\lambda_i(\tau_t-\tau_s)}\tilde Q_i\vec v.
    \end{aligned}
    \end{equation*}

    For the first term, we again apply Lem.~\ref{lem:berry-esseen}, noting $N_{s, b}^\zeta=(1+o(1))\zeta\pi_b$ to get
    \begin{equation*}
    \begin{aligned}
    \sum_bw_t(&C_{t, b}^\zeta-N_{s, b}^\zeta\vec b^Te^{(\tau_t^\zeta-\tau_s^\zeta)\mathcal L}
)/\sqrt{\vec\pi}\\
    \leadsto &\sqrt{1-e^{-2\tau_t}} \sum_b\sqrt{\pi_b}\mathcal N\left(0, \diag(\vec b^Te^{(\tau_t^\zeta-\tau_s^\zeta)\mathcal L})- e^{(\tau_t^\zeta-\tau_s^\zeta)\mathcal L^T}\vec b\vec b^Te^{(\tau_t^\zeta-\tau_s^\zeta)\mathcal L}\right)/\sqrt{\vec\pi}\\
    =&\sqrt{1-e^{-2\tau_t}}\mathcal N\left(0, \diag(\vec \pi^Te^{(\tau_t^\zeta-\tau_s^\zeta)\mathcal L})- e^{(\tau_t^\zeta-\tau_s^\zeta)\mathcal L^T}\mathrm{diag}(\vec\pi)e^{(\tau_t^\zeta-\tau_s^\zeta)\mathcal L}\right)/\sqrt{\vec\pi}\\
    =&\sqrt{1-e^{-2\tau_t}}\mathcal N\left(0, \diag(\vec \pi)- (\sum_i e^{-\lambda_i(\tau_t^\zeta-\tau_s^\zeta)}P_i)\mathrm{diag}(\vec\pi)(\sum_i e^{-\lambda_i(\tau_t^\zeta-\tau_s^\zeta)}P_i^T)\right)/\sqrt{\vec\pi}\\
    =&\sqrt{1-e^{-2\tau_t}}\mathcal N\left(0, I- (\sum_i e^{-\lambda_i(\tau_t^\zeta-\tau_s^\zeta)}\tilde Q_i)(\sum_i e^{-\lambda_i(\tau_t^\zeta-\tau_s^\zeta)}\tilde Q_i^T)\right).
    \end{aligned}
    \end{equation*}
    Therefore, as desired,
    \begin{equation*}
    \begin{aligned}
    Q_1\vec x_t^\zeta\mid\vec x_s^\zeta=&\vec v\sim e^{-(\tau_t-\tau_s)}Q_1\vec v+\sqrt{(1-e^{-2\tau_t})\left(1-\frac{1-e^{-2\tau_s}}{1-e^{-2\tau_t}}e^{-2(\tau_t-\tau_s)}\right)}\mathcal N(0, I)\\
    =&\vec v\sim e^{-(\tau_t-\tau_s)}Q_1\vec v+\sqrt{1-e^{-2(\tau_t-\tau_s)}}\mathcal N(0, I)
    \end{aligned}
    \end{equation*}
    and similarly
    \begin{equation*}
    \begin{aligned}(1-e^{-2\tau_t})^{-1/2}Q_i\vec x_t^\zeta\mid\vec x_s^\zeta=\vec v\sim& \left(\frac{1-e^{-2\tau_s}}{1-e^{-2\tau_t}}\right)^{\lambda_i/2}e^{-\lambda_i(\tau_t-\tau_s)}((1-e^{-2\tau_2})^{-1/2}Q_i\vec v)\\
    &+\sqrt{1-\left(\frac{1-e^{-2\tau_s}}{1-e^{-2\tau_t}}\right)^{\lambda_i}e^{-2\lambda_i(\tau_t-\tau_s)}}\mathcal N(0, I)\\
    =& e^{-(\tau_t^{(i)}-\tau_s^{(i)})}((1-e^{-2\tau_2})^{-1/2}Q_i\vec v)\\
    &+\sqrt{1-e^{-2(\tau_t^{(i)}-\tau_s^{(i)})}}\mathcal N(0, I)\\
    \end{aligned}
    \end{equation*}

    Finally, convergence is clearly uniform for nearby $\vec v$ using the uniformity of Lem.~\ref{lem:berry-esseen}.

    \textbf{Part 3. Tightness:}
    Pick $s< t\in(0, 1).$
    \begin{equation*}
    \mathbb{E}\|\vec{x}_t^\zeta - \vec{x}_s^\zeta\|^2 = \mathbb{E}\|\mathbb{E}[\vec{x}_t^\zeta|\vec{x}_s^\zeta] - \vec{x}_s^\zeta\|^2 + \mathbb{E}\|\vec{x}_t^\zeta - \mathbb{E}[\vec{x}_t^\zeta|\vec{x}_s^\zeta]\|^2
    \end{equation*}
    The first term has, for each $x_0$,
    \begin{equation*}
    \begin{aligned}
    \mathbb{E}\|\mathbb{E}[\vec{x}_t^\zeta|\vec{x}_s^\zeta] - \vec{x}_s^\zeta\|^2
    =&\mathbb E\|w_t(\vec x_se^{(\tau_t^\zeta-\tau_s^\zeta)\mathcal L}-\vec\pi)/\sqrt{\vec\pi}-\vec{x}_s^\zeta\|^2\\
    =&\frac{1}{\min_b\pi_b}\mathbb E\|w_t(\vec x_se^{(\tau_t^\zeta-\tau_s^\zeta)\mathcal L}-\vec x_s)-(w_s-w_t)(\vec x_s-\vec\pi)\|^2\\
    \leq &\frac{1}{\min_b\pi_b}\mathbb E\left(|w_t|\|\vec x_se^{(\tau_t^\zeta-\tau_s^\zeta)\mathcal L}-\vec x_s\|+|w_s-w_t|\|\vec x_s-\vec\pi\|\right)^2\\
    = &\frac{1}{\min_b\pi_b}\mathbb E\left(|w_t|\|(\vec x_s-\vec\pi)^T(I-e^{(\tau_t^\zeta-\tau_s^\zeta)\mathcal L})\|+|w_s-w_t|\|\vec x_s-\vec\pi\|\right)^2\\
    \leq &\frac{1}{\min_b\pi_b}\mathbb E\left(|w_t|(1-e^{-(\tau_t^\zeta-\tau_s^\zeta)\lambda_B})\|\vec x_s-\vec\pi\|+|w_s-w_t|\|\vec x_s-\vec\pi\|\right)^2\\
    =&\frac{1}{\min_b\pi_b}\left(|w_t|(1-e^{-(\tau_t^\zeta-\tau_s^\zeta)\lambda_B})+|w_s-w_t|\right)^2\mathbb E\|\vec x_s-\vec\pi\|^2\\
    \leq &\frac{\zeta}{\min_b\pi_b}\left((1-e^{-(\tau_t^\zeta-\tau_s^\zeta)\lambda_B})+|1-\frac{w_s}{w_t}|\right)^2\\
    &\times\left(\mathbb E\mathrm{Tr}\mathrm{Cov}(\mathrm{Mult}(\zeta, \vec x_0^Te^{\tau_s^\zeta)\mathcal L}/\zeta)+\|\vec x_0^Te^{\tau_s^\zeta\mathcal L}-\vec\pi\|^2\right)\\
    \leq &\frac{1}{\min_b\pi_b}\left((1-e^{-(\tau_t^\zeta-\tau_s^\zeta)\lambda_B})+|1-\frac{w_s}{w_t}|\right)^2(1+\zeta e^{-2\tau_s^\zeta})\\
    \leq &\frac{1}{\min_b\pi_b}\left((1-e^{-(\tau_t^\zeta-\tau_s^\zeta)\lambda_B})+|1-\frac{w_s}{w_t}|\right)^2\left(1+\frac{1}{e^{2\tau_s}-1}\right)\\
    \end{aligned}
    \end{equation*}
    Now,
    \begin{equation*}
    \begin{aligned}
    1-e^{-(\tau_t^\zeta-\tau_s^\zeta)\lambda_B}=&1-e^{-2\lambda_B(\tau_t-\tau_s)}\left(\frac{1-e^{-2\tau_s}(1-\zeta^{-1})}{1-e^{-2\tau_t}(1-\zeta^{-1})}\right)^{\lambda_B/2}\\
    \leq& 1- e^{-2\lambda_B(\tau_t-\tau_s)}\\
    &+1-\left(\frac{1-e^{-2\tau_s}(1-\zeta^{-1})}{1-e^{-2\tau_t}(1-\zeta^{-1})}\right)^{\lambda_B/2}.
    \end{aligned}
    \end{equation*}
    When $|\tau_s-\tau_t|<1/4\lambda_B$
    $$1- e^{-2\lambda_B(\tau_t-\tau_s)}\leq 4\lambda_B(\tau_t-\tau_s)\leq 4\lambda_B|t-s|\sup_{u\in[s, t]}\dot\tau_u.$$
    Next note that if $\alpha\geq 1$, $x\mapsto1-x^\alpha$ has decreasing derivative, from $0$ to $-\alpha$ on the interval $x\in[0, 1]$, so, it is dominated on this interval by $\alpha(1-x).$
    If $\zeta>1$,
    \begin{equation*}
    \begin{aligned}
    1-\left(\frac{1-e^{-2\tau_s}(1-\zeta^{-1})}{1-e^{-2\tau_t}(1-\zeta^{-1})}\right)^{\lambda_B/2}\leq& 1-\left(\frac{1-e^{-2\tau_s}(1-\zeta^{-1})}{1-e^{-2\tau_t}(1-\zeta^{-1})}\right)^{1\vee(\lambda_B/2)}\\
    \leq &(1\vee(\lambda_B/2))\left( 1-\left(\frac{1-e^{-2\tau_s}(1-\zeta^{-1})}{1-e^{-2\tau_t}(1-\zeta^{-1})}\right)\right)\\
    \leq &(1\vee(\lambda_B/2))\left(\frac{(e^{-2\tau_s}-e^{-2\tau_t})(1-\zeta^{-1})}{1-e^{-2\tau_t}}\right)\\
    \leq &\frac{1\vee(\lambda_B/2)e^{-2\tau_s}}{1-e^{-2\tau_t}}\left({1-e^{-2(\tau_t-\tau_s)}}\right)\\
    \leq&\frac{4\vee(2\lambda_B)e^{-2\tau_s}}{1-e^{-2\tau_t}}|t-s|\sup_{u\in[s, t]}\dot\tau_u
    \end{aligned}
    \end{equation*}
    
    Finally
    \begin{equation*}
    \begin{aligned}
    1-\frac{w_s}{w_t}=1-\left(\frac{1-e^{-2\tau_s}(1-\zeta^{-1})}{1-e^{-2\tau_t}(1-\zeta^{-1})}\right)^{1/2}.
    \end{aligned}
    \end{equation*}
    which is similar to above.
    
    The second term has 
    \begin{equation*}
    \begin{aligned}
    \mathbb{E}\|\vec{x}_t^\zeta - \mathbb{E}[\vec{x}_t^\zeta|\vec{x}_s^\zeta]\|^2\leq &\frac{2\zeta}{\min_b\pi_b}\sum_b\mathbb E\mathrm{Tr}\mathrm{Cov}(\mathrm{Mult}(\zeta \vec x_{s, b}, \vec b^Te^{(\tau_t^\zeta-\tau_s^\zeta)\mathcal L})/\zeta\mid\vec x_t^\zeta) \\
    = &\frac{2}{\min_b\pi_b}\sum_b\mathbb E\vec x_{s, b}^\zeta\sum_{b'}(\vec b^Te^{(\tau_t^\zeta-\tau_s^\zeta)\mathcal L}\vec b')(1-\vec b^Te^{(\tau_t^\zeta-\tau_s^\zeta)\mathcal L}\vec b') \\
    \leq& \frac{2}{\min_b\pi_b}\left(\sum_{b\neq b'}\vec b^Te^{(\tau_t^\zeta-\tau_s^\zeta)\mathcal L}\vec b' + \sum_{b}(1-\vec b^Te^{(\tau_t^\zeta-\tau_s^\zeta)\mathcal L}\vec b)\right) \\
    =&\frac{4}{\min_b\pi_b}\sum_{b}(1-\vec b^Te^{(\tau_t^\zeta-\tau_s^\zeta)\mathcal L}\vec b) \\
    \leq &\frac{4B}{\min_b\pi_b}(1-e^{-(\tau_t^\zeta-\tau_s^\zeta)\lambda_B})
    \end{aligned}
    \end{equation*}
    which is bounded similar to the first term.

    \textbf{Part 4. Convergence of the ELBO:}
    Define
    $p = \vec x_0^Te^{\tau_t^\zeta\mathcal L}$.
    We've shown above that
    $$p=\vec \pi+\sqrt{\frac{1}{\zeta(e^{2\tau_t}-1)}}P_1\vec x_0+o(\zeta^{-1/2})$$
    so
    $$\frac{p_{b_2}}{p_{b_1}}=\frac{\pi_{b_2}}{\pi_{b_1}}+\frac{1}{\pi_{b_1}}\sqrt{\frac{1}{\zeta(e^{2\tau_t}-1)}}\left(\vec b_2-\frac{\pi_{b_2}}{\pi_{b_1}}\vec b_1\right)^TP_1\vec x_0+o(\zeta^{-1/2})$$
    and similar for $q$.
    Using a second-order Taylor expansion on $\mathbb{D}$, we get
    $$\mathbb{D}\left(\hat w(x_0)_{b_2, b_1}||\hat w(\tilde x_0)_{b_2, b_1}\right)=\frac 1 2\frac{\pi_{b_1}}{\pi_{b_2}}\frac{1}{\pi_{b_1}^2\zeta(e^{2\tau_t}-1)}\left(\left(\vec b_2-\frac{\pi_{b_2}}{\pi_{b_1}}\vec b_1\right)^TP_1(\vec x_0-\tilde x_0)\right)^2+o(\zeta^{-1}).$$
    Next note $\dot\tau_t^\zeta=\dot\tau_t\frac{e^{2\tau_t}}{e^{2\tau_t}-1}+o(1)$.
    Finally note 
    $$\vec x_t(\vec v)=\sqrt{\pi}\circ\vec v/\sqrt{\zeta-(\zeta-1)e^{-2\tau_t}}+\pi=\pi+o(1).$$
    Putting this together, we get
    \begin{equation*}
    \begin{aligned}
        L(\vec v,& t, \vec x_0, \tilde x_0)\\
        = &\sum_{b_1\neq b_2}\mathcal L_{b_2\to b_1} \dot\tau_t^\zeta \zeta \vec x_{t, b_1}(\vec v) \mathbb{D}\left(\frac{p_{b_2}}{p_{b_1}}\bigg|\bigg|\frac{q_{b_2}}{q_{b_1}}\right)\\
        =&\dot\tau_t\sum_{b_1\neq b_2}\mathcal L_{b_2\to b_1}\frac{e^{2\tau_t}}{e^{2\tau_t}-1}\pi_{b_1}\frac{1}{2\pi_{b_2}\pi_{b_1}}\frac{1}{(e^{2\tau_t}-1)}\left(\left(\vec b_2-\frac{\pi_{b_2}}{\pi_{b_1}}\vec b_1\right)^TP_1(\vec x_0-\tilde x_0)\right)^2+o(1)\\
        =&\frac{\dot\tau_t e^{2\tau_t}}{2(e^{2\tau_t}-1)^2}\sum_{b_1\neq b_2}\mathcal L_{b_2\to b_1}\left(\left(\vec b_2-\sqrt{\frac{\pi_{b_2}}{\pi_{b_1}}}\vec b_1\right)^T\tilde Q_1\left((\vec x_0-\tilde x_0)/\sqrt{\vec\pi}\right)\right)^2+o(1)\\
        =&\frac{\dot\tau_t e^{-2\tau_t}}{(1-e^{-2\tau_t})^2}\left\|\tilde Q_1\left((\vec x_0-\tilde x_0)/\sqrt{\vec\pi}\right)\right\|_\Sigma^2+o(1)
    \end{aligned}
    \end{equation*}
    where 
    $$\Sigma = \frac 1 2\sum_{b_1\neq b_2}\mathcal L_{b_2\to b_1}\left(\vec b_2-\sqrt{\frac{\pi_{b_2}}{\pi_{b_1}}}\vec b_1\right)\left(\vec b_2-\sqrt{\frac{\pi_{b_2}}{\pi_{b_1}}}\vec b_1\right)^T.$$
    To solve $\Sigma$, we note
    \begin{gather*}
        \sum_{b_1\neq b_2}\mathcal L_{b_2\to b_1}\vec b_2\vec b_2^T=\sum_{ b_2} \vec b_2\vec b_2^T \sum_{b_1\neq b_2}\mathcal L_{b_2\to b_1}=-\sum_{ b_2} \vec b_2\vec b_2^T\mathcal L_{b_2, b_2}\\
        \sum_{b_1\neq b_2}\frac{\pi_{b_2}}{\pi_{b_1}}\mathcal L_{b_2\to b_1}\vec b_2\vec b_2^T=\sum_{ b_1} \vec b_1\vec b_1^T \sum_{b_2\neq b_1}\frac{\pi_{b_2}}{\pi_{b_1}}\mathcal L_{b_2\to b_1}=-\sum_{ b_1} \vec b_1\vec b_1^T\mathcal L_{b_1, b_1}\\
        \sum_{b_1\neq b_2}\sqrt{\frac{\pi_{b_2}}{\pi_{b_1}}}\mathcal L_{b_2\to b_1}\vec b_2\vec b_1^T=\mathrm{diag}(\sqrt{\vec\pi})\left(\mathcal L-\diag\mathcal L\right)\mathrm{diag}(1/\sqrt{\vec\pi})\\
        \sum_{b_1\neq b_2}\sqrt{\frac{\pi_{b_2}}{\pi_{b_1}}}\mathcal L_{b_2\to b_1}\vec b_1\vec b_2^T=(\mathrm{diag}(\sqrt{\vec\pi})\left(\mathcal L-\diag\mathcal L\right)\mathrm{diag}(1/\sqrt{\vec\pi}))^T.\\
    \end{gather*}
    So,
    $$\Sigma = -\frac 1 2\mathrm{diag}(\sqrt{\vec\pi})\mathcal L\mathrm{diag}(1/\sqrt{\vec\pi})-\frac 1 2(\mathrm{diag}(\sqrt{\vec\pi})\mathcal L\mathrm{diag}(1/\sqrt{\vec\pi}))^T.$$
    In particular, since $\tilde Q_1^T\mathrm{diag}(\sqrt{\vec\pi})\mathcal L\mathrm{diag}(1/\sqrt{\vec\pi})=-\tilde Q_1^T$,
    $$\tilde Q_1^T\Sigma\tilde Q_1=\tilde Q_1^T\tilde Q_1=Q_1^TQ_1.$$
    This gives us
    \begin{equation*}
    \begin{aligned}
        L(\vec v,& t, \vec x_0, \tilde x_0)\to\frac{\dot\tau_t e^{-2\tau_t}}{(1-e^{-2\tau_t})^2}\left\|\mathrm{emb}(x_0)-\mathrm{emb}(\tilde x_0)\right\|^2.
    \end{aligned}
    \end{equation*}
\end{proof}

\subsection{Hollow parameterization solves Gaussian ELBO singularity}\label{app: hollow proof}
Here we show that the hollow parametrization introduced in Sec.~\ref{sec: solve theory} resolves the singularity of the Gaussian ELBO in Alg.~\ref{alg: Gaussian diffusion} at $t\to 0^+$.
Before going into the proof, let us give some intuition.
Assume, $x_0^d$ were distributed uniformly and independently.
Then
$$p(x_0^d\mid x_t, t)\propto p(x_t^d\mid x_0^d, t)p(x_0^d\mid x_t^{-d}, t),$$
where $x_t^{-d}$ includes all positions but $d$.
However
$$p(x_0^d\mid x_t^{-d}, t)=\int p(x_0^d\mid x_0^{-d})dp(x_0^{-d}\mid x_t^{-d}, t)=\mathrm{Uniform}.$$
Therefore, we get $p(x_0^d\mid x_t, t)\propto p(x_t^d\mid x_0^d, t).$
At initialization, we can say our neural network $q_\theta(x_0^d\mid x_t^{-d}, t)\approx \mathrm{Uniform},$
so, 
$$q_\theta(x_0^d\mid x_t, t)\approx p(x_0^d\mid x_t, t).$$
Therefore, \textbf{the hollow parametrization initializes the diffusion model near a uniform, site-wise independent model.}
The proof below involves a lot of algebra, but the basic intuition for why we should not see singularities is that by initializing at a \textit{valid} diffusion model, we get comparable ELBOs.

Again we assume $D=1$ for simplicity as results are straightforward to generalize to higher $D$.
\begin{proposition}
    Assume $\mathrm{emb}$ is injective and $\tau_t$ is increasing and differentiable.
    Define $$L = \frac{\dot\tau_t e^{-2\tau_t}}{(1-e^{-2\tau_t})^2}\Vert \mathrm{emb}( x_0)-\mathrm{emb}( \tilde x_0)\Vert^2\vphantom{x^Te^{\tau_t\mathcal L}},$$
    and the normalized vectors $\vec\phi(x_t, t)\propto p(x_t\mid x_0, t)$.
    For $\tilde x_0$ build using the hollow predictor $\tilde x_0=\vec\phi(x_t, t)\circ\vec q/\vec\phi(x_t, t)^T\circ\vec q$ for a vector $t$ bounded away from $0$ and $\infty$,
    $$0<c=\min_b\vec q_b\leq \max_b\vec q_b<C<\infty,$$
    we have
    $$\mathbb E_{t, x_0, x_t}L<\infty.$$
\end{proposition}
\begin{proof}
    Note first
    $$\Vert \mathrm{emb}( x_0)-\mathrm{emb}( \tilde x_0)\Vert^2\leq \|\mathrm{emb}\|^2\|\vec x_0-\tilde x_0\|$$
    and, simplifying $\vec \phi=\vec\phi(x_t, t)$,
    $$\mathbb E_{x_0\mid x_t}\|\vec x_0-\tilde x_0\|=\|\vec \phi\circ\vec p/\vec\phi^T\vec p-\vec \phi\circ\vec q/\vec\phi^T\vec q\|$$
    for $\vec p_b=p(x_0).$

    Call $b=\mathrm{argmax}_{b'}\vec\phi_{b'},$ so 
    \begin{equation*}
        \begin{aligned}
            \|\vec \phi\circ\vec p/\vec\phi^T\vec p-\vec \phi\circ\vec q/\vec\phi^T\vec q\|\leq &\left(\frac{\vec\phi_b\vec p_b}{\vec\phi^T\vec p}-\frac{\vec\phi_b\vec q_b}{\vec\phi^T\vec q}\right)^2+(1-\vec\phi_b)^2\sum_{b'\neq b}\left(\frac{\vec p_{b'}}{\vec\phi^T\vec p}-\frac{\vec q_{b'}}{\vec\phi^T\vec q}\right)^2\\
            =&\left(\frac{1}{1+\sum_{b'\neq b}\frac{\vec\phi_{b'}\vec p_{b'}}{\vec\phi_{b}\vec p_{b}}}-\frac{1}{1+\sum_{b'\neq b}\frac{\vec\phi_{b'}\vec q_{b'}}{\vec\phi_{b}\vec q_{b}}}\right)^2\\
            &+\left(\frac{C}{c}\right)^2B(1-\phi_b)^2\\
            \leq& \left(1-\frac{1}{1+\frac{CB(1-\phi_b)}{c}}\right)^2\\
            \leq& \left(\frac{CB}{c}\right)^2(1-\phi_b)^2+\left(\frac{C}{c}\right)^2B(1-\phi_b)^2.
        \end{aligned}
    \end{equation*}
    
    We've therefore bounded $\mathbb E_{t, x_0, x_t}L$ above by some constant times $\mathbb E_{t, x_t}\frac{\dot\tau_t e^{-2\tau_t}}{(1-e^{-2\tau_t})^2}(1-\max_b\vec\phi_b)^2$.
    Note without the hollow parameterization, we wouldn't have the $(1-\max_b\vec\phi_b)^2$ term; we now show this becomes small very fast as $t\to 0$ (because $x_0$ becomes ``obvious'' from $x_t$), cancelling out the singularity.
    
    Next note, calling $b=\mathrm{argmin}_{b'}\|\mathrm{emb}(b')-\vec x_t\|,$
    \begin{equation*}
        \begin{aligned}
            (1-\max_b\vec\phi_b)^2=&\left(1-\frac{1}{1+\sum_{b'\neq b}\exp(-\frac{1}{2(1-e^{-2\tau_t})^2}(\|\mathrm{emb}(b')-\vec x_t\|^2-\|\mathrm{emb}(b)-\vec x_t\|^2))}\right)^2\\
            \leq &\sum_{b'\neq b}\exp\left(-\frac{1}{2(1-e^{-2\tau_t})}(\|\mathrm{emb}(b')-\vec x_t\|^2-\|\mathrm{emb}(b)-\vec x_t\|^2)\right),
        \end{aligned}
    \end{equation*}
    which is only large if $\vec x_t$ is roughly equidistant to two potential $x_0$.
    Call $\epsilon=\min_{b\neq b'}\|\mathrm{emb}(b)-\mathrm{emb}(b')\|/4,$ so, if $\mathrm{min}_{b'}\|\mathrm{emb}(b')-\vec x_t\|<\epsilon$ then, by the triangle inequality
    \begin{equation*}
        \begin{aligned}
            \|\mathrm{emb}(b')-\vec x_t\|^2-\|\mathrm{emb}(b)-\vec x_t\|^2\geq &(\|\mathrm{emb}(b)-\mathrm{emb}(b')\|-\|\mathrm{emb}(b)-\vec x_t\|)^2\\
            &-\|\mathrm{emb}(b)-\vec x_t\|^2\\
            =&\|\mathrm{emb}(b)-\mathrm{emb}(b')\|\\
            &-2\|\mathrm{emb}(b)-\mathrm{emb}(b')\|\|\mathrm{emb}(b)-\vec x_t\|\\
            \geq & 16\epsilon^2-8\epsilon^2=8\epsilon^2.
        \end{aligned}
    \end{equation*}
    Therefore, $\mathbb E_{t, x_t}\frac{\dot\tau_t e^{-2\tau_t}}{(1-e^{-2\tau_t})^2}(1-\max_b\vec\phi_b)^2$ is bounded by
    $$B\mathbb E_{t}\frac{\dot\tau_t e^{-2\tau_t}}{(1-e^{-2\tau_t})^2}\left(\exp\left(-\frac{4\epsilon^2}{(1-e^{-2\tau_t})}\right)+p(\mathrm{min}_{b'}\|\mathrm{emb}(b')-\vec x_t\|\geq\epsilon)\right).$$
    To deal with the first term, perform a change of variables $u=(1-e^{-2\tau_t})^{-1}$, giving
    $$\mathbb E_{t}\frac{\dot\tau_t e^{-2\tau_t}}{(1-e^{-2\tau_t})^2}\exp\left(-\frac{4\epsilon^2}{(1-e^{-2\tau_t})}\right)=\frac 1 2\int_0^\infty du\exp(-4\epsilon^2u)<\infty.$$
    For the second term, note
    \begin{equation*}
        \begin{aligned}
        p(\mathrm{min}_{b'}\|\mathrm{emb}(b')-\vec x_t\|\geq\epsilon)\leq &\sum_{b}p(x_0=b)p(\|\mathcal N(0, (1-e^{-2\tau_t})I_{r\times r})\|>\epsilon)\\
        =&p(\chi^2_r/\epsilon^2>1/(1-e^{-2\tau_t}))
        \end{aligned}
    \end{equation*}
    where $\chi^2_r$ is a chi-squared distribution with $r$ degrees of freedom.
    Finally, by the same change of variables $u$ as above, we get
    $$\mathbb E_{t}\frac{\dot\tau_t e^{-2\tau_t}}{(1-e^{-2\tau_t})^2}p(\mathrm{min}_{b'}\|\mathrm{emb}(b')-\vec x_t\|\geq\epsilon)=\frac 1 2\int_0^\infty dup(\chi^2_r/\epsilon^2>u)=\frac 1 2\mathbb E \chi^2_r/\epsilon^2<\infty.$$

\end{proof}

\subsection{Every embedding can be induced from some infinitesimal generator}\label{app: rep to inf gen}

Define an injective embedding $\mathrm{emb}:\{1, \dots, B\}\to \mathbb R^{r}$ for some $r$.
For an infinitesimal generator $\mathcal L$ with a unique stationary distribution $\vec \pi$, define $Q_1=\mathfrak j_1(\tilde Q_1\tilde Q_1^T)^{-1/2}\tilde Q_1$, $\mathfrak j_1$ is some isometry, $\tilde Q_1=\diag(\vec\pi)^{-1/2}P_1\diag(\vec\pi)^{1/2}$ where $P_1$ is the projection onto the first left eigenspace.
Is there a choice of $\mathcal L$ such that $Q_1(\vec x_0/\sqrt{\vec\pi_{x_0}})=\mathrm{emb}(x_0)$ for every $x_0$?

If we restrict to $\mathcal L\in\mathbb R^{B\times B}$ then the answer is no.
To see this, call $W\in\mathbb R^{B\times r}$ the matrix with $W\vec b=\mathrm{emb}(b)$.
Then, defining $D=\diag(\vec\pi)^{-1/2}$, we need $W^TW=DQ_1Q_1^TD=DPD$ for some orthogonal projection $P$ or rank $r$.
The space $\{W^TW\mid W\in \mathbb R^{B\times r}\}$ generates all rank-$r$ positive-semi-definite matrices, an algebraic variety of dimension $B\times r$.
Meanwhile, $P$ has $r(B-r)$ degrees of freedom and $D$ has $B-1$, so $DPD$ generates an algebraic variety of dimension at most $B\times r - r^2+B-1$, which is less than $B\times r$ when $r$ is large.

If however we allow $r+1$ ``dummy'' tokens, to let $\mathcal L\in\mathbb R^{(B+r)\times (B+r)}$, then the next proposition shows that the answer is yes.
This demonstrates an important distinction between the design space of Gaussian and discrete diffusions: dummy variables which never appear in the data have no effect on the training of Gaussian diffusion, but can serve as transient states in discrete diffusion.

\begin{proposition}
    There is some infinitesimal generator $\mathcal L\in\mathbb R^{(B+r+1)\times (B+r+1)}$ such that $Q_1(\vec x_0/\sqrt{\vec\pi})=\mathrm{emb}(x_0)$ for every $x_0\in\{1, \dots, B\}$.
    There are infinitely many such generators.
\end{proposition}
\begin{proof}
    Call $W\in\mathbb R^{B\times r}$ the matrix with $W\vec b=\mathrm{emb}(b)$.
    Call $\Lambda=W^TW$ and without loss of generality, assume its first $r$ rows are linearly independent.
    We split the proof into two parts:
    first we show that $Q_1^TQ_1$ can equal $DPD$ for any orthogonal projection matrix $P$ with $P\sqrt{\vec\pi}=0$ and $D=\diag(\vec\pi)^{-1/2}$ for any distribution $\pi$;
    then we show that $\Lambda$ can be written as the top $B\times B$ submatrix of $DPD$ for some choice opf $D$ and $P$.
    This will show that there is a $Q_1$ such that $Q_1(\cdot/\sqrt{\vec\pi_\cdot})$ is equivalent to $\mathrm{emb}$ up to isometry.
    
    \paragraph{Part 1}
    Pick an orthogonal projection $P$ and a distribution $ \pi$ such that $P\sqrt{\vec\pi} =0$
    Call $\tilde P = \diag(\vec\pi)^{-1/2} P\diag(\vec\pi)^{1/2}$ and
    $$\mathcal L_\mu = -\left(I-\mathbbm 1\vec\pi^T\right)+\mu\tilde P.$$
    Clearly, for every $\mu$, $\mathcal L_\mu\mathbbm 1=0$ and $\vec\pi^T\mathcal L_\mu=0$.
    Also, $\mathcal L_1=-I+\mathbbm 1\vec\pi^T$, so for $\mu$ in a neighbourhood of $1$, $\mathcal L_\mu$ has positive entries off the diagonal -- therefore it's an infinitesimal generator -- and $\mathcal L_\mu$ has a kernel of dimension $1$ -- so $\pi$ is the unique stationary distribution of $\mathcal L_\mu$.

    When $\mu$ is slightly greater than $0$, the first eigenspace of $\mathcal L_\mu$ is that of $\tilde P$; in particular, when $\tilde P$ is a projection,
    $P_1=\tilde P^T$ so $\tilde Q_1=P.$
    Note $Q_1^TQ_1=\tilde Q_1^T(\tilde Q_1\tilde Q_1^T)^{-1}\tilde Q_1$ which is the projection onto the orthogonal complement of $\mathrm{Ker} \tilde Q_1=\mathrm{Ker}P$;
    therefore it is equal to $P.$

    $P$ and $\vec \pi$ are the same for any small value of $\mu$, justifying the ``infinitely many'' proposal in the statement. 

    \paragraph{Part 2}
    First we need to ensure the rare case that $\mathbbm 1$ is orthogonal to the top eigenspace of $\Lambda$ does not occur.
    To ensure this, simple add another embedding $\mathrm{emb}(B+1)=\sum_b\vec w_b\mathrm{emb}(b)$ for some $\vec w$ to get a new matrix $\Lambda$ adding this extra token:
    \begin{equation*}
        \tilde \Lambda := \begin{bmatrix} \Lambda&\Lambda\vec w\\
        (\Lambda\vec w)^T&\vec w^T\Lambda\vec w
        \end{bmatrix}
    \end{equation*}
    Pick a $\vec v$ so $\vec v^T\Lambda\vec v\neq 0$ and $\vec w = \eta\vec v$.
    As $\eta\to\infty$, $\tilde\Lambda/\eta^2\to\vec v^T\Lambda\vec v (\vec e\vec e^T)$
    where $e$ is the indicator vector for position $B+1$.
    Therefore for some $\eta$, the top eigenvector approaches $\vec e$ and is not orthogonal to $\mathbbm 1$.
    Below we simply assume that $\mathbbm 1$ is not orthogonal to the top eigenspace of $\Lambda$.
    
    Decompose $\Lambda=\eta V\mathrm{diag}(\vec\lambda/\eta)V^T$ for a matrix $V\in\mathbb R^{B\times r}$ with orthonormal columns, a vector $\lambda$ of eigenvalues, and a scalar $\eta>\max_i\lambda_i$ to be chosen later.
    For an orthonormal matrix $U\in\mathbb{R}^{r\times r}$ to be chosen later, define
    \begin{equation*}
        \tilde V = \begin{bmatrix} V\mathrm{diag}(\vec{\lambda}/\eta)^{1/2}\\
        U(I-\mathrm{diag}(\vec{\lambda}/\eta))^{1/2}
        \end{bmatrix}
    \end{equation*}
    so $\tilde V$ has orthonormal columns.
    Define the orthogonal projection $P=\tilde V\tilde V^T$, so in particular, the upper $B\times B$ submatrix of $P$ is $\Lambda/\eta.$
    
    Finally we'll pick $\eta$ and $U$ to get a positive normalized vector $\vec\pi$ such that $\vec\pi_b=1/\eta$ for all $b\in\{1, \dots, B\}$ and $\tilde V\sqrt{\vec\pi}=0$, completing the proof.
    Breaking $\vec\pi$ into its first $B$ components and other $r$ components, $[\mathbbm 1/\eta, \vec \pi_2]$, we can write the equation $\tilde V^T\sqrt{\vec\pi}=0$ as
    $$\vec\pi_2=-\eta^{-3/2}U^{-1}\mathrm{diag}\left(\frac{\vec{\lambda}}{I-\vec\lambda/\eta}\right)^{1/2}V^T\mathbbm 1.$$
    We can always choose $U$ to rotate to get $\vec\pi_2=\mathbbm 1 \eta'$ where 
    $$\eta'=\eta^{-3/2}\left\|\mathrm{diag}\left(\frac{\vec{\lambda}}{I-\vec\lambda/\eta}\right)^{1/2}V^T\mathbbm 1\right\|/\sqrt{r}.$$
    Finally we need to solve for $\eta$ in
    $$1=B/\eta + \eta^{-3}\left\|\mathrm{diag}\left(\frac{\vec{\lambda}}{I-\vec\lambda/\eta}\right)^{1/2}V^T\mathbbm 1\right\|^2.$$
    This is possible by the intermediate value theorem as the right hand side goes to $0$ as $\eta\to\infty$ and goes to $\infty$ as $\eta\to \max_i\lambda_i$ from above (as we've assumed $V_{:, i}^T\mathbbm 1\neq 0$ for $i$ where $\lambda_i$ is the maximum eigenvalue).
\end{proof}

\subsection{Proof of Wright-Fisher convergence}\label{app: wf proof}

Define $\Delta^B\subset \mathbb R^B$ be the simplex, i.e. the set of non-negative vectors with components summing to $1$.
Let $(\vec x_t^\zeta)_{t=0}^1$ be a stochastic process
on $(\frac 1 \zeta\mathbb Z^B)\cap\Delta^B$ with $\vec x^\zeta_0=\vec x_0$ evolving with respect to $\mathcal L^\mathrm{mut}+\zeta\mathcal L^\mathrm{wf}$ where
$$\mathcal L^\mathrm{wf}_{\vec x^\zeta\to \vec x^{\prime\zeta}}=\frac{\zeta!}{\prod_b\zeta \vec x^{\prime\zeta}_{b}!}\prod_b (\vec x^\zeta_b)^{\zeta \vec x^{\prime\zeta}_{b}}=\mathrm{Mult}(\zeta, \vec x^{\zeta})(\zeta\vec x^{\prime\zeta}),$$
and, if $\vec x^\zeta, \vec x^{\prime\zeta}$ differ by one count $b\to b'$,
$$\mathcal L^\mathrm{mut}_{\vec x^\zeta\to \vec x^{\prime\zeta}}=(\psi(\mathbbm 1\vec\pi^T-I))_{b, b'}=\psi\vec\pi_{b'}$$
otherwise it's $0$.
Let $(\vec z_t)_t$ be a continuous Wright-Fisher process, that is, $\vec z_t=\vec x_0$ and 
\begin{equation}\label{eqn: wf sde}
    d\vec z_t=\frac{\psi}{2}(\vec\pi-\vec z_t)dt+\mathrm{diag}\left(\sqrt{\vec z_t}\right)\left(I-\sqrt{\vec z_t}\sqrt{\vec z_t}^T\right)d\vec W_t
\end{equation}
where $(W_t)_t$ is a Brownian motion.

\subsubsection{Convergence of the forward process}\label{app: wf proof fwd}
We have convergence of the forward processes from previous literature.
\begin{theorem}
    (Thm 1.1 \citet[Chapter~10]{Ethier1986-gm})
    Assume $\mathcal L=\psi\times(\mathbbm{1}\vec\pi^T-I)$.
    In the topology of convergence of compact sets,
    $(\vec x^\zeta_t)_{t\in[0, 1)}\leadsto (\vec z_{\tau_t})_{t\in[0, 1)}.$
\end{theorem}

Note when $B=2$, $(\vec z_t)_t$ is distributed as the Jacobi process described in~\citet{Avdeyev2023-mv}.

When $B>2$ \citet{Avdeyev2023-mv} considers $B-1$ parallel Wright-Fisher processes with $B=2$;
they then use a stick-breaking procedure to get an SDE on the simplex.
This SDE is distinct to ours in Eqn.~\ref{eqn: wf sde} and is not symmetric to the order of letters -- it requires us to specify a first letter, second letter, and so on, which behave differently in paths $(x_t)_t$ -- except for at stationary.
We instead directly consider the multi-allelic Wright-Fisher from \citet[Chapter~10]{Ethier1986-gm} which is invariant to permutations of letters in the alphabet and simplifies our derivations.

\subsubsection{Convergence of the ELBO}\label{app: wf proof elbo}

Call $\vec s(\vec v\mid x_0)=\nabla\log p(z_t|x_0, t)|_{z_t=\vec v},$ and $\vec s(\vec v\mid\tilde x_0, t)=\sum_b\tilde x_{0, b}\vec s(\vec v\mid x_0=b, t)$.

\begin{theorem}
    (Proof of Thm~\ref{thm: simplex loss})
    Call the ELBO in Alg.~\ref{alg: discrete diffusion}
    $$L^\zeta(\vec x^{\zeta}, t, x_0, \tilde x_0)=\sum_{\vec x_t^{\prime\zeta}\neq \vec x_t^{\zeta}}(\zeta\mathcal L_{\vec x_t^{\prime\zeta}\to \vec x_t^{\zeta}}^{\mathrm{wf}}+\mathcal L_{\vec x_t^{\prime\zeta}\to \vec x_t^{\zeta}}^{\mathrm{mut}})\dot\tau_t\mathbb{D}\left(\frac{p(\vec x_t^{\prime\zeta}\mid x_0, t)}{p(\vec x_t^{\zeta}\mid x_0, t)}\bigg|\bigg|\sum_b\tilde x_{0b}\frac{p(\vec x_t^{\prime\zeta}\mid x_0=b, t)}{p(\vec x_t^{\zeta}\mid x_0=b, t)}\right).$$
    For $\vec v$ for which $\zeta \vec v$ are not integers, define $L(\vec v, t, x_0, \tilde x_0)=L^\zeta(\vec x^{\zeta}, t, x_0, \tilde x_0)$ for a $\vec x^{\zeta}$ nearest to $\vec v$.
    Then, for all $\vec v$ in the interior of $\Delta^B$, $t\in(0, 1)$, $\tilde x_0\in\Delta^B$, and $x_0$,
    $$L^\zeta(\vec v, t, x_0, \tilde x_0)\to\frac{\dot\tau_t}{2}\|\vec s(\vec v\mid x_0, t)-\vec s(\vec v\mid \tilde x_0, t)\|^2_{\mathrm{diag}{\vec v}-\vec v\vec v^T}$$
\end{theorem}
\begin{proof}

\textbf{Overview of proof:}
For notational convenience, define
$$\mathbb D(\vec x^{\prime\zeta}_t)=\mathbb{D}\left(\frac{p(\vec x_t^{\prime\zeta}\mid x_0, t)}{p(\vec x_t^{\zeta}\mid x_0, t)}\bigg|\bigg|\sum_b\tilde x_{0b}\frac{p(\vec x_t^{\prime\zeta}\mid x_0=b, t)}{p(\vec x_t^{\zeta}\mid x_0=b, t)}\right).$$
Much of the proof consists of checking uniform convergence and regularity conditions.
The main idea however is that when $\zeta$ is very large, the transition rates $\zeta\mathcal L_{\vec x_t^{\prime\zeta}\to \vec x_t^{\zeta}}^{\mathrm{wf}}+\mathcal L_{\vec x_t^{\prime\zeta}\to \vec x_t^{\zeta}}^{\mathrm{mut}}$ are only large for $\vec x^{\prime\zeta}_t$ very close to $\vec v$. 
For those terms, we can perform a second order Taylor expansion 
\begin{equation*}
\begin{aligned}
    \mathbb{D}(\vec x_t^{\prime\zeta})\approx&\frac 1 2\Vert \vec s(\vec v\mid x_0, t)-\vec s(\vec v\mid \tilde x_0, t)\Vert^2_{(\vec x_t^{\prime\zeta}-\vec v)(\vec x_t^{\prime\zeta}-\vec v)^T}
    \end{aligned}
\end{equation*}
so
$$L^\zeta(\vec v, t, x_0, \tilde x_0)\approx\frac{\dot\tau_t}{2}\Vert \vec s(\vec v\mid x_0, t)-\vec s(\vec v\mid \tilde x_0, t)\Vert^2_{\Sigma}$$
where $\Sigma = (\sum_{\vec x_t^{\prime\zeta}\neq \vec x_t^{\zeta}}\zeta\mathcal L_{\vec x_t^{\prime\zeta}\to \vec x_t^{\zeta}}^{\mathrm{wf}}+\mathcal L_{\vec x_t^{\prime\zeta}\to \vec x_t^{\zeta}}^{\mathrm{mut}})(\vec x_t^{\prime\zeta}-\vec v)(\vec x_t^{\prime\zeta}-\vec v)^T.$
Finally, 
we show 
$\sum_{\vec x_t^{\prime\zeta}\neq \vec x_t^{\zeta}}\mathcal L_{\vec x_t^{\prime\zeta}\to \vec x_t^{\zeta}}^{\mathrm{mut}}(\vec x_t^{\prime\zeta}-\vec v)(\vec x_t^{\prime\zeta}-\vec v)^T\to 0,$
and through a central limit theorem,
$\sum_{\vec x_t^{\prime\zeta}\neq \vec x_t^{\zeta}}\zeta\mathcal L_{\vec x_t^{\prime\zeta}\to \vec x_t^{\zeta}}^{\mathrm{wf}}(\vec x_t^{\prime\zeta}-\vec v)(\vec x_t^{\prime\zeta}-\vec v)^T\to\diag\vec v-\vec v\vec v^T.$

Crucial to our proof is Lem.~\ref{lem: finite wf} which states that for each $\vec x^\zeta_t$ in the interior of the simplex,
$$p(\vec x^\zeta_t\mid x_0, t)=\mathbb E_{m\sim A^{(\zeta)}(\psi, \tau_t)}\mathbb E_{\vec p\sim \mathrm{Dir}(\psi\vec\pi+m\vec x_0)}\mathrm{Mult}(\zeta, \vec p)(\vec x_t^\zeta)$$
for a distribution $A^{(\zeta)}(\psi, \tau_t)$ such that $A^{(\zeta)}(\psi, \tau_t)(m)\to A(\psi, \tau_t)(m)$ quickly for each $m\leq \zeta$ as $\zeta\to\infty.$

\paragraph{Part 1: Eliminating the boundary}
For a small $\epsilon>0$, call $\Delta^B_\epsilon$ the points in $\Delta^B$ that have an entry less than $\epsilon$;
in particular, define $\epsilon<(4B)^{-2/\min_b\vec v_b}.$
We first show that the contribution form the epsilon-boundary vanishes, i.e.
$$E(\zeta)=\sum_{\vec x_t^{\prime\zeta}\not\in\Delta_\epsilon^B}(\zeta\mathcal L_{\vec x_t^{\prime\zeta}\to \vec x_t^{\zeta}}^{\mathrm{wf}}+\mathcal L_{\vec x_t^{\prime\zeta}\to \vec x_t^{\zeta}}^{\mathrm{mut}})\dot\tau_t\mathbb D(\vec x^{\prime\zeta}_t)\to 0.$$
First note for large enough $\zeta$, $\mathcal L_{\vec x_t^{\prime\zeta}\to \vec x_t^{\zeta}}^{\mathrm{mut}}=0$ for all $\vec x_t^{\prime\zeta}\not\in\Delta_\epsilon^B$ and 
$$\mathcal L_{\vec x_t^{\prime\zeta}\to \vec x_t^{\zeta}}^{\mathrm{wf}}\leq \binom{\zeta}{\zeta\vec x^\zeta_t}(\min_b\vec x^{\prime\zeta}_{t, b})^{\min_b\zeta\vec x^\zeta_{t, b}(\vec v)}\leq C(\epsilon^{2\min_b\vec v_b})^{\zeta}<C(4B)^{-\zeta}$$
for some $C>0$.
Also note for any $\vec x_t^{\zeta}$,
$$1\geq p(\vec x_t^{\zeta}\mid x_0, t)\geq p(A(\psi, \tau_t)=0)\mathbb E_{\vec p\sim \mathrm{Dir}(\psi\vec\pi)}\mathrm{Mult}(\zeta, \vec p)(\vec x_t^\zeta)$$
Taking the leading term of the divergence $\mathbb D(\vec x_t^{\prime \zeta})$,
$$E(\zeta)\lesssim \sum_{\vec x_t^{\prime\zeta}\not\in\Delta_\epsilon^B}(4B)^{-\zeta}\frac{-\log \mathbb E_{\vec p\sim \mathrm{Dir}(\psi\vec\pi)}\mathrm{Mult}(\zeta, \vec p)(\vec x_t^{\prime\zeta})}{\mathbb E_{\vec p\sim \mathrm{Dir}(\psi\vec\pi)}\mathrm{Mult}(\zeta, \vec p)(\vec x_t^{\prime\zeta})}.$$
Now $\mathrm{Mult}(\zeta, \vec p)(\vec x_t^{\prime\zeta})\geq(\min_b \vec p_b)^\zeta$ so the denominator is $\geq P_{\vec p\sim \mathrm{Dir}(\psi\vec\pi)}(\min_b\vec p\geq 1/2B)(2B)^{-\zeta}.$
Therefore
\begin{equation*}
    \begin{aligned}
        E(\zeta)\lesssim& (4B)^{-\zeta}\sum_{\vec x_t^{\prime\zeta}\not\in\Delta_\epsilon^B}\frac{\zeta\log(2B)}{(2B)^{-\zeta}}\\
        \lesssim&2^{-\zeta}\zeta\times \zeta^{B-1}\\
        \to& 0
    \end{aligned}
\end{equation*}
since there are $O(\zeta^{B-1})$ elements with $\vec x_t^{\prime\zeta}\not\in\Delta_\epsilon^B$.

\paragraph{Part 2: Uniform convergence of the likelihood}
Next we show $\frac{p(\vec x_t^{\zeta}(\vec v)\mid x_0, t)}{p(\vec z_t=\vec v\mid x_0, t)}= 1+O(\zeta^{-1})$
uniformly in $\Delta^B_\epsilon$. 
While something like this is implied by the convergence of the process from previous work, the fast uniform convergence will be important for our results below.

We will do so by showing the same property for each of the quotients
$$\frac{\mathbb E_{m\sim A^{(\zeta)}(\psi, \tau_t)}\mathbb E_{\vec p\sim \mathrm{Dir}(\psi\vec\pi+m\vec x_0)}\mathrm{Mult}(\zeta, \vec p)(\vec x_t^\zeta(\vec v))}{\mathbb E_{m\sim A^{(\zeta)}(\psi, \tau_t)} \mathrm{Dir}(\psi\vec\pi+m\vec x_0)(\vec x_t^\zeta(\vec v))}, \frac{\mathbb E_{m\sim A^{(\zeta)}(\psi, \tau_t)} \mathrm{Dir}(\psi\vec\pi+m\vec x_0)(\vec x_t^\zeta(\vec v))}{\mathbb E_{m\sim A(\psi, \tau_t)} \mathrm{Dir}(\psi\vec\pi+m\vec x_0)(\vec v)}.$$

The first quotient converges by the concentration of a Bayesian posterior~\citep{Miller2019-sh}.
In particular, by the uniform Stirling approximation~\citep{Robbins1955-td} uniformly for $\vec x^{\zeta}_t\in\Delta^B_{\epsilon/2}$,
\begin{equation*}
    \begin{aligned}
        \mathrm{Mult}(\zeta, \vec p)(\vec x_t^\zeta)=&(1+O(\zeta^{-1}))\left(\prod_b \vec x^{\zeta}_b\right)^{-1/2}(2\pi\zeta)^{-(B-1)/2}e^{-\zeta\mathrm{KL}(\vec x^{\zeta}_t||\vec p)}.
        \end{aligned}
\end{equation*}
We'd like to write this as approximately a normal density with mean $\vec x^{\zeta}$ and variance restricted to vectors summing to $0$ $\{\vec w\in\mathbb R^B\mid \vec w^T\mathbbm 1=0\}$.
We can do so with a Taylor expansion; for $\vec p$ near $\vec x^{\zeta}$,
$$\mathrm{KL}(\vec x^{\zeta}_t||\vec p)=\frac 1 2 \|\vec x^{\zeta}_t-\vec p\|^2_{\mathrm{diag}(\vec x_t^\zeta)^{-1}} -O(\|\vec x^{\zeta}_t-\vec p\|^3).$$
We can also write $\|\vec x^{\zeta}-\vec p\|^2_{\mathrm{diag}(\vec x_t^\zeta)^{-1}}=\|\vec x^{\zeta}-\vec p\|^2_{\Lambda^{\dagger}}$ where $\Lambda=\diag\vec x^{\zeta}-\vec x^{\zeta}\vec x^{\zeta T}$ has kernel orthogonal to vectors summing to $0$ and $$\Lambda^\dagger=\diag(\vec x^{\zeta}_t)^{-1}-\frac 1B\vec x^{\zeta-1}_t\mathbbm 1^T-\frac 1B\mathbbm 1\vec x^{\zeta-1, T}_t+\frac{\sum_b\vec x^{\zeta-1}_{t, b}}{B^2}\mathbbm 1\mathbbm 1^T.$$
Note also that the pseudo-determinant of $\Lambda$ is $\prod_b \vec x^{\zeta}_b$
so we can write
\begin{equation*}
    \begin{aligned}
        \mathrm{Mult}(\zeta, \vec p)(\vec x_t^\zeta)=&(1+O(\zeta^{-1}))(1-O(\zeta\|\vec x^{\zeta}-\vec p\|^3))\mathcal N(\vec x^\zeta_t, \zeta^{-1}\Lambda).
        \end{aligned}
\end{equation*}
This allows us to write
$$\frac{\mathbb E_{P(\vec p)}\mathrm{Mult}(\zeta, \vec p)(\vec x_t^\zeta)}{P(\vec x_t^\zeta)}=(1+O(\zeta^{-1}))\frac{\mathbb E_{\vec w\sim \mathcal N(0, \Lambda)}P(\vec x_t^\zeta+\zeta^{-1/2}\vec w)(1-O(\|\vec w\|^3/\zeta^{1/2}))}{P(\vec x_t^\zeta)}.$$
For a small $\delta < \epsilon/4$ call $\phi$ a $C^\infty$ function with support in the $\delta$-ball, and which is $1$ in the $\delta/2$-ball.
We break the numerator up into
\begin{equation*}
    \begin{aligned}
    \mathbb E_{\vec w\sim \mathcal N(0, \Lambda)}&\phi(\zeta^{-1/2}\vec w)P(\vec x_t^\zeta+\zeta^{-1/2}\vec w)(1-O(\|\vec w\|^3/\zeta^{1/2}))\\
    &+\mathbb E_{\vec w\sim \mathcal N(0, \Lambda)}(1-\phi(\zeta^{-1/2}\vec w))P(\vec x_t^\zeta+\zeta^{-1/2}\vec w)(1-O(\|\vec w\|^3/\zeta^{1/2})).
    \end{aligned}
\end{equation*}
The second term is less than
\begin{equation*}
    \begin{aligned}
    \mathbb E_{P(\vec p)}(1-\phi(\vec p-\vec x_t^\zeta))\mathcal N(0, \Lambda)(\sqrt\zeta(\vec x_t^\zeta-\vec p))\leq &\mathbb E_{P(\vec p)}\mathbbm 1(\|\vec x_t^\zeta-\vec p\|>\delta/2)\mathcal N(0, \Lambda)(\sqrt\zeta(\vec x_t^\zeta-\vec p))\\
    \lesssim &e^{-\zeta C\delta^2}
    \end{aligned}
\end{equation*}
for some $C$.
For the first term, we can define $\tilde P(\vec p)=P(\vec p)\phi(\vec p-\vec x_t^\zeta)$ which is a compactly supported $C^\infty$ function.
Therefore
\begin{equation*}
    \begin{aligned}
    \mathbb E_{\vec w\sim \mathcal N(0, \Lambda)}&\tilde P(\vec x_t^\zeta+\zeta^{-1/2}\vec w)(1-O(\|\vec w\|^3/\zeta^{1/2}))\\
    =&\tilde P(\vec x_t^\zeta)+\nabla\tilde P(\vec x_t^\zeta)^T\mathbb E_{\vec w\sim \mathcal N(0, \Lambda)}(\zeta^{-1/2}\vec w+O(\zeta\|w\|^2))(1-O(\|\vec w\|^3/\zeta^{1/2}))\\
    =&P(\vec x^\zeta_t)+O(\zeta^{-1}).
    \end{aligned}
\end{equation*}

For the second quotient, note the denominator is bounded below for $\vec v\in\Delta^B_\epsilon$.
By Lem.~\ref{lem: finite wf}
\begin{equation*}
\begin{aligned}
\sup_{\vec v\in\Delta^B_\epsilon}|\mathbb E_{m\sim A^{(\zeta)}(\psi, \tau_t)} \mathrm{Dir}&(\psi\vec\pi+m\vec x_0)(\vec x^\zeta_t(\vec v))-\mathbb E_{m\sim A(\psi, \tau_t)} \mathrm{Dir}(\psi\vec\pi+m\vec x_0)(\vec x^\zeta_t(\vec v))|\\
\lesssim & \zeta^{-1}\sum_m e^{-cm^2}\sup_{\vec v\in\Delta^B_{\epsilon/2}}\mathrm{Dir}(\psi\vec\pi+m\vec x_0)(\vec v).
\end{aligned}
\end{equation*}
Since $\sup_{\vec v\in\Delta^B_{\epsilon/2}}\mathrm{Dir}(\psi\vec\pi+m\vec x_0)(\vec v)\leq (m+\psi)^\psi(1-\epsilon/2)^{m-1}$ is eventually decreasing in $m$, the whole quotient is $O(\zeta^{-1}).$
Next note the derivative of $\mathbb E_{m\sim A(\psi, \tau_t)} \mathrm{Dir}(\psi\vec\pi+m\vec x_0)(\cdot)$ is bounded on the compact set $\Delta_\epsilon^B$ so 
\begin{equation*}
\begin{aligned}
\sup_{\vec v\in\Delta^B_\epsilon}|\mathbb E_{m\sim A(\psi, \tau_t)} \mathrm{Dir}(\psi\vec\pi+m\vec x_0)(\vec x^\zeta_t(\vec v))-&\mathbb E_{m\sim A(\psi, \tau_t)} \mathrm{Dir}(\psi\vec\pi+m\vec x_0)(\vec v)|\\
=&O(\|\vec x^\zeta_t(\vec v)-\vec v\|)=O(\zeta^{-1}).
\end{aligned}
\end{equation*}

\paragraph{Part 3: Taylor expansion of the divergence}
Given the calculation above, for $\zeta$ large enough and any $\vec x^{\prime\zeta}_t= \vec x^{\zeta}_t(\vec v)+O(\zeta^{-1/2})$, we can approximate
\begin{equation*}
\begin{aligned}
    \frac{p(\vec x_t^{\prime\zeta}\mid x_0, t)}{p(\vec x_t^{\zeta}(\vec v)\mid x_0, t)}=&\exp\left(\log p\left(\vec z_t=\vec x_t^{\prime\zeta}\mid x_0, t\right)-\log p\left(\vec z_t=\vec x_t^{\zeta}(\vec v)\mid x_0, t\right)\right)+O(\zeta^{-1})\\
    =&1+ \vec s(\vec v\mid x_0, t)^T(\vec x_t^{\prime\zeta}-\vec v)+O(\zeta^{-1}).
\end{aligned}
\end{equation*}
A second order Taylor expansion then gives
\begin{equation*}
\begin{aligned}
    \mathbb{D}(\vec x_t^{\prime\zeta})= &\frac 1 2\left((\vec s(\vec v\mid x_0, t)-\vec s(\vec v\mid \tilde x_0, t))^T(\vec x_t^{\prime\zeta}-\vec v)\right)^2+o(\zeta^{-1})\\
    =&\frac 1 2\Vert \vec s(\vec v\mid x_0, t)-\vec s(\vec v\mid \tilde x_0, t)\Vert^2_{(\vec x_t^{\prime\zeta}-\vec v)(\vec x_t^{\prime\zeta}-\vec v)^T}+o(\zeta^{-1}).
    \end{aligned}
\end{equation*}

Given the calculation above, we note that since $\mathcal L_{\vec x_t^{\prime\zeta}\to \vec x_t^{\zeta}(\vec v)}^{\mathrm{mut}}$ is only non-zero for $\zeta$ values of $\vec x_t^{\prime\zeta}$ each with $\vec x_t^{\prime\zeta} = \vec x_t^{\zeta}(\vec v)+O(\zeta^{-1}),$
$$\sum_{\vec x_t^{\prime\zeta}}\mathcal L_{\vec x_t^{\prime\zeta}\to \vec x_t^{\zeta}}^{\mathrm{mut}}\mathbb D(\vec x^{\prime\zeta}_t)=O(\zeta\times \zeta^{-2})=o(1).$$
This gives
$$L^\zeta(\vec v, t, x_0, \tilde x_0)=\frac{\dot\tau_t}{2}\Vert \vec s(\vec v\mid x_0, t)-\vec s(\vec v\mid \tilde x_0, t)\Vert^2_{\Sigma}+o(1)$$
where
$$\Sigma = \sum_{\vec x_t^{\prime\zeta}\in\Delta_\epsilon^B}\zeta\mathcal L_{\vec x_t^{\prime\zeta}\to \vec x_t^{\zeta}}^{\mathrm{wf}}(\vec x_t^{\prime\zeta}-\vec v)(\vec x_t^{\prime\zeta}-\vec v)^T.$$
The proof is therefore finished if we show $\Sigma\to\mathrm{diag}{\vec v}-\vec v\vec v^T.$

\paragraph{Part 4: Convergence of $\Sigma$}
Note, by the uniform Stirling approximation~\citep{Robbins1955-td} uniformly for $\vec x^{\prime\zeta}\in\Delta^B_\epsilon\setminus\{\vec x^\zeta_t\}$, the infinitesimal generator approximates a Normal distribution near $\vec v$,
\begin{equation*}
    \begin{aligned}
        \mathcal L^\mathrm{wf}_{\vec x^{\prime\zeta}\to \vec x^\zeta(\vec v)}=&(1+o(1))\left(\prod_b \vec v_b\right)^{-1/2}(2\pi\zeta)^{-(B-1)/2}e^{-\zeta\mathrm{KL}(\vec v||\vec x^{\prime\zeta})}\\
        =&(1+o(1)+O(\zeta\|\vec v-\vec x^{\prime\zeta} \|^3))\mathcal N\left(\vec v, \zeta^{-1}(\diag(\vec v)-\vec v\vec v^T)\right)(\vec x^{\prime\zeta})
    \end{aligned}
\end{equation*}
Noting that, by Pinsker's inequality, $\mathrm{KL}(\vec v||\vec x^{\prime\zeta})\geq 2\|\vec v-\vec x^{\prime\zeta}\|_{1}^2\geq \frac 2 { B}\|\vec v-\vec x^{\prime\zeta}\|^2$, for some very small $\delta>0$
\begin{equation*}
    \begin{aligned}
    \|\Sigma-(\mathrm{diag}{\vec v}-\vec v\vec v^T)\|\lesssim&\sum_{\vec x_t^{\prime\zeta}\in\Delta_\epsilon^B, \|\vec x^{\prime\zeta}_t-\vec v\|>\zeta^{-1/3-\delta}}\zeta^{-(B-1)/2+1}e^{-\zeta\mathrm{KL}(\vec v||\vec x^{\prime\zeta})}\\
    \lesssim&\zeta^{-(B-1)/2+B}e^{-\frac 2 { B}\zeta\zeta^{-2/3-2\delta}}\\
    =&o(1)
    \end{aligned}
\end{equation*}
\end{proof}

\subsection{Wright-Fisher loss calculations}\label{app: wf scores proof}
See the discussion above Prop.~\ref{prop: wf scores} for definitions.
\begin{proposition}
    (Proof of Prop.~\ref{prop: wf scores})
    $$p(\vec x_t\mid x_0, t)=\mathrm{Dirichlet}(\pi\psi)(\vec x_t)G_\psi(\tau_t, x_0, \vec x_t).$$
    For $\vec c(\vec x_t)=\nabla\log \mathrm{Dirichlet}(\pi\psi)(\vec x_t)$ which does not depend on $x_0$,
    \begin{equation*}
        \vec s = \vec s(\vec x_t\mid x_0, t)=\vec c(\vec x_t)+\vec x_0w(x_0)
    \end{equation*}
    where
    $$w(x_0) = \frac{e^{-\psi\tau_t/2}(\psi+1)}{\pi(x_0)}\frac{F_\psi(\tau_t, x_0, \vec x_t)}{G_\psi(\tau_t, x_0, \vec x_t)}.$$
\end{proposition}
\begin{proof}
    For $m_t\sim A(\psi, \tau_t)$,
\begin{equation*}
    \begin{aligned}
        p(\vec x_t\mid x_0, t)=&\mathbb E_{m_t}\mathrm{Dirichlet}(\psi\pi+m_tx_0)(\vec x_t)\\
        =&\prod_{b\neq x_0}\vec x_{t, b}^{\psi\pi_b-1}\mathbb E_{m_t}\frac{\Gamma(\psi+m_t)}{\Gamma(\psi\pi_{x_0}+m_t)\prod_{b\neq x_0}\Gamma(\psi\pi_b)}\vec x_{t, x_0}^{\psi\pi_{x_0}+m_t-1}\\
        =&\frac{\Gamma(\psi)}{\prod_{b\in\mathcal B}\Gamma(\psi\pi_b)}\prod_{b\in\mathcal B}\vec x_{t, b}^{\psi\pi_b-1}\mathbb E_{m_t}\frac{\Gamma(\psi\pi(x_0))\Gamma(\psi+m_t)}{\Gamma(\psi)\Gamma(\psi\pi_{x_0}+m_t)}\vec x_{t, x_0}^{m_t}\\
        =&\mathrm{Dirichlet}(\psi\pi)(\vec x_t)
        \mathbb E_{m_t}\frac{(\psi)_{(m_t)}}{(\psi\pi(x_0))_{(m_t)}}\vec x_{t, x_0}^{m_t}.\\
    \end{aligned}
\end{equation*}

From Eqn. 5.2 of \citet{Tavare1984-uz}, we have
\begin{equation*}
p(m_t=j)=\sum_{k=j}^\infty e^{-k(k+\psi-1)\tau_t/2}(-1)^k(-1)^j\frac{(2k+\psi-1)(j+\psi)_{(k-1)}}{j!(k-j)!}.
\end{equation*}
He wrote, in Eqn. A5,
\begin{equation*}
\begin{aligned}
    \sum_{j=1}^\infty &x^jp(m_t=j)\\
    =&\sum_{k=1}^\infty e^{-k(k+\psi-1)\tau_t/2}(-1)^k(2k+\psi-1)\sum_{j=1}^k\frac{x^j}{j!}\frac{(j+\psi)_{(k-1)}}{(k-j)!(-1)^j}\\
    =&\sum_{k=1}^\infty e^{-k(k+\psi-1)\tau_t/2}(-1)^k(2k+\psi-1)\sum_{j=1}^k\frac{x^j}{j!}\frac{(\psi)_{(j+k-1)}(-k)_{(j)}}{k!\psi_{(j)}}\\
    =&\sum_{k=1}^\infty e^{-k(k+\psi-1)\tau_t/2}\frac{(-1)^k(2k+\psi-1)(\psi)_{(k-1)}}{k!}\sum_{j=1}^k\frac{x^j}{j!}\frac{(\psi+k-1)_{(j)}(-k)_{(j)}}{\psi_{(j)}}.
\end{aligned}
\end{equation*}
The last sum is then written as $_2F_1(-k,\psi+k-1;\psi;x)-1$ for the hyper-geometric function $_2F_1$.
A very simple extension gives us
\begin{equation*}
\begin{aligned}
\sum_{j=1}^\infty \frac{(\psi)_{(j)}}{(\psi\pi_{x_0})_{(j)}}x^jp(m_t=j)=\sum_{k=1}^\infty &e^{-k(k+\psi-1)t/2}\frac{(-1)^k(2k+\psi-1)(\psi)_{(k-1)}}{k!}\\
&\times\left(_2F_1(-k,\psi+k-1;\psi\pi_{x_0};x)-1\right).
\end{aligned}
\end{equation*}
Including the $j=0$ term, by Eqn 5.3 of \citet{Tavare1984-uz}, cancels out the $-1$ in the brackets above, so our expectation
\begin{equation*}
\begin{aligned}
E_{m_t}\frac{(\psi)_{(m_t)}}{(\psi\pi_{x_0})_{(m_t)}}\vec x_{t, x_0}^{m_t}=&1+\sum_{k=1}^\infty e^{-k(k+\psi-1)\tau_t/2}\frac{(-1)^k(2k+\psi-1)(\psi)_{(k-1)}}{k!}\\
&\ \ \ \ \ \ \ \ \ \ \ \ \times{_2F_1(-k,\psi+k-1;\psi\pi_{x_0};\vec x_{t, x_0})}\\
=&G_\psi(t, x_0, \vec x_t).
\end{aligned}
\end{equation*}
Finally, using identities of the hypergeometric function,
\begin{equation*}
\begin{aligned}\nabla_{\vec x_{t, x_0}} G_\psi(t, x_0, \vec x_t)=&\sum_{k=1}^\infty e^{-k(k+\psi-1)\tau_t/2}\frac{(-1)^k(2k+\psi-1)(\psi)_{(k-1)}}{k!}\frac{-k(\psi+k-1)}{\psi\pi_{x_0}}\\
&\ \ \ \ \ \ \times{_2F_1(-k+1,\psi+k;\psi\pi_{x_0}+1;\vec x_{t, x_0})}\\
=&\frac{1}{\psi\pi_{x_0}}\sum_{k=1}^\infty e^{-k(k+\psi-1)\tau_t/2}\frac{(-1)^{k-1}(2k+\psi-1)(\psi+k-1)(\psi)_{(k-1)}}{(k-1)!}\\
&\ \ \ \ \ \ \ \ \ \ \ \ \ \ \ \ \ \ \ \ \ \times{_2F_1(-k+1,\psi+k;\psi\pi_{x_0}+1;\vec x_{t, x_0})}\\
=&\frac{1}{\psi\pi_{x_0}}\sum_{k=0}^\infty e^{-(k+1)(k+\psi)\tau_t/2}\frac{(-1)^{k}(2k+\psi+1)(\psi+k)(\psi)_{(k)}}{k!}\\
&\ \ \ \ \ \ \ \ \ \ \ \ \ \ \ \ \ \ \ \ \ \times{_2F_1(-k,\psi+k+1;\psi\pi_{x_0}+1;\vec x_{t, x_0})}\\
=&\frac{e^{-\psi t/2}(\psi+1)}{\pi_{x_0}}\sum_{k=0}^\infty e^{-k(k+\psi+1)\tau_t/2}\frac{(-1)^{k}(\psi)_{(k)}}{k!}\frac{(2k+\psi+1)(\psi+k)}{(\psi+1)\psi}\\
&\ \ \ \ \ \ \ \ \ \ \ \ \ \ \ \ \ \ \ \ \ \ \ \ \ \ \ \ \ \ \ \times{_2F_1(-k,\psi+k+1;\psi\pi_{x_0}+1;\vec x_{t, x_0})}\\
=:&\frac{e^{-\psi t/2}(\psi+1)}{\pi_{x_0}}F_\psi(t, x_0, \vec x_t).
\end{aligned}
\end{equation*}
\end{proof}

\subsection{Proof of sufficient statistics}\label{app: ssp proof}

\begin{proposition}
    (Proof of Prop.~\ref{prop: ssp})
    There is a function $F^d$, \textbf{depending on $p(x_0)$ and not on the diffusion process or $t$,}
    such that
    $$p(x_0^d\mid x_t^{-d}, t)=F^d(\vec\phi(\vec x_t^1, t), \dots, \vec\phi(\vec x_t^D, t)).$$
\end{proposition}
\begin{proof}
    \begin{equation*}
        \begin{aligned}
            p(x_0^d\mid x_t^{-d})=&\int p(x_0^d\mid x_0^{-d})dp(x_0^{-d}\mid x_t^{-d})\\
            =& \frac{1}{p(x_t^{-d})}\int p(x_0^d\mid x_0^{-d})p(x_t^{-d}\mid x_0^{-d})dp(x_0^{-d})\\
            = &\frac{1}{p(x_t^{-d})}\int p(x_0^d\mid x_0^{-d})\prod_{d'\neq d}p(x_t^{d'}\mid x_0^{d'})dp(x_0^{-d})\\
            = &\frac{\prod_{d'\neq d}\sum_b p(x_t^{d'}\mid x_0^{d'}=b)}{p(x_t^{-d})}\int p(x_0^d\mid x_0^{-d})\prod_{d'\neq d}\frac{p(x_t^{d'}\mid x_0^{d'})}{\sum_b p(x_t^{d'}\mid x_0^{d'}=b)}dp(x_0^{-d})\\
            =& E_{p(x_0^{-d})}\left(p(x_0^d|x_0^{-d})\prod_{d'\neq d}\vec \phi(x_t^{d'})_{x_0^{d'}}\right)/E_{p(x_0^{-d})}\left(\prod_{d'\neq d}\vec \phi(x_t^{d'})_{x_0^{d'}}\right),
        \end{aligned}
    \end{equation*}
\end{proof}

\subsection{Lemmas}

Our first lemma establishes conditions for convergence of paths using standard techniques inspired by arguments used throughout ~\citet{Ethier1986-gm} or \citet{Bass2011-ut} for example.
\begin{lemma}\label{lem: convergence}
    Say $(\vec x_t^\zeta)_{t\in(0, 1)}$ are Markov processes on $\mathbb R^r$ for $\zeta=1, 2, \dots$ and $(\vec z_t)_{t\in(0, 1)}$ is another Markov process on $\mathbb R^r$.
    Say the following conditions are satisfied
    \begin{enumerate}
        \item (Convergence of marginals) $\vec x_t^\zeta\leadsto \vec z_t$ for each $t$.\label{as marg conv}
        \item (Local uniform convergence of conditionals) Conditional distributions exist such that for each $\vec v\in\mathbb R^r$, $s<t$, and bounded compactly supported measurable function $f$, there is an $\epsilon>0$, such that $$\sup_{\|\vec w-\vec v\|<\epsilon}|\mathbb E_{\vec x_t^\zeta\mid \vec x_s^\zeta=\vec w}f-\mathbb E_{\vec z_t\mid\vec z_s=\vec w}f|\to 0.$$\label{as uniform conv}
        \item \sloppy(Tightness) For every $[a, b]\subset (0, 1)$, there are $\beta, \theta, M>0$ such that for all $s, t\in[a, b]$, $\sup_{\zeta>M}\mathbb E\|\vec x^\zeta_{s}-\vec x^\zeta_{t}\|^\beta<C(s-t)^\theta.$ \label{as: tight}
    \end{enumerate}
    Then, with the topology of convergence on compact sets\footnote{This is a standard topology for these results. See for example Thm 1.1 of \citet[Chapter~10]{Ethier1986-gm}.}, the paths converge in distribution 
    $$(\vec x_t^\zeta)_{t\in(0, 1)}\leadsto (\vec z_t)_{t\in(0, 1)}.$$
\end{lemma}
\begin{proof}
    Pick a compact set $[a, b]\subset(0, 1)$.
     We show $(\vec x_t^\zeta)_{t\in[a, b]}\leadsto (\vec z_t)_{t\in[a, b]}.$
    Say $(\vec x_t^{\zeta_m})_{t\in[a, b]}$ is a subsequence which doesn't enter a neighbourhood of $(\vec z_t)_{t\in[a, b]}$; we'll now show a contradiction.
    By Prokhorov's theorem, since it's tight by Assumption~\ref{as: tight} and Thm. 8.8 of~\citet[Chapter~3]{Ethier1986-gm}, it has a subsequence which converges to a process $(\vec y_t)_{t\in[a, b]}$.
    As we'll show below, for every set $a\leq t_1<t_2< \dots<t_m\leq b$, $(\vec y_t)_{t\in\{t_i\}_{i=1}^m}=(\vec z_t)_{t\in\{t_i\}_{i=1}^m}$.
    This must mean $(\vec y_t)_t=(\vec z_t)_t$ by the Kolmogorov extension theorem, a contradiction.

    What remains is to show, for $a\leq t_1<t_2< \dots<t_m\leq b$, $(\vec x_t^\zeta)_{t\in\{t_i\}_{i=1}^m}\leadsto(\vec z_t)_{t\in\{t_i\}_{i=1}^m}.$
    It is sufficient to prove that for any $t_1< \dots<t_m $ and compactly supported continuous function on $\mathbb R^r$, $h$,
    \begin{equation}\label{eq: condition}
        Eh(\vec x^\zeta_1, \dots, \vec x^\zeta_m)\to Eh(\vec z_1, \dots, \vec z_m).
    \end{equation}
    By the Stone-Weierstrass theorem, each such $h$ can be arbitrarily well approximated by product of $m$ univariate functions, so it is sufficient to consider $h(\vec z_1, \dots, \vec z_m)=\prod_{i=1}^mh_i(\vec z_i).$
    Finally, by the Markov property,
    $$\mathbb Eh(\vec x^\zeta_1, \dots, \vec x^\zeta_m)=\mathbb E_{\vec x_1^\zeta|\vec x_0^\zeta}h_1(\vec x_1^\zeta)\mathbb E_{\vec x_2^\zeta|\vec x_1^\zeta}h_2(\vec x_2^\zeta)\cdots \mathbb E_{\vec x_m^\zeta|\vec x_{m-1}^\zeta}h_m(\vec x_m^\zeta).$$
    We can call $\tilde h_{m-1}^\zeta(\vec x_{m-1}^\zeta)= h_m(\vec x_{m-1}^\zeta)E_{\vec x_m^\zeta|\vec x_{m-1}^\zeta}h_m(\vec x_m^\zeta)$.
    By Assumption~\ref{as uniform conv} $\tilde h_{m-1}^\zeta(\vec x_{m-1}^\zeta)$ converges uniformly to $h_m(\vec x_{m-1}^\zeta)E_{\vec z_m|\vec z_{m-1}=\vec x_{m-1}^\zeta}h_m(\vec z_m)$, a bounded function with compact support.
    Therefore, to prove Eqn.~\ref{eq: condition} it is sufficient to show
    the result replacing $h$ with $h_1\times h_2\times\dots\times h_{m-2}\times\tilde h_{m-1}.$
    By induction, we reach $h=\tilde h_1$for which we get Eqn.~\ref{eq: condition} by Assumption~\ref{as marg conv}.
\end{proof}

Our next Lemma is a non-asymptotic bound on the convergence of multinomials to Normal distributions.
It states that as long as $\zeta\to\infty$ and the probabilities don't get too low, we can bound the expectation of a function by $O(\zeta^{-1/2}).$
\begin{lemma}\label{lem:berry-esseen}
Let $Y_\zeta \sim \text{Mult}(\zeta, \vec p)$ for probability vector $\vec p\in\mathbb R^B$ with $\min_i p_i \geq c> 0$.
Call $Z_\zeta = \zeta^{-1/2}(Y_\zeta - \zeta p)$.
For any bounded measurable function $f$,
$$\left|\mathbb Ef(Z_\zeta) - \mathbb Ef(Z)\right| = o_{c, B, f}(1)$$
where $Z \sim \mathcal{N}(0, \diag(\vec p) - \vec p\vec p^T)$ and the rate of decay $o_{c, B, f}(1)$ only depends on $c$, $B$, and $f$.
\end{lemma}
\begin{proof}
    For every $\epsilon$, pick a compactly supported $C^\infty$ function $g_\epsilon$ such that $\|g_\epsilon-f\|_\infty<\epsilon/2$, so
    $$\left|\mathbb Ef(Z_\zeta) - \mathbb Ef(Z)\right| = \epsilon+\left|\mathbb Eg_\epsilon(Z_\zeta) - \mathbb Eg_\epsilon(Z)\right|=\epsilon+o_{c, B, g_\epsilon}(1)$$
    by Thm 1.3 of~\citet{Gotze1991-ei}.
\end{proof}

Our final lemma characterizes the distribution of the finite population Wright-Fisher process as described in Sec.~\ref{sec: wf} and App.~\ref{app: wf details}.
\begin{lemma}\label{lem: finite wf}
    For each $\vec x^\zeta_t$ in the interior of the simplex,
    $$p(\vec x^\zeta_t\mid x_0, t)=\mathbb E_{m\sim A^{(\zeta)}(\psi, \tau_t)}\mathbb E_{\vec p\sim \mathrm{Dir}(\psi\vec\pi+m\vec x_0)}\mathrm{Mult}(\zeta, \vec p)(\vec x_t^\zeta)$$
    for a distribution over the natural numbers $A^{(\zeta)}(\psi, \tau_t)$ supported on $\{1, \dots, \zeta\}$ such that $|A^{(\zeta)}(\psi, \tau_t)(m)- A(\psi, \tau_t)(m)|=C\zeta^{-1}\exp(-C'm^2)$ for constants $C, C'$ only depending on $\psi, \tau_t$, each $m$.
\end{lemma}
\begin{proof}
    This is standard in the population genetics literature.
    Define $A^{(\zeta)}(\psi, \tau_t)(m)$ the probability that $m$ alleles survive backwards in the coalescent of population $\zeta$ up to time $\tau_t$.
    Conditioned on observing $m$ individuals with allele $x_0$, \citet{Hoppe1984-in} showed that sampling more individuals from the population is equivalent to sampling from a Pólya urn with allele probabilities $\psi\vec\pi+m\vec x_0$, giving the Dirichlet multinomials.

    \citet{Tavare1984-uz} shows $A(\psi, \tau_t)(m)=\lim_{\zeta\to\infty}A^{(\zeta)}(\psi, \tau_t)(m)$ and for $m>0$
    $$A^{(\zeta)}(\psi, \tau_t)(m)=\sum_{k=m}^\zeta e^{-k(k+\psi-1)\tau_t/2}(-1)^{k-m}\frac{(2k+\psi-1)(m+\psi)_{(k-1)}}{m!(k-m)!}\frac{(\zeta-k+1)_{(k)}}{(\zeta+\psi)_{(k)}}.$$
    Note
    $$\frac{(m+\psi)_{(k-1)}}{(k-m)!}\leq\frac{(m+\psi)_{(k-1)}}{(k-1)!}\frac{(k-1)!}{(k-m)!}\leq k^{m+\psi}(m+\psi)^ck^m$$
    for some $c>0$ and 
    $$\left|\frac{(\zeta-k+1)_{(k)}}{(\zeta+\psi)_{(k)}}-1\right|\lesssim k^2/\zeta.$$
    Therefore
    \begin{equation*}
    \begin{aligned}|
        A^{(\zeta)}(\psi, \tau_t)(m)-A(\psi, \tau_t)(m)|\lesssim &\sum_{k=m}^\infty e^{-k(k+\psi-1)\tau_t/2}(2k+\psi-1)k^{2m+\psi}\frac{m^c}{m!}\left(\frac{k^2}{\zeta}\wedge 1\right)\\
        \lesssim&\zeta^{-1}\frac{m^c}{m!}\sum_{k=m}^\infty e^{-k(k+\psi-1)\tau_t/2}k^{2m+\psi+3}\\
        \leq&\zeta^{-1}\frac{m^c}{m!}e^{-m(m+\psi-1)\tau_t/2}\sum_{j=0}^\infty e^{-j(m+\psi-1)\tau_t/2}(j+m)^{2m+\psi+3}\\
    \end{aligned}
    \end{equation*}
    and
    \begin{equation*}
    \begin{aligned}
        \sum_{j=0}^\infty e^{-j(m+\psi-1)\tau_t/2}(j+m)^{2m+\psi+3}\leq &\sum_{j=0}^m (j+m)^{2m+\psi+3}\\
        &+\sum_{j=m}^\infty e^{-j(m+\psi-1)\tau_t/2}(j+m)^{2m+\psi+3}\\
        \leq &m (2m)^{2m+\psi+3}+\sum_{j=m}^\infty e^{-j(m+\psi-1)\tau_t/2}(2j)^{2m+\psi+3}\\
        \leq &(2m)^{2m+\psi+4}+2^{2m+\psi+3}\sum_{j=0}^\infty e^{-j(m+\psi-1)\tau_t/2}j^{2m+\psi+3}\\
        = &(2m)^{2m+\psi+4}+2^{2m+\psi+3}\sum_{j=1}^\infty e^{-(j-\frac{4}{\tau_t}\log j)(m+\psi-1)\tau_t/2}j^{5-\psi}\\
        \leq &(2m)^{2m+\psi+4}+2^{2m+\psi+3}\sum_{j=1}^\infty e^{-(j-\frac{4}{\tau_t}\log j)(\psi-1)\tau_t/2}j^{5-\psi}\\
        \lesssim &(2m)^{2m+\psi+4}+2^{2m+\psi+3}.
    \end{aligned}
    \end{equation*}
\end{proof}

\section{Experimental Details}\label{app: experiment-dets}

\subsection{DNA}
We describe the experiments in Sec.~\ref{sec: application wf improved}.

\paragraph{Training and data}
For all DNA models, we use the same base CNN model and optimizer hyperparameters used to train DDSM \citep{Avdeyev2023-mv} and Dirichlet flow matching \citep{Stark2024-rf} with code from \url{https://github.com/jzhoulab/ddsm} used with compliance with their licence and
\url{https://github.com/HannesStark/dirichlet-flow-matching} used with an MIT licence.
We train our Wright-Fisher simplicial model on the FlyBrain enhancer data from \url{https://zenodo.org/records/10184648}.

We trained our model on an A100 80GB GPU over 11 h for 700 epochs like \citet{Avdeyev2023-mv}.
We trained a DDSM model using the code in \url{https://github.com/jzhoulab/ddsm} and used a pre-trained flow-matching model from \url{https://github.com/HannesStark/dirichlet-flow-matching}.

\paragraph{Computational comparison}
All three models we tested need to pass their noisy $\vec x_t$ through a neural network.
We chose the same neural network for our diffusion model as used in \citep{Avdeyev2023-mv} and \citep{Stark2024-rf}.
For a reasonably sized model, like the ESM model for protein experiments, the neural network computations took 75\% of our compute time on average, meaning the overhead from sampling and loss computations cannot be more than 25\%.

However the DNA architecture was very small, at only 3 M parameters.
For the DNA setting then we precomputed and cached $\vec x_t, F_\psi$ and $G_\psi$ so that a majority of training time would come from the neural network.
Indeed our model took 3 hours on an A100 to train for 200 epochs, comparable to 7 hours on an A6000 for 200 epochs in~\citet{Stark2024-rf}.

\paragraph{DNA accessibility (ATAC) predictor} To get accurate predictions of a position-resolution epigenetic marker for DNA-accessibility (a property one often wants to design), we use the CNN bpAITAC model from \citet{Chandra2025} to predict chromatin accessibility traces with code from \texttt{https://github.com/nuriachandra/bpAITAC}.
The model is trained on  embryonic drosophila chromatin accessibility from the \citet{calderon2022} developmental fly dataset 16-20 hour subset, with ATAC-seq reads combined across cell types, using held-out chromosome chr2L for validation. bpAITAC was trained on a single NVIDIA TITAN RTX GPU (24GB) with early stopping based on validation loss.

bpAITAC produces two outputs: 1) total counts (a measure of regional accessibility), and 2) probability distribution of the counts. The base-pair resolution counts prediction is easily computed by multiplying the two outputs. The resulting per-base counts are modelled by a Poisson distribution.
That is, \texttt{bpAITAC} takes a one-hot-encoded sequence $x_0$ of length $D=500$ and predicts a positive $250$-dimensional vector that represents the predicted ``accessibility-profile'' in the centre $250$ positions of the sequence.
For a target profile of $250$ numbers, $\vec y\in\mathbb N^{250}$, we compute the probability by using the \texttt{bpAITAC} predictions as means of independent Poisson distributions
$$p(\vec y|x_0)=\prod_{d=1}^{D} \text{Poisson}\left(\vec v_d\right)(\vec y_d)$$
where $\vec v=\texttt{bpAITAC}(x_0)$ is the output of the predictor.
Since \texttt{bpAITAC} is a neural network which accepts one-hot-encoded $x_0t$, we may also pass $x_0$ which have each position $x_{0, d}$ lying on the simplex.

\paragraph{Evaluation}
We use the \texttt{ode\_likelihood} function to evaluate the likelihood of the trained diffusion model in the code of~\citet{Avdeyev2023-mv}. 

We collected 100 trace predictions from \cite{calderon2022} validation chromosome chr2L to use as targets, picking the 100 peaks with the highest combined signal.
For each target and model we sampled 10 conditional samples using 1000 function evaluations.
We sampled from our simplicial diffusion model using the procedure described in App.~\ref{app: sampling}.
To sample from the flow matching model in~\citet{Stark2024-rf}, we modified the \texttt{get\_cls\_score} function in their code to return the one-step predictor that we used in our App~\ref{app: sampling}.
Finally, we write custom code based on reversing an SDE to sample from the simplicial diffusion model in~\citet{Avdeyev2023-mv}.
To do so, we note they perform diffusion in a space with each position $\vec v_{t, d}\in[0, 1]^{B-1}$.
We compute their prediction $\tilde x_0(\vec v_t)$ by transforming the output of their neural network and then compute a prediction of $\nabla_{\vec v_t}\log p(y|\tilde x_0(\vec v_t))$ with a one-step estimator as in in our App~\ref{app: sampling}, and add it to their score for $\vec v$ every step.
We add this modification into their function \texttt{Euler\_Maruyama\_Sampler}.

To calculate $\tilde x_0(\vec v_t)$ we note they build a neural network to predict $\vec s = \nabla_{\vec v_t}\log p(\vec v_t)$ which equals $$\sum_b\tilde x_{0, b}\nabla_{\vec v_t}\log p(\vec v_t|x_0=b)$$
for some implicit prediction $\tilde x_{0, b}$ which we must solve for.
Now note, by the choice of the reverse stick-breaking procedure of \citet{Avdeyev2023-mv}, $\hat U_{b, b'}:=(\nabla_{\vec v_t}\log p(\vec v_t|x_0=b))_{b'}=\nabla_{ v_{t, b'}}\log p(v_{t, b'}|v_{0, b'}=\delta_{b, b'})$ for $b'\leq b$ and $(\nabla_{\vec v_t}\log p(\vec v_t|x_0=b))_{b'}=\nabla_{\vec v_{t, b'}}\log\mathrm{Beta}(1, B-b')(\vec v_{t, b'})$ otherwise.
So, $\vec s=U\tilde x_0=U_{:, :-1}\tilde x_{0, :-1}+U_{:, -1}(1-\tilde x_{0, :-1}^T\mathbbm 1)=(U_{:, :-1}-U_{:, -1}\mathbbm 1^T)\tilde x_{0, :-1}+U_{:, -1}$.
Therefore we can solve for $\tilde x_{0, :-1}$ by solving this linear system.

\subsection{Protein}

We describe the protein experiments in Sec.~\ref{sec:application unification}.

\paragraph{Training and data} 
For all protein models, we started from pre-trained ESM2 150M weights~\citep{Lin2023} under an MIT license as in MDLM~\citep{Wang2024-bn}.
We trained with a learning rate of $10^-5$ for an A100 80GB GPU over 48 h for 3 million sequences, substantially less than the training budget of~\citet{Wang2024-bn}.
We trained on UniRef50~\citep{Suzek2007-ev} data from \url{https://zenodo.org/records/6564798}.

\paragraph{Evaluation}
From each model we sampled 1000 sequences of length 200. We used a uniform grid of 100 points and integrated backwards, and we applied 4 corrector steps per predictor step as described in~\citet{Campbell2022-zm}. 
Then we predicted pLDDTs of sequences with Omegafold~\citet{Wu2022-ma} under the Apache-2.0 License, with 1 cycle for each sequence.

\subsection{Language}

We describe the language experiments in Sec.~\ref{sec:application unification}.

\paragraph{Training and data} 
We used the same architecture and training settings as~\citet{Lou2023-vm}, using their code at \url{https://github.com/louaaron/Score-Entropy-Discrete-Diffusion} under an MIT license.
We trained our model on $4$ A100 80GB GPUs for between 40 and 50 hours total on $33$ billion tokens taken from the \texttt{lm1b} dataset.
We used a learning rate of $3 \times 10^{-4}$ and an EMA of $0.9999$.
Our diffusion transformer had an embedding dimension of $768$ with $12$ layers and $12$ attention heads. 
The Gaussian models used pre-trained BERT embeddings scaled by a factor of $8$. 

For our individual discrete and Gaussian language models, each device used a physical batch size of $64$ and took $2$ gradient accumulation steps for an overall batch size of $512$.
For our unified model, we accumulated over Gaussian and discrete batches to get an overall batch size of $1024$.

\paragraph{Evaluation}
We sample using $1000$ iterations. 
Following \citet{Lou2023-vm}, we evaluate the sample quality of our models through the generative perplexity of their unconditional samples according to GPT2-large. 

\section{Supplementary Experiments}

\subsection{Antibody optimization downstream task}\label{app: thermo experiment}

We test our unified models from Sec.~\ref{sec:application unification} on a different downstream task.
\citet{Hie2023-sr} suggested that generative protein models can be used to suggest mutations that improve the stability of antibody sequences.
To test our diffusion model’s ability to successfully improve antibody properties, we perturb a parental VHH sequence by noising a UniRef50-trained diffusion model by $t$ then denoising with 128 steps.\footnote{Note that this follows established methods for ML-based antibody diversification as in \citet{raghu2025guided}.}
To emulate a realistic wet lab setting, we investigate sampling 50 unique single- and double-point mutants of the seed VHH by rejection sampling.
We repeat this process 100 times, selecting the top resulting sequence from each repeat ``experiment’’ according to a proprietary thermostability oracle. 
The amount of noising for each of the individual Gaussian, simplicial, and discrete diffusion models was determined by a hyperparameter sweep over $t \in [0.01, 0.02, 0.05, 0.1, 0.2, 0.5]$, where the chosen hyperparameter gave the most unique sequences with fewer than or equal to 5 mutations to the parental sequence.
This hyperparameter was then shared with each sub-model of the unified model.

The thermostability oracle we used is an ensemble of 10 CARP/ByteNet regressors \citep{yang2024convolutions}, pretrained on approximately 537,000 sequences from phage display, processed using Next Generation Sequencing (NGS), and 9556 $T_m$ datapoints obtained from NanoDSF. 
The resulting ensemble achieved a test cross-validated Spearman correlation of 0.72.

In Fig.~\ref{fig: transfer bighat} we see that unification does not substantially harm performance on this downstream task.

\begin{figure}[h]
    \centering
    \includegraphics[width=0.4\linewidth]{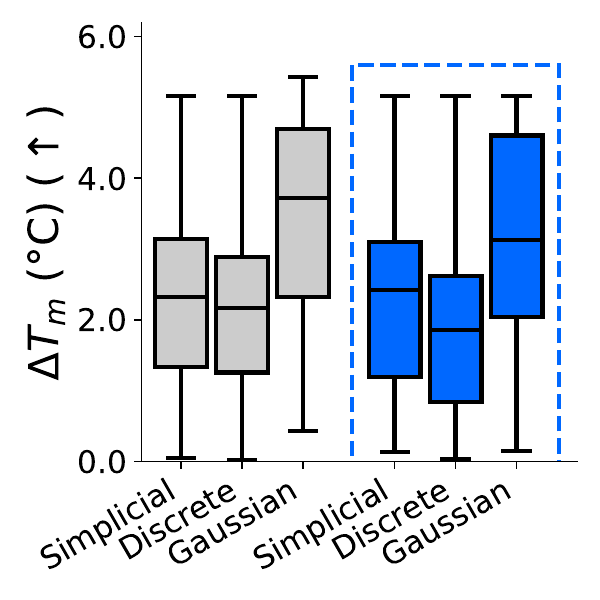}
    \caption{\textbf{The sufficient statistic parametrization enables a single model to perform competitive discrete, Gaussian, and simplicial optimization of antibodies.}
    Using our protein models from Fig.~\ref{fig: transfer}, we ``denoise'' antibody sequences and plot the predicted improvement in melting temperature in libraries of size 100.}
    \label{fig: transfer bighat}
\end{figure}

\subsection{Fitting image data: MNIST}\label{app: mnist}

We perform the analysis of Fig.~\ref{fig: transfer} for image data and find a similar result.
We evaluate our unified discrete diffusion framework on the MNIST dataset, consisting of 28x28 grayscale images. We discrete the pixel intensities to $N=8$ levels using uniform quantization, preserving the continuous structure of the token identities while reducing the computational cost.
We compare the performance of our single unified model (SSP) to the performance of three individually-trained diffusion models: discrete, simplicial, and Gaussian. 

All models use a U-Net backbone with an embedding size of 128, 4 downsampling/upsampling blocks, and ReLU activations. 
Models are trained with the Adam optimizer with learning rate 0.001 and batch size 128 for 20 epochs. For Gaussian diffusion, we map each class index $x \in \{0, \cdots, C\}$ to a 2D continuous embedding with a circular parameterization $\mathrm{emb}(x)=( \cos(\theta), \sin(\theta))$ where $\theta = \frac{x}{C-1}\pi$. This embeddings encodes the similarity of different pixel values and ensures that the resulting diffusion process closely resembles continuous diffusion.
We found models with a 1-D parameterization $\mathrm{emb}(x)=2 \times (\frac{x}{C-1})-1$ performed much worse.

We evaluate the model performance using validation likelihood, as shown in \Cref{fig: mnist nll}. 
First, as in Fig.~\ref{fig: transfer} we find that the likelihoods between the unified model are competitive with the individually trained models.
In fact, we are even able to achieve slightly better performance for discrete and Gaussian diffusion, perhaps because the parameterization is easier to learn from, or because of a benefit from learning on diverse data.

As well, while we might expect Gaussian diffusion to achieve the best data fit due to the continuous nature of the data, we see the opposite: among our individual models, Gaussian surprisingly achieves the worst likelihood. 
This demonstrates the importance of considering multiple types of diffusion paradigms depending on the downstream tasks, thereby motivating our approach of training a single unified model.

We also generate 64 unconditional samples per model using 1,000 steps of ancestral sampling. Through our visualizations in \Cref{fig: mnist samples}, we see that the unified model does not lead to a noticeable drop in sample quality compared to individual models.

\begin{figure}
    \centering
    \includegraphics[width=0.5\linewidth]{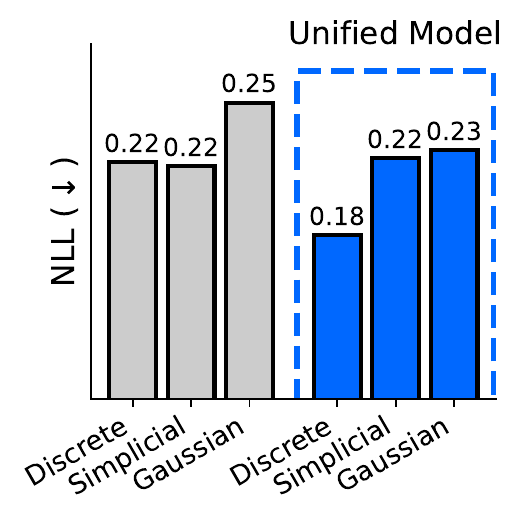}
    \caption{\textbf{The SSP enables a single model to fit image data across 3 modalities.} We perform the analysis of Fig.~\ref{fig: transfer} for image data and find a similar result on MNIST.}
    \label{fig: mnist nll}
\end{figure}

\begin{figure}
    \centering
    \begin{subfigure}[b]{0.32\textwidth}
    \includegraphics[width=\textwidth]{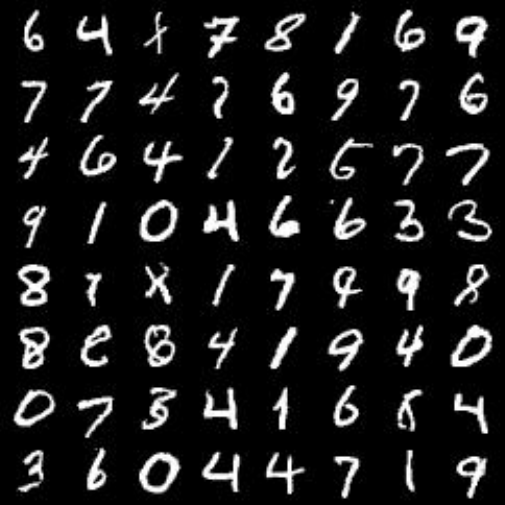}
    \caption{Discrete}
    \end{subfigure}%
    \hfill
    \begin{subfigure}[b]{0.32\textwidth}
    \includegraphics[width=\textwidth]{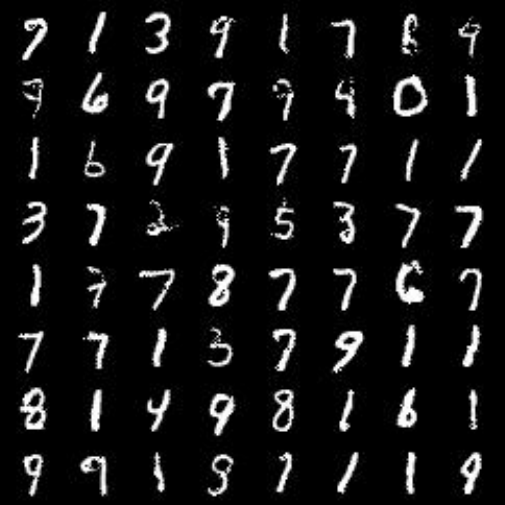}
    \caption{Simplicial}
    \end{subfigure}%
    \hfill
    \begin{subfigure}[b]{0.32\textwidth}
    \includegraphics[width=\textwidth]{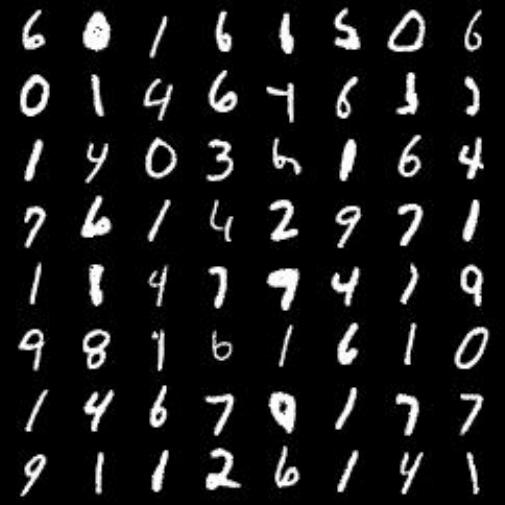}
    \caption{Gaussian}
    \end{subfigure}
    \begin{subfigure}[b]{0.32\textwidth}
    \includegraphics[width=\textwidth]{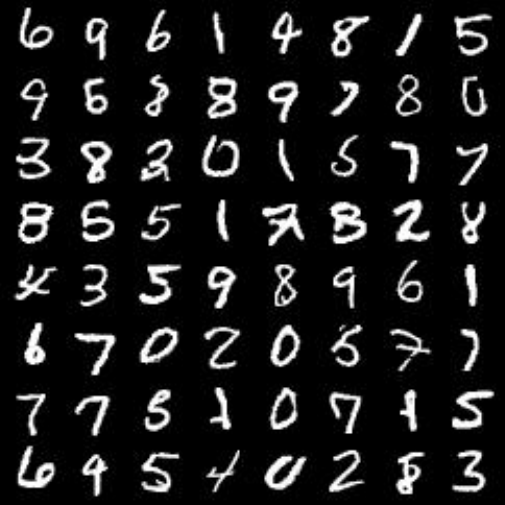}
    \caption{Unified Discrete}
    \end{subfigure}%
    \hfill
    \begin{subfigure}[b]{0.32\textwidth}
    \includegraphics[width=\textwidth]{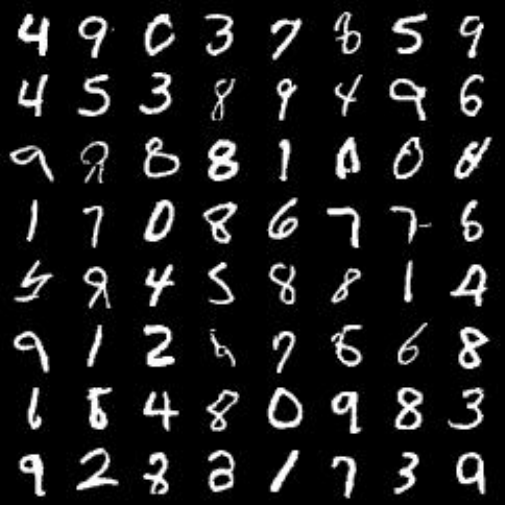}
    \caption{Unified Simplicial}
    \end{subfigure}%
    \hfill
    \begin{subfigure}[b]{0.32\textwidth}
    \includegraphics[width=\textwidth]{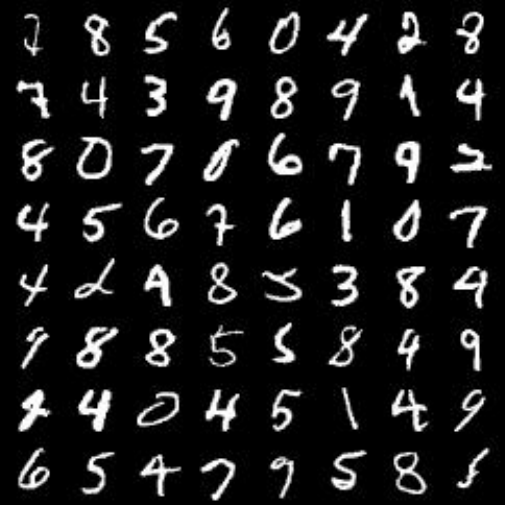}
    \caption{Unified Gaussian}
    \end{subfigure}
    
    \caption{\textbf{The SSP results in no noticeable drop in generation quality for image models.}
    We plot samples from models trained on MNIST.}
    \label{fig: mnist samples}

\end{figure}

\end{document}